\documentclass[11pt,oneside]{article}

\usepackage[abbrvbib, preprint]{jmlr2e}

\usepackage{amsmath,amsfonts,amscd,amssymb}
\usepackage{natbib}
\usepackage[latin1]{inputenc}
\usepackage[T1]{fontenc}
\usepackage{mathtools}
\usepackage{dsfont}
\usepackage{color}
\usepackage[mathscr]{euscript}
\usepackage{graphicx}
\usepackage{placeins}
\usepackage{extsizes}
\usepackage{float}
\usepackage{multirow}
\usepackage{pdflscape}
\usepackage{caption}
\usepackage{pgfplots}
\usepackage{subcaption}
\usepackage{mathabx}
\usepackage{tikz}
\usepackage{longtable}
\usepackage{todonotes}
\pgfdeclarelayer{bg}    
\pgfsetlayers{bg,main}
\usepackage{rotate}
\usepackage[noend]{algorithmic}
\usepackage{algorithm}
\usepackage{setspace}
\usepackage{layout}
\usepackage{hyperref}
\usepackage{changes}
\graphicspath{{figures/}}
\definecolor{color1bg}{HTML}{f73d28}
\definecolor{color2bg}{HTML}{FA8072}
\definecolor{bblue}{HTML}{00BFFF}
\definecolor{bblue2}{HTML}{00ffff}

\def\sB{\mathcal{B}}
\def\sX{\mathcal{X}}
\def\sY{\mathcal{Y}}

\def\sH{\mathcal{H}}

\def\RR{{\mathbb{R}}}

\usepackage{lineno}

\definecolor{bblue}{rgb}{.2,0.2,.8}

\usepackage{lastpage}
\jmlrheading{}{}{1-\pageref{LastPage}}{; Revised }{}{}{Diego Marcondes and Ulisses Braga-Neto}

\ShortHeadings{Generalized Resubstitution for Regression Error Estimation}{Marcondes and Braga-Neto}
\firstpageno{1}

\begin{document}

\title{Generalized Resubstitution for Regression Error Estimation}

\author{\name Diego Marcondes \email dmarcondes@ime.usp.br \\
	\addr Department of Computer Science, Institute of Mathematics and Statistics\\
	University of S\~ao Paulo\\
	S\~ao Paulo, Brazil
	\AND
	\name Ulisses Braga-Neto \email ulisses@tamu.edu \\
	\addr Department of Electrical and Computer Engineering\\
	Texas A\&M University\\
	College Station, TX 77843, USA}

\editor{}

\maketitle

\begin{abstract}
	We propose generalized resubstitution error estimators for regression, a broad family of  estimators, each corresponding to a choice of  empirical probability measures and loss function.  The usual sum of squares criterion is a special case corresponding to the standard empirical probability measure and the quadratic loss. Other choices of empirical probability measure lead to more general estimators with superior bias and variance properties. We prove that these error estimators are consistent under broad assumptions. In addition, procedures for choosing the empirical measure based on the method of moments and maximum pseudo-likelihood are proposed and investigated.  Detailed experimental results using polynomial regression demonstrate empirically the superior finite-sample bias and variance properties of the proposed estimators. The R code for the experiments is provided.
\end{abstract}


\section{Introduction}

In supervised learning, the objective is to build a mapping $\psi: \sX \rightarrow \sY$ to predict the value of a target $Y \in \sY$ from an input $X \in \sX$, where $\sX$ and $\sY$ are suitable spaces, using a training data set $S_n = \{(X_1,Y_1),\ldots,(X_n,Y_n)\}$. Once $\psi$ is built, it is necessary to evaluate its performance and, in recent years, it has become the norm to use the empirical error on data not appearing in $S_n$ to benchmark predictors~\citep{russakovsky2015imagenet,jiang2019fantastic}. A more rigorous alternative to performance evaluation is provided by a {\em statistical approach} to learning, where $X$ and $Y$ are assumed to be random variables taking values in $\sX$ and $\sY$, respectively, and the training data $S_n$ is generally assumed to be an independent and identically distributed sample from a joint probability measure $\nu(X,Y)$. This approach allows one to define rigorously the {\em generalization error} of the predictor on future data, given an appropriate loss criterion, as the expected loss with respect to $\nu$.   

A crucial problem in practice is how to estimate accurately the generalization error. A predictor is useful only if its generalization error can be stated with confidence. In the statistical approach to learning, the problem of error estimation has been extensively studied in the classification case, where $\sY$ consists of a finite set of labels, and the loss criterion is simply whether the predictor recovers the label or not; see ~\cite{Tous:74,Hand:86,McLa:87,SchiHand:00, BragDougEEPR:15} for comprehensive surveys. However, the problem of estimating the generalization error in the case of regression, where $\sY$ is an Euclidean space, is much less studied. In this paper, we provide a comprehensive study of a new class of error estimators for regression problems,  namely, the family of {\em generalized regression resubstitution estimators}. 

The popular test-set error estimator is known to have excellent statistical properties; it is unbiased and consistent regardless of the sample size or data generating distribution \citep{BragDougEEPR:15}. However, this is only true if the test data is truly independent of training and used only once~\citep{yousefi2011multiple}. In practice, this is almost never the case, with the test data being reused, in some cases heavily, to measure performance improvement, creating a situation known as ``training to the test data'' \citep{recht2019imagenet}. In addition, if training and testing sample sizes are small, the test-set error estimator can display large variance and become unreliable, which means that test-set error estimation requires cheap access to plentiful labeled data.

The simplest alternative that requires no separate test data is the empirical error on the training data; this is known as the {\em resubstitution} error estimator~\citep{Smit:47}. The resubstitution estimator is however usually optimistically biased, the more so the more the prediction algorithm overfits the training data. Optimistic bias implies that the difference between resubstitution estimate and the true error, which has been called the ``generalization gap''~\citep{keskar2016large}, is negative with a high probability. It is key therefore to investigate mechanisms to reduce the bias.

Bolstered resubstitution, introduced in~\cite{braga2004bolstered}, proposed a modification to resubstitution for classification, where the empirical measure, which produces the plain resubstitution estimator, is smoothed by kernels in order to reduce both the bias and variance in small-sample cases. In~\cite{ghane2022generalized}, the family of generalized resubstitution error estimators for classification was proposed and studied. These estimators are defined in terms of arbitrary empirical probability measures and include as special cases both plain and bolstered resubstitution, as well as posterior-probability \citep{LugoPawl:94}, Gaussian-process~\citep{hefny2010new}, and the Bayesian~\citep{DaltDoug:11a,DaltDoug:11b} classification error estimators. It was shown in that paper that generalized resubstitution error estimators are consistent and asymptotically unbiased for the two-class problem if the corresponding empirical probability measure converges uniformly to the standard empirical probability measure and the hypothesis space has a finite VC dimension. 

In this work, we extend the generalized resubstitution error estimators to the general statistical learning framework for regression, when the loss function has a moment of order $>1$ uniformly bounded in the hypothesis space. In particular, we consider regression problems under the quadratic loss function and extend the special cases studied in \cite{ghane2022generalized} to them. The generalized error estimators can be defined as the expectation of the loss function under an arbitrary empirical probability measure or as the expectation of a generalized loss function under the standard empirical probability measure. 

These representations allow us to establish several sufficient conditions for the consistency of these estimators that not only extend that of \cite{ghane2022generalized}, but are also weaker than it. For example, we show that if the generalized loss function converges uniformly to the original one, or if the expectation of the generalized loss function under the data-generating distribution converges to that of the original loss function, then the generalized resubstitution error estimator is consistent. In particular, this last condition allows for the variance of the empirical measure to not converge to zero. Sufficient conditions for consistency based on the variance of the empirical measure are also established for the case of twice-differentiable loss functions.

Finally, since consistency is attained under many conditions, there is plenty of room to choose the empirical measure, and we propose methods to estimate the parameters of the empirical measure from the data. We focus on bolstered error estimators and propose method of moments and maximum pseudo-likelihood estimators of the covariance matrix of the bolstered empirical measure, which formalizes the heuristic approach proposed in \cite{braga2004bolstered} for that purpose.

We summarize our contributions as follows:
\begin{itemize}
	\item We propose generalized resubstitution error estimators for regression, a general framework for resubstitution-like regression error estimators. Like the classification error estimators in \cite{ghane2022generalized}, our estimators are based on arbitrary empirical probability measures. 
	\item We establish several sufficient conditions for the consistency of these estimators that not only extend that of \cite{ghane2022generalized}, but are also weaker than it.
	\item Additional sufficient conditions for consistency based on the variance of the empirical measure are established for the case of twice-differentiable loss functions.
	\item To address the issue of choosing the proper amount of smoothing to cancel estimator bias, we propose method of moments and maximum
	pseudo-likelihood estimators of the covariance matrix of the Gaussian empirical probability measure for Gaussian bolstering error estimation.
\end{itemize}
Our paper is organized as follows:
\begin{itemize}
	\item In Sections \ref{SecBasicNotation} and \ref{SecGenErrorEst} we define the problem of error estimation in statistical learning and give an informal presentation of the generalized resubstitution error estimators.
	\item In Section \ref{SecError}, we formally define the generalized resubstitution error estimators, define the consistency of error estimators, and present sufficient conditions for the consistency of generalized resubstitution error estimators. In particular, in Section \ref{SubSec_GenC} we present sufficient conditions that hold in general cases, in Section \ref{SecSuffSize} we focus on estimators in which the associated empirical measure depends only on the sample size, and in Section \ref{SubSec_2diff} we study the case of twice-differentiable loss functions.
	\item The sufficient conditions established give room to choose the generalized empirical measure in various ways, and in Section \ref{SecHeu} we propose two methods for learning the covariance matrix for Gaussian bolstering from the training data.
	\item In Section \ref{SecExp}, we empirically assess the performance of the proposed error estimators in polynomial regression experiments, and in Section \ref{SecDiscuss} we provide concluding remarks.
\end{itemize}
Proofs for the results are presented in Appendix \ref{AppProofs}. In Appendixes \ref{AppVCdim}, \ref{AppLinf} and \ref{SecSquareLoss} we discuss the VC dimension of hypothesis spaces under general loss functions and prove some auxiliary results that characterize the loss functions and hypothesis spaces for which some of the established sufficient conditions hold.

\section{Background}
\label{SecBasicNotation}

Let the pair of random vectors $(X,Y)$ take values in a product space $\sX \!\times\! \sY$. In most cases of interest, $\sX \subseteq \RR^d$ and $\sY \subseteq \RR^m$, where usually $m = 1$ (though this is not necessary). The pair $(X,Y)$ is jointly distributed with an unknown probability law $\nu$ on $(\sX \!\times\! \sY,\,\sB_{XY})$, where $\sB_{XY}$ denotes the Borel $\sigma$-algebra in $\sX \!\times\! \sY$. Random variables $X$ and $Y$ are the {\em feature vector} and {\em target}, respectively, in a statistical learning problem. The {\em hypothesis space} $\sH$ is a set of Borel-measurable functions $\psi: \sX \!\to\! \sY$. For example, in a classification problem, $\sY = \{0,1,\ldots,c-1\}$ and $\psi$ is called a {\em classifier}, whereas in a standard regression problem, $\sY = \RR$ and $\psi$ is called a {\em regression function}.

Each hypothesis $\psi$ in $\mathcal{H}$ has an error defined as
\begin{linenomath}
	\begin{equation*}
		L(\psi) = \int_{\sX \times \sY} \ell(y,\psi(x)) \ d\nu(x,y)	
	\end{equation*}
\end{linenomath}
where $\ell:\sY \times \sY \mapsto \mathbb{R}^+$ is an appropriate {\em loss function}. For example, with $\sY = \RR$, the {\em \index quadratic loss} $\ell(y,\psi(x)) = (y - \psi(x))^2$ yields the regression mean squared error, whereas with $\sY = \{0,1,\ldots,c-1\}$, the {\em misclassification loss} $\ell(y,\psi(x)) \,=\, \mathds{1}_{y \neq \psi(x)}$ yields the probability of classification error. This error is unknown, and in order to estimate it in practice, one collects an i.i.d.\ {\em training sample} of $(X,Y)$
\begin{linenomath}
	\begin{equation*}
		S_{n} = \{(X_{1},Y_{1}),\dots,(X_{n},Y_{n})\}
	\end{equation*}
\end{linenomath}
and consider the empirical error
\begin{equation*}
	L_{n}(\psi) = \frac{1}{n} \sum_{i=1}^{n} \ell(Y_{i},\psi(X_{i}))
\end{equation*}
as an estimator of $L(\psi)$.

Solving a learning problem in this instance means picking a hypothesis $\psi_n$ in $\sH$ according to a given empirical criterion based on a sample $S_{n}$. This task can be described as a prediction rule $\Psi: (\sX \times \sY)^{n} \to \mathcal{H}$ which associates each sample $S_{n} \in (\sX \times \sY)^{n}$ to a hypothesis $\psi_{n} \coloneqq \Psi(S_{n}) \in \mathcal{H}$. Observe that $\psi_{n}$ is random since it is a function of sample $S_{n}$.

The quantity of interest in statistical learning is how well one can expect $\psi_{n}$ to perform on data not in the sample, but generated by the same law as $(X,Y)$, that is the prediction error of $\psi_n$, defined as 
\begin{equation}
	\varepsilon_n \,=\, L(\psi_{n}) \,=\,
	\int_{\sX \times \sY} \!\!\ell(y,\psi_n(x))\, d\nu(x,y)\,.
	\label{eq-err2}
\end{equation}

\section{Generalized Resubstitution Error Estimators}
\label{SecGenErrorEst}

A family of estimators of the prediction error $\varepsilon_{n}$ in \eqref{eq-err2} is obtained by replacing the generally unknown probability measure $\nu$ in \eqref{eq-err2} for an {\em empirical measure} $\nu_n$, which is a probability measure defined on $(\sX \times \sY,\,\sB_{XY})$ that is a function of the sample~$S_n$:
\begin{equation}
	\hat{\varepsilon}^{\nu_{n}}_n \,=\, \int_{\sX \times \sY} \!\!\ell(y,\psi_n(x))\, d\nu_n(x,y)\,.
	\label{eq-gresub2}
\end{equation}
Following \cite{ghane2022generalized}, we call these \textit{generalized resubstitution error estimators}.

This general family of error estimators includes the estimator generated by the standard empirical measure, which puts discrete mass $1/n$ on each data point:
\begin{equation}
	\nu_n\,=\, \frac{1}{n}\sum_{i=1}^n \delta_{X_i,Y_i}\,,
	\label{eq-empdist}
\end{equation}
where $\delta_{X_i,Y_i}$ is the (random) point measure located at $(X_i,Y_i)$. Substituting \eqref{eq-empdist} in (\ref{eq-gresub2}) yields the empirical error $L_{n}(\psi_n)$ of $\psi_{n}$ on $S_n$, which is known as the {\em resubstitution error estimator}:
\begin{equation}
	\hat{\varepsilon}^{\,r}_n\,=\, L_{n}(\psi_n) \,=\, \frac{1}{n} \sum_{i=1}^n \ell(Y_i,\psi_n(X_i))\,.
	\label{eq-plainresub}
\end{equation}
For example, with $\sY = \mathbb{R}$ and the quadratic loss, this is the sum of squared errors in regression, while with $\sY = \{0,1,\ldots,c-1\}$ and the misclassification loss, this is the training error of a classifier.

We restrict our attention in this paper to a specific, yet broad, class of generalized resubstitution estimators, which are based on the family of empirical measures of the form:
\begin{equation}
	\nu_n \,=\, \frac{1}{n}\sum_{i=1}^n \beta_{n,i} ,
	\label{eq-smooth}
\end{equation}
where $\beta_{n,i}$ is a measure on $(\sX\times \sY,\sB_{\sX\sY})$, for $i=1,\ldots,n$. Although $\beta_{n,i}$ may depend on the entire sample $S_{n}$, we assume that it has a special dependence on $X_{i}, Y_{i}$ and $\psi_{n}$, and we may use the explicit notation $\beta_{X_{i},Y_{i},\psi_{n},S_{n}}$ to make this clear. When $\beta_{n,i} = \delta_{X_i,Y_i}$ we recover the standard resubstitution error estimator.

A special case of interest is when $\beta_{n,i}$ is a product measure,
\begin{equation}
	\beta_{n,i} = \mu_{n,i}\,\pi_{n,i}\,, 
	\label{eq-prodmeas}
\end{equation}
where $\mu_{n,i}$ and $\pi_{n,i}$ are empirical measures on $(\sX,\sB_\sX)$  and $(\sY,\sB_\sY)$, respectively. In this instance, substituting (\ref{eq-smooth}) into (\ref{eq-gresub2}) and using Fubini's Theorem yields the following generalized resubstitution error estimator:
\begin{equation}
	\begin{aligned}
		\hat{\varepsilon}^{\nu_n}_n & \,=\, \frac{1}{n} \sum_{i=1}^{n} \int_{\sX \times \sY} \!\!\ell(y,\psi_n(x))\, d\beta_{n,i}(x,y) \\[1ex]
		&\,=\, \frac{1}{n}\sum_{i=1}^{n} \int_{\sY} \left(\int_{\sX} \ell(y,\psi_n(x))\,d\mu_{n,i}(x)\right) d\pi_{n,i}(y)
	\end{aligned}
	\label{eq-gen_error}
\end{equation}

The standard resubstitution error estimator in (\ref{eq-plainresub}) is a special case of (\ref{eq-gen_error}), with $\mu_{n,i} = \delta_{X_i}$,  $\pi_{n,i} = \delta_{Y_i}$, and $\beta_{n,i} = \delta_{X_i}\delta_{Y_i} = \delta_{X_i,Y_i}$. One can interpret the empirical measures $\mu_{n,i}$, $\pi_{n,i}$, and $\beta_{n,i}$ as {\em smoothed} versions of the point measures $\delta_{X_i}$, $\delta_{Y_i}$, and $\delta_{X_i,Y_i}$, respectively. 

Specific examples of generalized resubstitution estimators are given in the next few subsections.

\subsection{Bolstered Resubstitution (smoothing only in the ``$X$ direction'')}

In this case, $\beta_{n,i} = \mu_{n,i}\,\pi_{n,i}$ is the product measure \eqref{eq-prodmeas}, where $\pi_{n,i} = \delta_{Y_i}$ and $\mu_{n,i}$ is a general empirical measure, in which case (\ref{eq-gen_error}) becomes:
\begin{equation}
	\begin{aligned}
		\hat{\varepsilon}^{br}_n &\,=\, \frac{1}{n}\sum_{i=1}^{n} \int_{\sX} \ell(Y_i,\psi_n(x))\, d\mu_{n,i}(x)\,.
	\end{aligned}
	\label{eq-gen_error1}
\end{equation}
This is the {\em bolstered resubstitution} error estimator, which was proposed in \cite{braga2004bolstered} for the case of binary classification, here extended to the general statistical learning case.  Notice that this estimator performs smoothing in the ``$X$ direction''.

Usually, $\mu_{n,i}$ is assumed to be absolutely continuous, with a probability density function $p_{n,i}(x)$. If in addition we assume the quadratic loss, (\ref{eq-gen_error1}) becomes
\begin{linenomath}
	\begin{equation}
		\label{gbr}
		\hat{\varepsilon}_{n}^{gbr} = \frac{1}{n} \sum_{i=1}^{n} \int_{\mathbb{R}^{d}} (\psi_{n}(x) - Y_{i})^{2} \ p_{n,i}(x) \ dx
	\end{equation}
\end{linenomath}

In practice, the density $p_{n,i}(x)$ is centered on the data point $X_i$, when it is called a {\em bolstering kernel}. Then the contribution of each point $(X_{i},Y_{i})$ to $\hat{\varepsilon}_{n}^{br}$ is given by assessing how well $\psi_{n}(x)$ approximates $y$ when $x$ is a random perturbation of $X_{i}$, given by $p_{n,i}(x)$, and $y = Y_{i}$. This is illustrated in Figure \ref{fig_ex_poly_gb}, with a Gaussian density $p_{n,x}$ centered at each $X_i$.

\begin{figure}[ht]
	\centering
	\includegraphics[width=\linewidth]{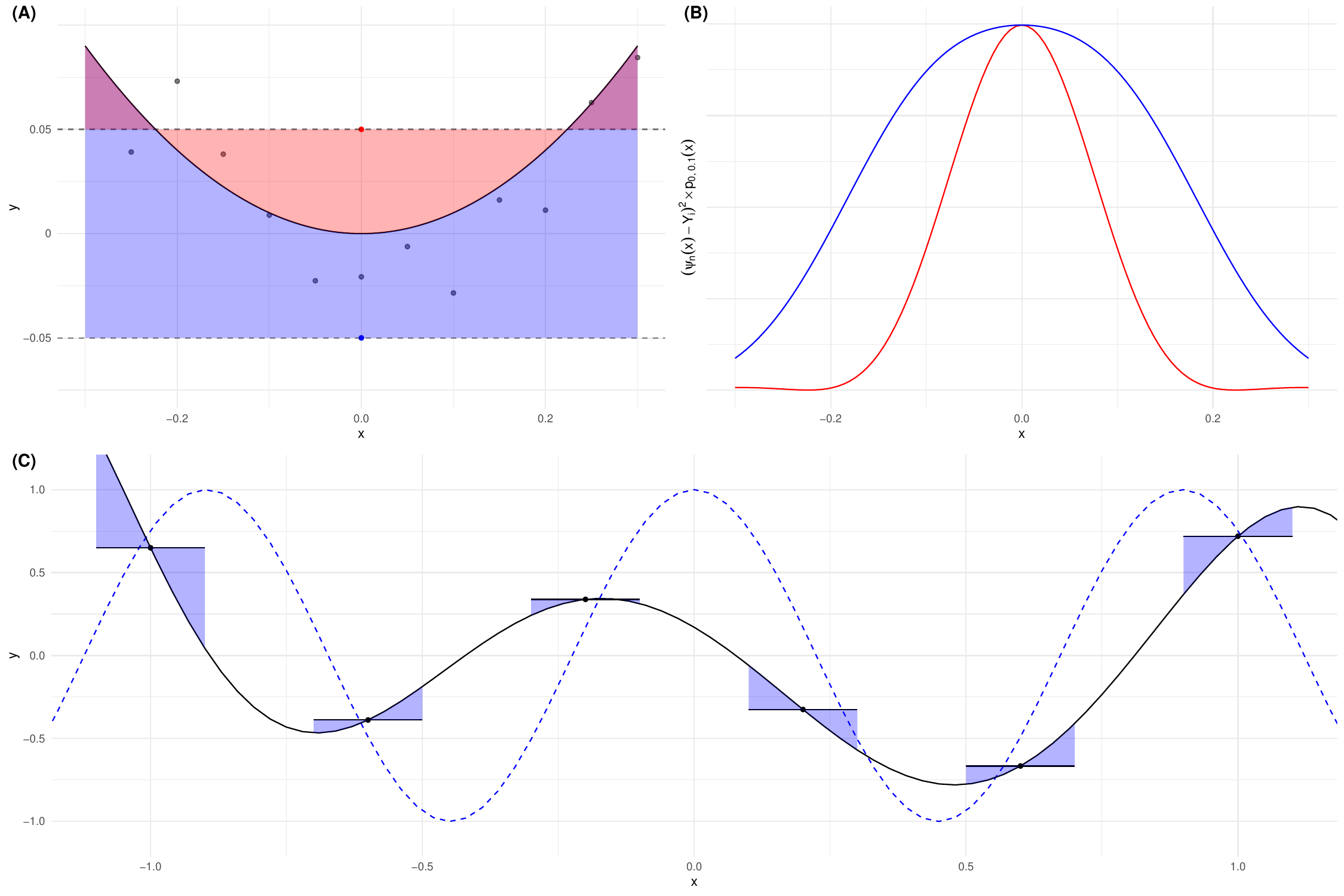}
	\caption{\footnotesize An illustration of Gaussian bolstering for regression. \textbf{(A)} An example in which $\psi_{n}(x) = x^2$ with the points $(0,0.05)$ and $(0,-0.05)$ outlined in red and blue, respectively. The predictor $\psi_{n}$ has the same quadratic loss in both points. The shaded area represents the distance from $\psi_{n}(x)$ to $0.05$ (red) and $-0.05$ (blue). \textbf{(B)} The squared distance from $\psi_{n}(x)$ to $0.05$ (red) and $-0.05$ (blue) in \textbf{(A)} weighted by a Gaussian density with mean zero and standard deviation $0.1$. The contribution of each point to the Gaussian bolstering estimator \eqref{gbr} is the area under the respective curve, and hence that of the blue point is greater. \textbf{(C)} Points $(x,y)$ generated by $y = \psi^{\star}(x) + \epsilon$, in which $\psi^{\star}$ is the dashed curve in blue and $\epsilon$ is a Gaussian noise. The shaded regions represent $X_{i} \pm 3\sigma$ in which $\sigma$ is the standard deviation of the Gaussian bolstering distribution. The polynomial $\psi_{n}$ (black) interpolates the data, so its resubstitution error is zero. Nevertheless, the Gaussian bolstering error estimator is not zero and equals essentially the mean distance from $\psi_{n}(x)$ to $Y_{i}$ for $x$ in the shaded neighborhood of $X_{i}$ when $x$ is distributed as a Gaussian distribution with mean $X_{1}$ and standard deviation $\sigma$.}
	\label{fig_ex_poly_gb}
\end{figure}

\subsection{Posterior-Probability Error Estimator (smoothing only in the ``$Y$ direction'')}

In this case, $\mu_{n,i} = \delta_{X_i}$ and $\pi_{n,i}$ is a general empirical measure so (\ref{eq-gen_error}) becomes:
\begin{equation}
	\begin{aligned}
		\hat{\varepsilon}^{\nu_{n}}_n &\,=\, \frac{1}{n}\sum_{i=1}^{n} \int_{\sY} \ell(y,\psi_n(X_{i}))\, d\pi_{n,i}(y)\,.
	\end{aligned}
	\label{eq-gen_errorY}
\end{equation}
A special case in classification problems is the \textit{posterior-probability resubstitution generalized error estimator} in which $\pi_{n,i}$ is a posterior probability of $Y$ conditioned on $X = X_{i}$:
\begin{linenomath}
	\begin{equation*}
		\label{law_pi_y}
		\pi_{n,i}(y) = \hat{P}_{n}\left(Y_{i} = y|X = X_{i}\right)
	\end{equation*}
\end{linenomath}
where $\hat{P}_{n}$ is a posterior probability in $\sY = \{0,\dots,c-1\}$, so \eqref{eq-gen_errorY} reduces to
\begin{linenomath}
	\begin{align*}
		\varepsilon_{n}^{ppr} = \frac{1}{n} \sum_{i=1}^{n} \hat{E}_{n}\left(\ell(Y_{i},\psi_{n}(X_{i})) | X = X_{i}\right).
	\end{align*}
\end{linenomath}
In the same manner, one could consider $\varepsilon_{n}^{ppr}$ in regression by considering a posterior probability density function as the density of the absolutely continuous measure $\pi_{n,i}$. In Figure \ref{fig_ex_posterior} we illustrate posterior-probability resubstitution for Bayesian regression. 

\begin{figure}[ht]
	\centering
	\includegraphics[width=\linewidth]{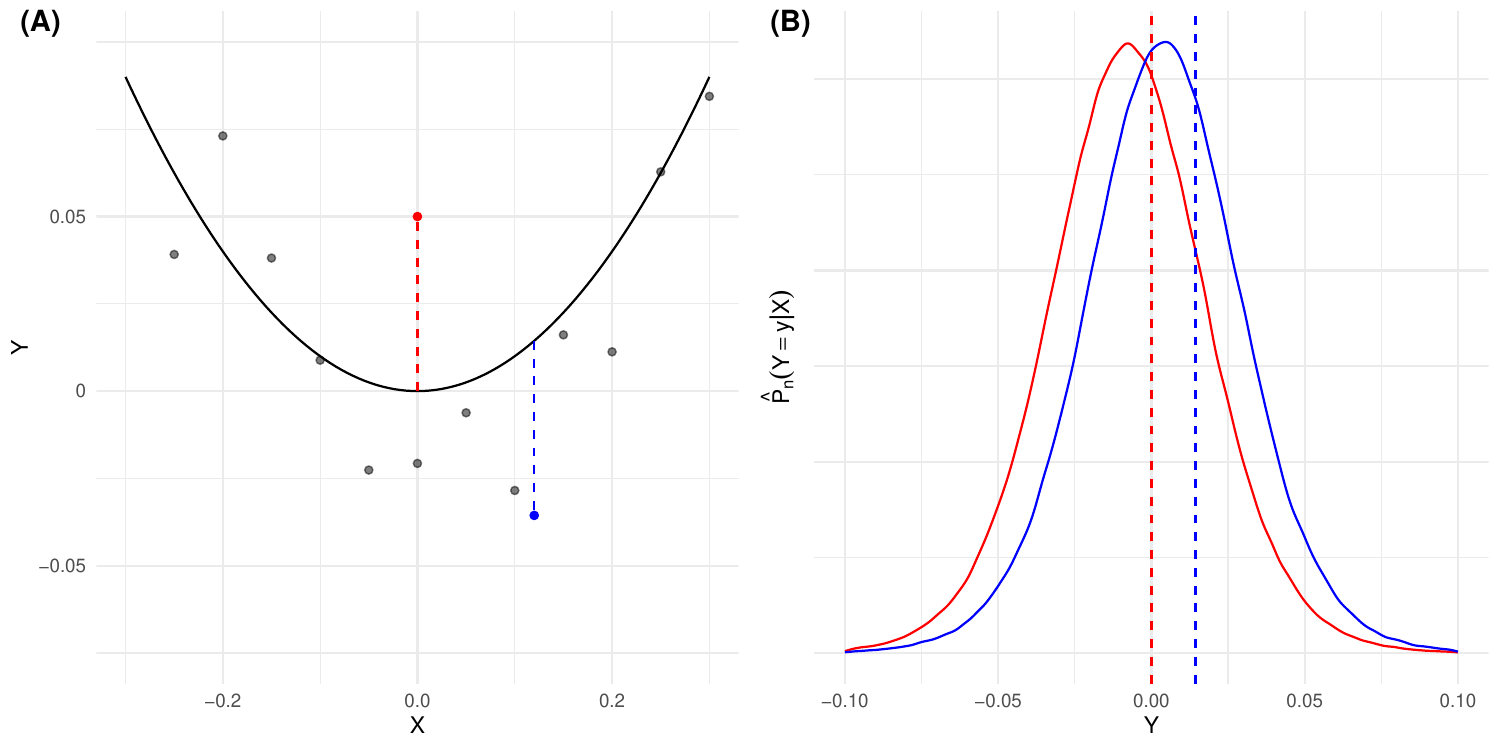}
	\caption{\footnotesize An illustration of posterior-probability resubstitution for Bayesian regression. \textbf{(A)} An example in which $\psi_{n}(x) = x^2$ with the points $(0,0.05)$ and $(0.12,-0.0356)$ outlined in red and blue, respectively. The predictor $\psi_{n}$ has the same quadratic loss in both of these points. \textbf{(B)} The posterior probability of $Y$ given $X_{i} = 0$ (red) and $X_{i} = 0.12$ (blue). The vertical dashed lines represent the respective value of $\psi_{n}(X_{i})$. The contribution of each point to $\varepsilon_{n}^{ppr}$ is the expected value of $(\psi_{n}(X_{i}) - Y)^2$ under the respective posterior probability, that is, the expected squared distance from $Y$ to the respective dashed line. The contribution of the blue point is greater than that of the red one.}
	\label{fig_ex_posterior}
\end{figure}

\subsection{Smoothing in ``both directions.''}

This is the general case, where neither $\mu_{n,i}$ and $\pi_{n,i}$ are point measures. For example, one could combine bolstered and posterior-probability resubstitution by considering $\pi_{n,i}$ as a posterior distribution of $Y$ conditioned on $X = X_{i}$ and $\mu_{n,i}$ with general bolstering kernel $p_{n,i}$. 

A more general setting is when $\beta_{n,i}$ is not a product measure. An important case is when $\beta_{n,i}$ is a Gaussian distribution with mean vector $(X_{i},Y_{i})$ and a possibly non-spherical covariance matrix $\Sigma_{i}$, as exemplified in Figure \ref{fig_ex_both}. We call this special case the $XY$\textit{-Gaussian bolstering} error estimator.

\begin{figure}[ht]
	\centering
	\includegraphics[width=\linewidth]{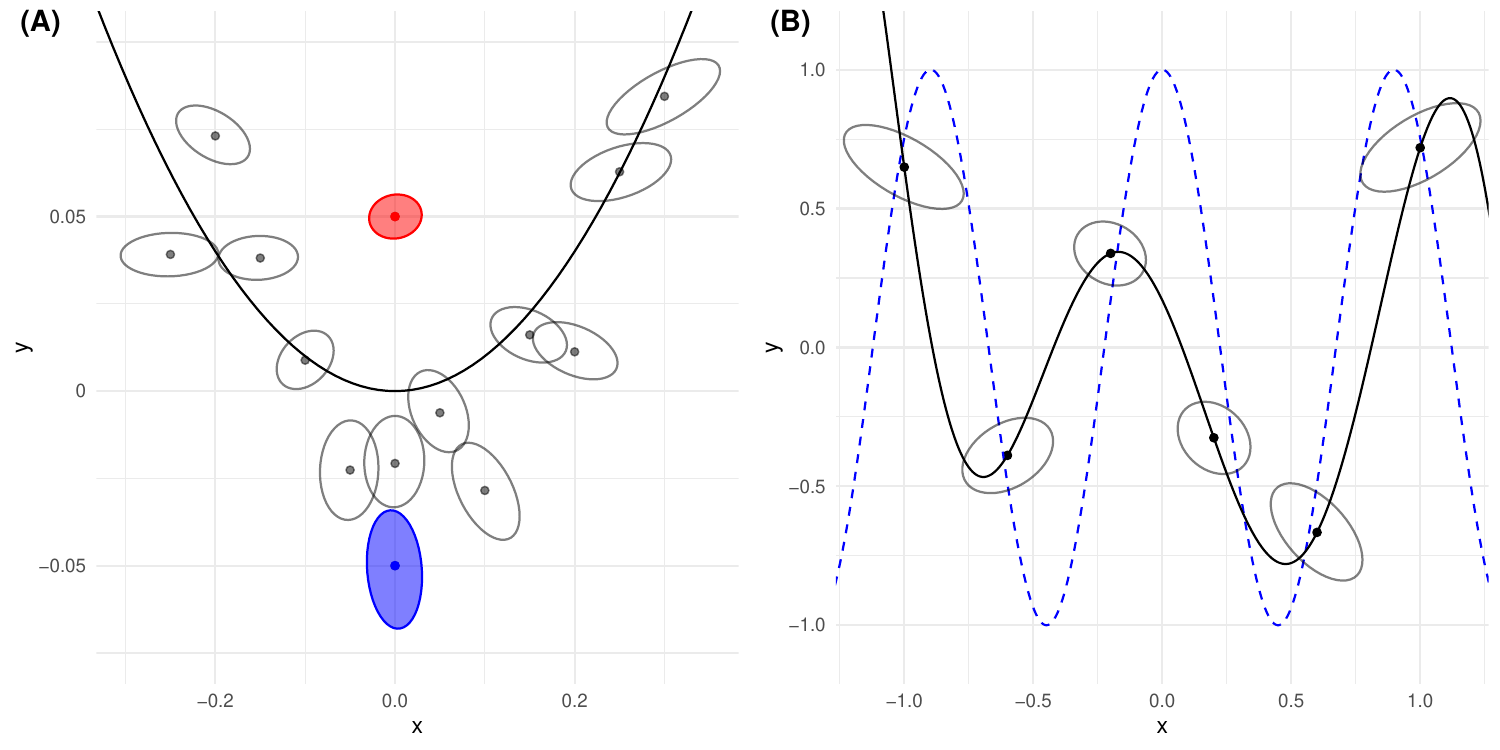}
	\caption{\footnotesize An illustration of the $XY$-Gaussian bolstering error estimator with the same data of Figure \ref{fig_ex_poly_gb}. The contribution of each data point $(X_{i},Y_{i})$ to the bolstered error estimator is the expected squared distance of $Y$ to $\psi_n(X)$ for $(X,Y)$ distributed as a Gaussian distribution with mean $(X_{i},Y_{i})$ and a covariance matrix $\Sigma_{i}$. The ellipses in \textbf{(A)} and \textbf{(B)} are level curves of the respective Gaussian distribution and illustrate the form of $\Sigma_{i}$.}
	\label{fig_ex_both}
\end{figure}

\subsection{Notation and assumptions}

In order to avoid heavy notation, we denote $\mathcal{Z} = \sX \times \sY$ and define $Z = (X,Y)$ as the random vector defined on $(\mathcal{Z},\sB_{Z})$ with probability law $\nu$. Following this notation, the sample becomes $S_{n} = \{Z_{1},\dots,Z_{n}\}$. Furthermore, for $z = (x,y) \in \mathcal{Z}$ and $\psi \in \mathcal{H}$ we denote by $\ell(z,\psi)$ or $\ell((x,y),\psi)$ the loss previously defined as $\ell(y,\psi(x))$. 

Let $C_{\ell} > 0$ be defined as
\begin{linenomath}
	\begin{equation*}
		C_{\ell} \coloneqq \sup\limits_{z \in \mathcal{Z}, \psi \in \mathcal{H}} \ell(z,\psi).
	\end{equation*}
\end{linenomath}
If $C_{\ell} < \infty$ then $\ell$ is a bounded loss function, and if $C_{\ell} = \infty$ then it is an unbounded loss function. We assume that $\ell$ has a finite moment of an order $\alpha > 1$ under $\nu$, for all $\psi \in \mathcal{H}$. Formally, denoting
\begin{linenomath}
	\begin{align*}
		L^{\alpha}(\psi) \coloneqq \int_{\mathcal{Z}} \ell^{\alpha}(z,\psi) \ d\nu(z),
	\end{align*}
\end{linenomath}
we assume there exists an $\alpha > 1$ such that 
\begin{linenomath}
	\begin{align}
		\label{finite_moments_text}
		\sup\limits_{\psi \in \mathcal{H}} L^{\alpha}(\psi) < \infty.
	\end{align}
\end{linenomath}
If $\ell$ is bounded, then \eqref{finite_moments_text} holds for all $\alpha > 1$, hence in this case, any measure $\nu$ satisfies \eqref{finite_moments_text}. When the loss function is unbounded, condition \eqref{finite_moments_text} defines a constrained class of measures.

In the next section, we state the main results of this paper about the consistency of resubstitution generalized error estimators.

\section{Consistency of generalized resubstitution error estimators}
\label{SecError}

We formally define the generalized resubstitution error estimators, extending the concept proposed by \cite{ghane2022generalized} for classification to the general statistical learning framework.

\begin{definition}
	\label{def_gen}
	Fix a hypothesis space $\mathcal{H}$, a loss function $\ell$ and a prediction rule $\Psi_{n}$. For each $n \geq 1$ and sample $S_{n} \in \mathcal{Z}^{n}$, let
	\begin{equation*}
		\mathscr{B}(S_{n}) = \{\beta_{z,\psi,S_{n}}: z \in \mathcal{Z},\psi \in \mathcal{H}\}	
	\end{equation*}
	be a collection of smoothing measures defined on $(\mathcal{Z},\sB_{Z})$ and consider the loss function $\ell_{\mathscr{B}}$ given by
	\begin{linenomath}
		\begin{equation*}
			\ell_{\mathscr{B}}(z,\psi) = \int_{\mathcal{Z}} \ell(z^{\prime},\psi) \ d\beta_{z,\psi,S_{n}}(z^{\prime})  
		\end{equation*}
	\end{linenomath}
	for $z \in \mathcal{Z}$ and $\psi \in \mathcal{H}$. We define the $\mathscr{B}$-generalized resubstitution error of $\psi_{n}$  as
	\begin{linenomath}
		\begin{equation}
			\label{gen_error}
			\hat{\varepsilon}_{n}^{\mathscr{B}} = \frac{1}{n} \sum_{i=1}^{n} \ell_{\mathscr{B}}(Z_{i},\psi_{n}) = \frac{1}{n} \sum_{i=1}^{n} \int_{\mathcal{Z}} \ell(z^{\prime},\psi_{n}) \ d\beta_{Z_{i},\psi_{n},S_{n}}(z^{\prime}).
		\end{equation}
	\end{linenomath} 
	When $\mathscr{B}(S_{n})$ depends on $S_{n}$ only through its size $n$ we denote it by $\mathscr{B}_{n}$, its measures by $\beta_{z,\psi,n}$, the generated loss by $\ell_{\mathscr{B}_{n}}$ and the estimator by $\hat{\varepsilon}_{n}^{\mathscr{B}_{n}}$.
\end{definition}

The estimator $\hat{\varepsilon}_{n}^{\mathscr{B}}$ is the standard resubstitution error of $\psi_{n}$ under the \textit{generalized loss function} $\ell_{\mathscr{B}}$ which is a smoothed version of the original loss $\ell$. For each sample $S_{n}$, the collection $\mathscr{B}(S_{n})$ generates the empirical measure
\begin{equation}
	\label{eq-emp-form}
	\nu_{n} = \frac{1}{n} \sum_{i=1}^{n} \beta_{Z_{i},\psi_{n},S_{n}}
\end{equation}
so $\hat{\varepsilon}_{n}^{\mathscr{B}}$ coincides with definition \eqref{eq-gresub2} for $\nu_{n}$ of form \eqref{eq-emp-form}. Analogous to the examples in Section \ref{SecBasicNotation}, the loss function $\ell_{\mathscr{B}}$ can be a smoothening of $\ell$ on the $X$, $Y$ or both directions by considering $\beta_{z,\psi_{n},S_{n}} = \mu_{z,\psi_{n},S_{n}} \pi_{z,\psi_{n},S_{n}}$ as the product of suitable measures on $(\sX,\sB_\sX)$ and $(\sY,\sB_\sY)$, respectively. Generalized error estimators can be defined equivalently by smoothing the empirical measure or the loss function, but, even though definition \eqref{eq-gresub2} might be more intuitive, the main results of this paper are better stated and proved considering Definition \ref{def_gen}.

A special case of generalized resubstitution error is when the collection $\mathscr{B}_{n}$ depends on $S_{n}$ only through its size. This is the case, for example, of Gaussian bolstering when each Gaussian distribution has a covariance matrix $K_{i,n}$ that depends only on the data points $X_{i}$ and $Y_{i}$, and on the sample size $n$, e.g., by converging to zero when $n$ increases. On the other hand, this is not the case for the posterior-probability error estimator when the posterior distribution is generated by the sample $S_{n}$ or for Gaussian bolstering when the kernel is estimated from $S_{n}$. 

The importance of this dependence lies on the fact that, when the smoothing measures depend on $S_{n}$ only thorough its size, the loss $\ell_{\mathscr{B}_{n}}(z,\psi)$ is not random and classical deviation-bounds of statistical learning theory can be applied to its expectation under the data generating distribution. Indeed, for each $\psi \in \mathcal{H}$, let
\begin{linenomath}
	\begin{align*}
		L^{\mathscr{B}}(\psi) = \int_{\mathcal{Z}} \ell_{\mathscr{B}}(z,\psi) \ d\nu(z) & & \text{ and } & & L^{\mathscr{B}}_{n}(\psi) = \frac{1}{n} \sum_{i=1}^{n} \ell_{\mathscr{B}}(Z_{i},\psi)
	\end{align*}
\end{linenomath}
be the expected loss $\ell_{\mathscr{B}}$ under the data generating distribution and on the sample $S_{n}$, respectively. If the collection of measures depends on $S_{n}$ only through its size, then $L^{\mathscr{B}_{n}}(\psi)$ is not a random quantity, while $L^{\mathscr{B}}_{n}(\psi)$ is always random.

\subsection{Consistency of generalized resubstitution error estimators}

In this paper, we are concerned with the consistency of generalized resubstitution error estimators, which is characterized by the convergence of the estimated error to $\varepsilon_{n} = L(\psi_n)$ with probability one as the sample size increases. Observe that $\varepsilon_n$ is a random variable since it depends on $\psi_n$.

\begin{definition}
	\label{def_consistency}
	Fix a hypothesis space $\mathcal{H}$, a loss function $\ell$, a prediction rule $\Psi_{n}$ and a collection $\mathscr{B}(S_{n})$ of probability measures for each sample $S_{n} \in \mathcal{Z}^{n}$. The generalized resubstitution error $\hat{\varepsilon}_{n}^{\mathscr{B}}$ is consistent if
	\begin{linenomath}
		\begin{align}
			\label{cond_consistency}
			\left| \hat{\varepsilon}_{n}^{\mathscr{B}} - \varepsilon_{n} \right| \to 0
		\end{align}
	\end{linenomath}
	with probability one as $n \to \infty$. 
\end{definition}

Assume without loss of generality that the random vectors in the sample $S_{n}$ are defined on a general probability space $(\Omega,\mathcal{S},\mathbb{P})$. Denote expectation in this probability space by $\mathbb{E}$. Define $\varepsilon_{n}^{\mathscr{B}} \coloneqq L^{\mathscr{B}}(\psi_{n})$ and observe that $\hat{\varepsilon}_{n}^{\mathscr{B}} = L^{\mathscr{B}}_{n}(\psi_{n})$. We assume there exists a constant $C$ such that, for all $n \geq 1$,
\begin{linenomath}
	\begin{align}
		\label{finite_moments_Q}
		\mathbb{P}\left(L^{\mathscr{B}}_{n}(\psi_{n}) < C\right) = 1.
	\end{align}
\end{linenomath}
Under assumption \eqref{finite_moments_Q}, if $\hat{\varepsilon}_{n}^{\mathscr{B}}$ is consistent, then it is asymptotically unbiased.

\begin{proposition}
	\label{proposition_unbia}
	Fix a loss function $\ell$, a prediction rule $\Psi_{n}$ and a collection $\mathscr{B}(S_{n})$ of probability measures for each sample $S_{n} \in \mathcal{Z}^{n}$. If $\hat{\varepsilon}_{n}^{\mathscr{B}}$ is consistent and \eqref{finite_moments_Q} holds, then
	\begin{linenomath}
		\begin{equation}
			\lim\limits_{n \to \infty} |\mathbb{E}[\hat{\varepsilon}_{n}^{\mathscr{B}}] - \mathbb{E}[\varepsilon_{n}]| = 0.
		\end{equation}
	\end{linenomath}
\end{proposition}

The next theorem states that the resubstitution error $\hat{\varepsilon}_{n}^{r}$ is consistent if $\mathcal{H}$ has finite VC dimension under the loss function $\ell$. See Appendix \ref{AppVCdim} for the formal definition of this VC dimension, which is also known as the pseudo-dimension \cite{pollard1989asymptotics,vapnik1998}. 

\begin{theorem}
	\label{theoremER}
	Fix a loss function $\ell$ and let $\mathcal{H}$ be a hypothesis space with $d_{VC}(\mathcal{H},\ell) < \infty$. For any prediction rule $\Psi_{n}$, the resubstitution error $\hat{\varepsilon}_{n}^{r}$ is consistent.
\end{theorem}

In Section \ref{SubSec_GenC}, we present general sufficient conditions for the consistency of generalized resubstitution errors, in Section \ref{SecSuffSize} we study the consistency when $\mathscr{B}_{n}$ depends only on the sample size, and in Section \ref{SubSec_2diff} we study the consistency when the loss function is twice differentiable in $z$.

\subsection{Sufficient conditions for consistency: general case}
\label{SubSec_GenC}

The first sufficient condition for the consistency of $\hat{\varepsilon}_{n}^{\mathscr{B}}$ is the convergence of $\hat{\varepsilon}_{n}^{\mathscr{B}}$ to $\hat{\varepsilon}_{n}^{r}$, which is a direct consequence of Theorem \ref{theoremER}. 

\begin{proposition}
	\label{prop_suf1}
	Fix a loss function $\ell$, a prediction rule $\Psi_{n}$ and a collection $\mathscr{B}(S_{n})$ of probability distributions for each sample $S_{n} \in \mathcal{Z}^{n}$. Let $\mathcal{H}$ be a hypothesis space with $d_{VC}(\mathcal{H},\ell) < \infty$. If it holds
	\begin{linenomath}
		\begin{align*}
			| \hat{\varepsilon}_{n}^{\mathscr{B}} - \hat{\varepsilon}_{n}^{r} | \to 0
		\end{align*}
	\end{linenomath}
	with probability one as $n \to \infty$, then $\hat{\varepsilon}_{n}^{\mathscr{B}}$ is consistent.
\end{proposition}

For a fixed loss function $\ell$, $\psi \in \mathcal{H}$ and $b \in (0,C_{\ell})$ define
\begin{linenomath}
	\begin{equation*}
		A_{\ell,\psi,b} = \{z \in \mathcal{Z}: \ell(z,\psi) > b\} \in \sB_{\mathcal{Z}}
	\end{equation*}
\end{linenomath}
as the points in $\mathcal{Z}$ where $\ell(z,\psi) > b$ and let 
\begin{linenomath}
	\begin{equation}
		\label{def_Astar}
		\mathcal{A}^{\star}_{\mathcal{H},\ell} = \left\{A_{\ell,\psi,b}: \psi \in \mathcal{H}, b \in (0,C_{\ell})\right\} \subset \sB_{\mathcal{Z}}
	\end{equation}
\end{linenomath}
be the collection of such sets for $\psi \in \mathcal{H}$ and $b \in (0,C_{\ell})$. An extension of Theorem 1 in \cite{ghane2022generalized} follows from Proposition \ref{prop_suf1} for bounded loss functions. 

\begin{corollary}
	\label{cor_suf2}
	Fix a bounded loss function $\ell$, a prediction rule $\Psi_{n}$ and a collection $\mathscr{B}(S_{n})$ of probability distributions for each sample $S_{n} \in \mathcal{Z}^{n}$. Let $\mathcal{H}$ be a hypothesis space with $d_{VC}(\mathcal{H},\ell) < \infty$. If
	\begin{linenomath}
		\begin{align}
			\label{eq_suf2}
			\sup\limits_{A \in \mathcal{A}^{\star}_{\mathcal{H},\ell}} \left|\frac{1}{n} \sum_{i=1}^{n} \beta_{Z_{i},\psi_{n},S_{n}}(A) - \frac{1}{n} \sum_{i=1}^{n} \delta_{Z_{i}}(A)\right| \to 0
		\end{align}
	\end{linenomath}
	with probability one as $n \to \infty$, then $\hat{\varepsilon}_{n}^{\mathscr{B}}$ is consistent. In special, if
	\begin{linenomath}
		\begin{equation*}
			\sup\limits_{\substack{A \in \mathcal{A}^{\star}_{\mathcal{H},\ell}\\ z \in \mathcal{Z}, \psi \in \mathcal{H}}} \left|\beta_{z,\psi,S_{n}}(A) - \delta_{z}(A)\right| \to 0
		\end{equation*}
	\end{linenomath}
	with probability one as $n \to \infty$, then $\hat{\varepsilon}_{n}^{\mathscr{B}}$ is consistent.
\end{corollary}

If $\ell_{\mathscr{B}}(z,\psi)$ converges to $\ell(z,\psi)$ uniformly on $\mathcal{Z}$ and $\mathcal{H}$, then $\hat{\varepsilon}_{n}^{\mathscr{B}}$ is consistent. This is a sufficient condition that also holds for unbounded loss functions.

\begin{corollary}
	\label{cor_suf3}
	Fix a loss function $\ell$, a prediction rule $\Psi_{n}$, a collection $\mathscr{B}(S_{n})$ of probability distributions for each sample $S_{n} \in \mathcal{Z}^{n}$, and let $\mathcal{H}$ be a hypothesis space with $d_{VC}(\mathcal{H},\ell) < \infty$. If
	\begin{linenomath}
		\begin{align*}
			\lim\limits_{n \to \infty} \sup\limits_{z \in \mathcal{Z},\psi \in \mathcal{H}} \left|\ell_{\mathscr{B}}(z,\psi) - \ell(z,\psi)\right| = 0,
		\end{align*}
	\end{linenomath}
	with probability one, then $\hat{\varepsilon}_{n}^{\mathscr{B}}$ is consistent.
\end{corollary}

\subsection{Sufficient conditions for consistency: dependence only on sample size}
\label{SecSuffSize}

When the collection $\mathscr{B}_{n}$ depends on the sample $S_{n}$ only through its size $n$, then
\begin{linenomath}
	\begin{align*}
		\hat{\varepsilon}_{n}^{\mathscr{B}_{n}} = \frac{1}{n} \sum_{i=1}^{n} \ell_{\mathscr{B}_{n}}(Z_{i},\psi_{n})
	\end{align*}
\end{linenomath}
is the resubstitution error of $\psi_{n}$ considering the loss function $\ell_{\mathscr{B}_{n}}$. It follows from Theorem \ref{theoremER} that $\hat{\varepsilon}_{n}^{\mathscr{B}_{n}}$ converges with probability one to $\varepsilon_{n}^{\mathscr{B}_{n}}$ if $\limsup d_{VC}(\mathcal{H},\ell_{\mathscr{B}_{n}}) < \infty$. Therefore, if $\varepsilon_{n}^{\mathscr{B}_{n}}$ converges to $\varepsilon_{n}$, then $\hat{\varepsilon}_{n}^{\mathscr{B}_{n}}$ converges to $\varepsilon_{n}$, so it is consistent. This result is stated in Proposition \ref{prop_condExchange}. In Appendix \ref{AppVClQ} we present sufficient conditions for $\limsup d_{VC}(\mathcal{H},\ell_{\mathscr{B}_{n}})$ to be finite when $d_{VC}(\mathcal{H},\ell) < \infty$.

\begin{proposition}
	\label{prop_condExchange}
	Fix a loss function $\ell$, a prediction rule $\Psi_{n}$ and a collection $\mathscr{B}_{n}$ of probability measures for each $n \geq 1$. Let $\mathcal{H}$ be a hypothesis space with $\limsup_{n \to \infty}d_{VC}(\mathcal{H},\ell_{\mathscr{B}_{n}}) < \infty$. If
	\begin{linenomath}
		\begin{equation*}
			\left|\varepsilon_{n}^{\mathscr{B}_{n}} - \varepsilon_{n} \right| \to 0
		\end{equation*}
	\end{linenomath}
	with probability one as $n \to \infty$, then $\hat{\varepsilon}_{n}^{\mathscr{B}_{n}}$ is consistent.
\end{proposition}

As a consequence of Proposition \ref{prop_condExchange}, the next proposition gives a sufficient condition when $\beta_{z,\psi,n}$ is absolutely continuous wrt $\nu$.

\begin{proposition}
	\label{prop_symmetric}
	Fix a loss function $\ell$, a prediction rule $\Psi_{n}$, a collection $\mathscr{B}_{n}$ of probability measures for each $n \geq 1$, and let $\mathcal{H}$ be a hypothesis space with $\limsup_{n \to \infty}d_{VC}(\mathcal{H},\ell_{\mathscr{B}_{n}}) < \infty$. If
	\begin{itemize}
		\item[(a)] For all $z \in \mathcal{Z}, \psi \in \mathcal{H}$ and $n \geq 1$, $\beta_{z,\psi,n}$ is absolutely continuous with respect to the data generating measure $\nu$. Denote by $\rho_{z,\psi,n} = \frac{d\beta_{z,\psi,n}}{d\nu}$ a version of the respective Radon-Nikodym derivative;
		\item[(b)] For all $\psi \in \mathcal{H	}$, $n \geq 1$ and $z^{\prime} \in \mathcal{Z}$  it holds
		\begin{linenomath}
			\begin{equation*}
				\int_{\mathcal{Z}}  \rho_{z,\psi,n}(z^{\prime})  \ d\nu(z) = 1;
			\end{equation*}
		\end{linenomath}
	\end{itemize}
	then  $\hat{\varepsilon}_{n}^{\mathscr{B}_{n}}$ is consistent.
\end{proposition}

It follows from Proposition \ref{prop_symmetric} that if $\rho_{z,\psi,n}(z^{\prime})$ is a symmetric function of $(z,z^{\prime})$ with $\nu$-probability one, then $\hat{\varepsilon}_{n}^{\mathscr{B}_{n}}$ is consistent.

\begin{corollary}
	\label{corr_symmetric}
	Fix a loss function $\ell$, a prediction rule $\Psi_{n}$, a collection $\mathscr{B}_{n}$ of probability measures for each $n \geq 1$, and let $\mathcal{H}$ be a hypothesis space with $\limsup_{n \to \infty}d_{VC}(\mathcal{H},\ell_{\mathscr{B}_{n}}) < \infty$. If (a) in Proposition \ref{prop_symmetric} holds and
	\begin{linenomath}
		\begin{equation}
			\label{cond-prop-symm}
			(\nu \times \nu)\left(\{(z,z^{\prime}) \in \mathcal{Z}^{2}: \rho_{z,\psi,n}(z^{\prime}) = \rho_{z^\prime,\psi,n}(z)\}\right) = 1
		\end{equation}
	\end{linenomath}
	for all $\psi \in \mathcal{H}$ and all $n \geq 1$, then  $\hat{\varepsilon}_{n}^{\mathscr{B}_{n}}$ is consistent.
\end{corollary}

Corollary \ref{corr_symmetric} can be applied to show the consistency of Gaussian bolstering when the covariance matrix of the Gaussian distributions is the same in all sample points.

\begin{example}[Gaussian bolstering]
	\normalfont \label{gaussianBols_symmetric}
	
	Assume that $X$ is absolutely continuous with respect to the Lebesgue measure in $\mathbb{R}^{d}$ and consider the Gaussian bolstering error estimator 
	\begin{linenomath}
		\begin{align}
			\label{gaus_br}
			\hat{\varepsilon}_{n}^{gbr} &\coloneqq \frac{1}{n} \sum_{i=1}^{n} \int_{\mathbb{R}^{d}} \ell((x,Y_{i}),\psi) \ p_{X_{i},\psi_{n},n}(x) \ dx
		\end{align}
	\end{linenomath}
	in which $p_{X_{i},\psi_{n},n}$ is the density function of a Gaussian distribution with mean $X_{i}$ and covariance matrix $K_{n,\psi_{n}}$, which may depend on $n$ and $\psi_n$. In this case, $\mathscr{B}_{n} = \{\beta_{x,y,\psi,n}: (x,y) \in \mathcal{Z}, \psi \in \mathcal{H}\}$ is a collection of measures $\beta_{x,y,\psi,n} = \mu_{x,\psi,n} \delta_{y}$ in which $\mu_{x,\psi,n}$ is a Gaussian measure with mean $x$ and covariance matrix $K_{n,\psi}$.
	
	The Gaussian bolstering error estimator \eqref{gaus_br} is consistent if $$\limsup_{n} d_{VC}(\mathcal{H},\ell_{\mathscr{B}_{n}}) < \infty,$$ a condition which holds, for example, for linear regression under the quadratic loss and in classification problems under binary loss functions if
	\begin{linenomath}
		\begin{equation*}
			\limsup_{n \to \infty} |\ell_{\mathscr{B}_{n}}(z,\psi) - \ell(z,\psi)| < 1/2
		\end{equation*}
	\end{linenomath}
	for all $z \in \mathcal{Z}$ and $\psi \in \mathcal{H}$ (cf. Appendix \ref{AppVClQ}).
	
	\begin{proposition}
		\label{prop_gaus_br}
		If $\limsup_{n \to \infty}d_{VC}(\mathcal{H},\ell_{\mathscr{B}_{n}}) < \infty$, then the Gaussian bolstering error estimator $\hat{\varepsilon}_{n}^{gbr}$ is consistent.
	\end{proposition}
	
	Proposition \ref{prop_gaus_br} establishes the consistency of the Gaussian bolstering error estimator in classification and regression scenarios in which the kernel does not necessarily converge to zero. This is an improvement from the results in \cite{ghane2022generalized} which required the kernel to converge to zero. This proposition also holds for $XY$-Gaussian bolstering.
	
	\hfill$\blacksquare$
\end{example}

\subsection{Sufficient conditions for consistency: twice differentiable loss functions}
\label{SubSec_2diff}

More specific sufficient conditions for the consistency of generalized resubstitution errors may be obtained under the assumption that $\ell((x,y),\psi)$ is twice differentiable in $(x,y)$ and its second derivatives are uniformly bounded on $\mathcal{X} \times \mathcal{Y}$ and $\mathcal{H}$. Important examples of this case consider $\ell$ as the quadratic loss function and $\psi(x)$ as twice differentiable functions.

Denote by $Z_{z,\psi,S_{n}}$ a random variable with a probability law $\beta_{z,\psi,S_{n}}$ for $z \in \mathcal{Z}$, $\psi \in \mathcal{H}$ and $S_{n} \in \mathcal{Z}^{n}$. If $\ell(z,\psi)$ is twice differentiable in $z$ and its second derivatives are uniformly bounded in $\mathcal{Z}$ and $\mathcal{H}$, then $\hat{\varepsilon}_{n}^{\mathscr{B}}$ is consistent if the expectation of $Z_{z,\psi,S_{n}}$ is $z$ and if the variance of its coordinates converges to zero as $n$ increases with probability one over the possible samples. This result is a consequence of Corollary \ref{cor_suf3}. Recall that $Z = (X,Y)$ is a random vector with $d + m$ coordinates.

\begin{proposition}
	\label{Prop_2der}
	Fix a loss function $\ell$, a prediction rule $\Psi_{n}$, a collection $\mathscr{B}(S_{n})$ of probability distributions for each sample $S_{n} \in \mathcal{Z}^{n}$, and let $\mathcal{H}$ be a hypothesis space with $d_{VC}(\mathcal{H},\ell) < \infty$. If
	\begin{itemize}
		\item[(a)] $\ell(\cdot,\psi) \in C^{2}(\mathcal{Z})$ for all $\psi \in \mathcal{H}$ and there exists a constant $C_{2}$ such that
		\begin{linenomath}
			\begin{equation}
				\label{bounded_dev2}
				\sup\limits_{z \in \mathcal{Z},\psi \in \mathcal{H}} \left|\frac{\partial^{2}\ell}{\partial z_{i} \partial z_{j}}(z,\psi)\right| < C_{2}
			\end{equation}
		\end{linenomath}
		for all $i,j = 1,\dots,d + m$;
		\item[(b)] For all $z \in \mathcal{Z}, \psi \in \mathcal{H}$ and $S_{n} \in \mathcal{Z}_{n}$ it holds $\mathbb{E}(Z_{z,\psi,S_{n}}) = z$ and \[\lim\limits_{n \to \infty} \sup\limits_{z \in \mathcal{Z},\psi \in \mathcal{H}} Var\left([Z_{z,\psi,S_{n}}]_{j}\right) = 0\] for all $j = 1,\dots,d + m$ with probability one;
	\end{itemize}
	then  $\hat{\varepsilon}_{n}^{\mathscr{B}}$ is consistent.
\end{proposition}

We extend Theorem 2 of \cite{ghane2022generalized} for Gaussian bolstering for regression under the quadratic loss function when $\psi$ is twice differentiable and has derivatives uniformly bounded on $\mathcal{X}$ and $\mathcal{H}$. In this case, the Gaussian bolstering error estimator reduces to
\begin{linenomath}
	\begin{align}
		\label{gaus_error2}
		\hat{\varepsilon}_{n}^{gbr} &\coloneqq \frac{1}{n} \sum_{i=1}^{n} \int_{\mathbb{R}^{d}} \left[\psi_{n}(x) - Y_{i}\right]^{2} p_{X_{i},Y_{i},\psi_n,S_{n}}(x) \ dx,
	\end{align}
\end{linenomath}
in which $p_{X_{i},Y_{i},\psi_n,S_{n}}$ is the density function of a Gaussian distribution with mean $X_{i}$ and covariance matrix $K_{X_{i},Y_{i},\psi_n,S_{n}}$. The next result presents a condition for the Gaussian bolstering error estimator to be consistent when $\psi$ is twice differentiable. The result also holds for $XY$-Gaussian bolstering.

\begin{proposition}
	\label{prop_gaus_brR}
	Fix a prediction rule $\Psi_{n}$ and a collection $\mathscr{B}(S_{n})$ of probability distributions for each sample $S_{n} \in \mathcal{Z}^{n}$. Let $\ell$ be the quadratic loss function, let $\mathcal{Y}$ be a compact set of $\mathbb{R}$ and let $\mathcal{H} \subset C^{2}(\mathcal{X},\mathcal{Y})$ be a set of twice differentiable functions with $d_{VC}(\mathcal{H},\ell) < \infty$. If
	\begin{linenomath}
		\begin{align}
			\label{cond_var_zero}
			\lim\limits_{n \to \infty} \sup\limits_{x \in \mathcal{X},y \in \mathcal{Y}, \psi \in \mathcal{H}} [K_{x,y,\psi,S_{n}}]_{j,j} = 0,
		\end{align}
	\end{linenomath}
	with probability one, and
	\begin{linenomath}
		\begin{align*}
			\sup\limits_{x \in \mathcal{X}, \psi \in \mathcal{H}} \left|\frac{\partial\psi}{\partial x_{i}}(x)\right| < \infty & & \text{ and } & & \sup\limits_{x \in \mathcal{X}, \psi \in \mathcal{H}} \left|\frac{\partial^{2}\psi}{\partial x_{i} \partial x_{j}}(x)\right| < \infty
		\end{align*}
	\end{linenomath}
	for all $i,j = 1,\dots,d$, then the Gaussian bolstering error estimator $\hat{\varepsilon}_{n}^{gbr}$ defined in \eqref{gaus_error2} is consistent.
\end{proposition}

Although Proposition \ref{prop_gaus_brR} was stated for Gaussian bolstering, it holds for any bolstered error estimator by considering a probability density $p_{x,y,\psi,S_{n}}$ with mean $x$ and covariance matrix $K_{x,y,\psi,S_{n}}$ that converges to zero with probability one when $n$ increases. 

\section{Estimation of the kernel in Gaussian bolstering}
\label{SecHeu}

The results in Section \ref{SecError} outline a range of instances in which the Gaussian bolstering error estimator is consistent, allowing room to choose the covariance matrix (kernel) based on the training data $S_{n}$. In this section, we propose estimators for the Gaussian bolstering kernel based on the method of moments and on the maximization of a pseudo-likelihood function.  The proposed estimators can also be used to estimate the kernel for Gaussian bolstering classification error estimation and in $XY$-Gaussian bolstering by applying them to the vector $(X,Y)$ instead of $X$. Throughout this section, we consider that the sample $S_{n}$ is fixed and that $\mathcal{X} = \mathbb{R}^{d}$.

\subsection{Method of moments}

We assume that the kernel has the form $K_{x,y,\psi,S_{n}} = \sigma_{S_{n}}^{2} \Sigma$ for a known fixed matrix $\Sigma$, and propose a method of moments estimator for $\sigma_{S_{n}}$. We also consider the case in which $K_{x,y,\psi,S_{n}} = \sigma_{y,S_{n}}^{2} \Sigma$ and estimate $\sigma_{y,S_{n}}$ for each class $y \in \mathcal{Y}$ in a classification problem. The method of moments estimator is obtained following ideas analogous to that of \cite{braga2004bolstered}.

For $x,x^{\prime} \in \mathcal{X}$ denote by
\begin{linenomath}
	\begin{align*}
		\delta(x,x^{\prime}) = \sqrt{(x - x^{\prime})^{T} \ \Sigma^{-1} \ (x - x^{\prime})}
	\end{align*}
\end{linenomath}
the \textit{Mahalanobis distance} from $x$ to $x^{\prime}$ and let
\begin{linenomath}
	\begin{align*}
		\delta(x) = \min\limits_{i} \ \{\delta(x,X_{i})\} & & \delta_{y}(x) = \min\limits_{i: Y_{i} = y} \  \{\delta(x,X_{i})\}
	\end{align*}
\end{linenomath}
be the minimum distance from $x$ to any $X_{i}$ and to an $X_{i}$ such that $Y_{i} = y \in \mathcal{Y}$, respectively. If there is no point in the sample with $Y_{i} = y$ we denote $\delta_{y}(x) = \infty$. When $\Sigma = I_d$ these distances reduce to the respective Euclidean distance.

On the one hand, the error $L(\psi_{n})$ can be interpreted as the expected loss of $\psi_{n}$ on a test point $(X,Y)$ generated by $\nu$. The random variable representing the distances above of this test point to the sample has a marginal and conditional cumulative distribution given $Y = y$, respectively,
\begin{linenomath}
	\begin{align*}
		D(\delta) = \nu\left(\delta(X) \leq \delta\right) & & \text{ and } & & D_{y}(\delta) = \frac{\nu\left(\delta_{y}(X) \leq \delta,Y = y\right)}{\nu(Y = y)}
	\end{align*}
\end{linenomath}
for $\delta > 0$ and $y \in \mathcal{Y}$ with $\nu(Y = y) > 0$.

On the other hand, $\hat{\varepsilon}_{n}^{gbr} = L_{n}^{\mathscr{B}}(\psi_{n})$ can be interpreted as the expected loss of $\psi_{n}$ on a test point $(\hat{X},\hat{Y})$ generated by the empirical law $\nu_{n}$ with
\begin{linenomath}
	\begin{equation*}
		\nu_{n} = \frac{1}{n} \sum_{i=1}^{n} \mu_{X_{i},Y_{i},S_{n}} \delta_{Y_{i}}
	\end{equation*}
\end{linenomath}
in which $\mu_{X_{i},Y_{i},S_{n}}$ is a Gaussian measure with mean $X_{i}$ and covariance matrix of form $\sigma_{Y_{i},S_{n}}^{2} \Sigma$, in which $\sigma_{Y_{i},S_{n}}$ may depend on $Y_{i}$. The random variable representing the above distances of this test point to the sample has marginal and conditional cumulative distribution given $\hat{Y} = y$, respectively,
\begin{linenomath}
	\begin{align*}
		\hat{D}(\delta) = \nu_{n}\left(\delta(\hat{X}) \leq \delta\right) & & \text{ and } & & \hat{D}_{y}(\delta) = \frac{\nu_{n}\left(\delta_{y}(\hat{X}) \leq \delta,\hat{Y} = y\right)}{\nu_{n}(\hat{Y} = y)}
	\end{align*}
\end{linenomath}
for $\delta > 0$ and $y = Y_{i}$ for some $i = 1,\dots,n$.

Method of moments estimators can be obtained for the free parameter $\sigma_{S_{n}}^{2}$ by assuming it does not depend on $y$ and then equating the sample mean of $D$ with the mean of $\hat{D}$. In classification problems, we can also equate the sample mean of $D_{y}$ and the mean of $\hat{D}_{y}$ and solve for a $\sigma_{y,S_{n}}^{2}$ for each class $y$ to obtain estimators that depend on the class $y$. 

The sample mean of $D$ and the sample mean of $D_{y}$ are, respectively,
\begin{linenomath}
	\begin{align}
		\label{sample_distance_mean}
		\bar{\delta}_{S_n} &= \frac{1}{n} \sum_{i=1}^{n} \min_{j \neq i} \ \{\delta(X_{i},X_{j})\} \text{ and} \\ \bar{\delta}_{y,S_n} &= \frac{1}{\sum_{i=1}^{n} \mathds{1}_{Y_{i} = y}} \sum_{i=1}^{n} \min\limits_{j \neq i,Y_{j} = Y_{i}} \ \{\delta(X_{i},X_{j})\} \ \mathds{1}_{Y_{i} = y},
	\end{align}
\end{linenomath}
that are the mean distance of the points in $S_{n}$ and the points in $S_{n}$ with $Y_{i}$ = y, respectively, to the respective sample $S_{n}\setminus\{X_{i},Y_{i}\}$. The next lemma presents the mean of $\delta(\hat{X})$ and of $\delta_{y}(\hat{X})|\hat{Y} = y$, that is the mean of $\hat{D}$ and $\hat{D}_{y}$, respectively. Denote by $p_{i,S_{n}}$ and $p_{i,y,S_{n}}$ the densities of Gaussian distributions with mean $X_{i}$ and covariance matrices $\sigma_{S_{n}}^{2}\Sigma$ and $\sigma_{y,S_{n}}^{2}\Sigma$, respectively.

\begin{lemma}
	\label{lemma_mean_distance}
	For all $y \in \mathcal{Y}$ such that $Y_{i} = y$ for some $i = 1,\dots,n$,
	\begin{linenomath}
		\begin{align*}
			\mathbb{E}\left[\delta_{y}(\hat{X})|\hat{Y} = y\right] &=  \frac{1}{\sum_{i=1}^{n} \mathds{1}_{Y_{i} = y}} \sum_{i: Y_{i} = y} \sum_{j: Y_{j} = y} \int_{\sX} \delta(x,X_{j}) \ p_{i,y,S_{n}}(x) \ \mathds{1}_{\delta(x) = \delta(x,X_{j})} \ dx \\
			\mathbb{E}\left[\delta(\hat{X})\right] &= \frac{1}{n} \sum_{i,j=1}^{n} \int_{\sX} \delta(x,X_{j}) \ p_{i,S_{n}}(x) \ \mathds{1}_{\delta(x) = \delta(x,X_{j})} \ dx.			
		\end{align*}
	\end{linenomath}
\end{lemma}

A method of moments estimator for $\sigma_{S_{n}}$ and $\sigma_{y,S_{n}}$ may be obtained by solving the equations
\begin{linenomath}
	\begin{align}
		\label{mm_equation}
		\mathbb{E}\left[\delta(\hat{X})\right] = \bar{\delta}_{S_n} & & \text{ and } & & \mathbb{E}\left[\delta_{y}(\hat{X})|\hat{Y} = y\right] = \bar{\delta}_{y,S_n}
	\end{align}
\end{linenomath} 
on $\sigma_{S_{n}}$ and $\sigma_{y,S_{n}}$, respectively. For lower-dimensional data, the integrals in Lemma \ref{lemma_mean_distance} may be solved exactly or via Monte Carlo integration. In Figure \ref{fig_kernel_exact} we illustrate the estimation of $\sigma_{S_{n}}$ by solving \eqref{mm_equation} via a grid search.

\begin{figure}[ht]
	\centering
	\includegraphics[width=\linewidth]{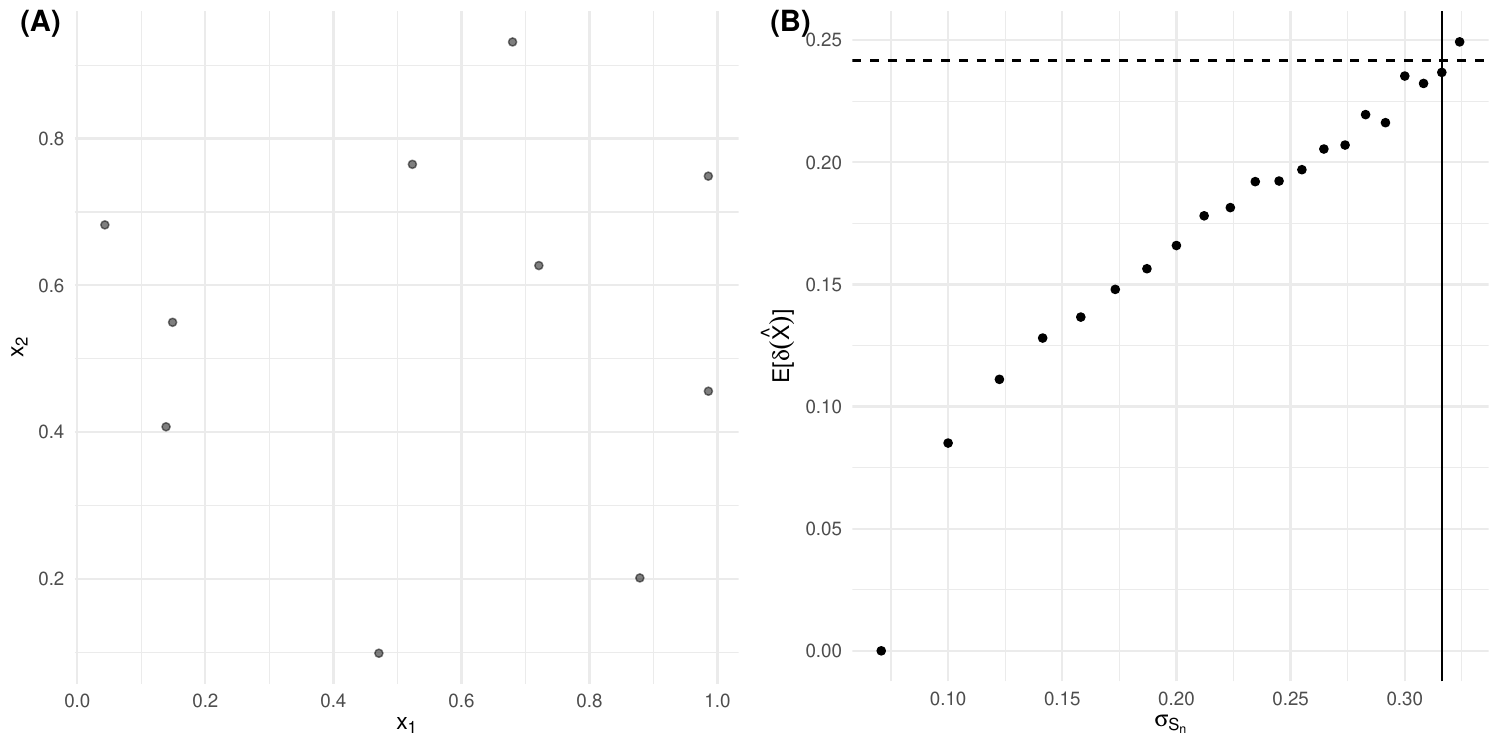}
	\caption{\footnotesize Illustration of the estimation of $\sigma_{S_{n}}$ by the method of moments. \textbf{(A)} A sample of $X$ in $d = 2$ dimensions with $\bar{\delta}_{S_n} = 0.24$. \textbf{(B)} A grid search to estimate $\sigma_{S_{n}}$ by solving equation \eqref{mm_equation} considering $\Sigma = I_{d}$. The expectation $\mathbb{E}[\delta(\hat{X})]$ is computed via Monte Carlo integration for $\sigma_{S_{n}}$ in a grid and the estimate $\hat{\sigma}_{S_{n}}$ is that with $\mathbb{E}[\delta(\hat{X})]$ closer to $\bar{\delta}_{S_n}$ that was $\hat{\sigma}_{S_{n}} = 0.31$.} \label{fig_kernel_exact}
\end{figure}

Since solving \eqref{mm_equation} may be computationally complex, specially for high-dimensional data, we propose an approximation for the expectations in Lemma \ref{lemma_mean_distance} that implies estimators that are analogous to that proposed by \cite{braga2004bolstered}. The next proposition states that, when $\sigma_{S_{n}}$ converges to zero, $\sigma_{S_{n}}^{-1} \mathbb{E}[\delta(\hat{X})]$ converges to the expectation of a \textit{chi} random variable with $d$ degrees of freedom, recalling that $d$ is the dimension of $X$. The same result follows for $\sigma_{y,S_{n}}^{-1} \mathbb{E}[\delta_{y}(\hat{X})|\hat{Y} = y]$.

\begin{proposition}
	\label{lemma_chi_approx}
	Let $\chi$ be a random variable with a chi distribution with $d$ degrees of freedom. Then,
	\begin{linenomath}
		\begin{align*}
			\frac{\mathbb{E}[\delta(\hat{X})]}{\sigma_{S_{n}}} \leq \mathbb{E}[\chi] & & \text{ and } & & \lim\limits_{\sigma_{S_{n}} \to 0} \frac{\mathbb{E}[\delta(\hat{X})]}{\sigma_{S_{n}}} = \mathbb{E}[\chi]
		\end{align*}
	\end{linenomath}
	for any fixed sample $S_{n} \in \mathcal{Z}^{n}$.
\end{proposition}

Proposition \ref{lemma_chi_approx} yields the approximate method of moments estimators
\begin{linenomath}
	\begin{align}
		\label{mm_approx_estimators}
		\hat{\sigma}_{S_{n}} = \frac{\bar{\delta}_{S_n}}{\mathbb{E}[\chi]} & & \hat{\sigma}_{y,S_{n}} = \frac{\bar{\delta}_{y,S_n}}{\mathbb{E}[\chi]}.
	\end{align}
\end{linenomath}
Since $\mathbb{E}[\chi]$ is an upper bound for $\sigma_{S_{n}}^{-1} \mathbb{E}[\delta(\hat{X})]$, for any value of $\sigma_{S_{n}}$, the estimators in \eqref{mm_approx_estimators} are upper bounds for those obtained by solving \eqref{mm_equation}, hence this approximation yields systematically more bolstering. 

\begin{remark}
	In \cite{braga2004bolstered} it was proposed to divide by the median of a $\chi$ distribution in \eqref{mm_approx_estimators}. The median for $d$ from $1$ to $5$ is, respectively, around $0.67$, $1.17$, $1.53$, $1.83$ and $2.08$, which are values close to the mean, that is around $0.79$, $1.25$, $1.59$, $1.88$ and $2.12$ for $d$ from $1$ to $5$, respectively. Therefore, exchanging them in \eqref{mm_approx_estimators} should not have a great impact.
\end{remark}

\begin{remark}
	Equation \eqref{mm_equation} can be solved to estimate any kernel $K_{S_{n}}$ that depends on only one unknown parameter, even if it is not of the form $\sigma_{S_{n}} \Sigma$.
\end{remark}

\begin{remark}
	Other distributions beyond Gaussian could be considered, and the deductions remains true by exchanging the densities $p_{i,S_{n}}$ and $p_{i,y,S_{n}}$, and the $\chi$ distribution, for the respective densities and distribution of the distance to the mean.
\end{remark}

\subsection{Maximum pseudo-likelihood estimator}

In this section, we assume that each point $(X_{i},Y_{i})$ has a distinct kernel $\Sigma_{i,S_{n}}$, that is a general positive-definite matrix, and we propose an estimator for these $n$ matrices based on a pseudo-likelihood function. We denote these matrices simply by $\Sigma_{i}$ to easy notation, and it should be implicit that they depend on the sample $S_{n}$.

We consider Gaussian bolstering with kernels $\Sigma_{i} = (n-1)^{-1} \tilde{\Sigma}_{i}$ which are obtained by maximizing
\begin{linenomath}
	\begin{equation}
		\label{pseudo_likelihood}
		\mathcal{L}_{S_{n}}(\Sigma_{1},\dots,\Sigma_{n},\pi_{1},\dots,\pi_{n}) = \prod_{i=1}^{n} \left(\sum_{j=1}^{n} \pi_{j} \ p_{j,\Sigma_{j}}(X_{i}) \ \mathds{1}_{i \neq j}\right)
	\end{equation}
\end{linenomath} 
in which $p_{j,\Sigma_{j}}$ is the density of a Gaussian distribution with mean $X_{j}$ and covariance matrix $\Sigma_{j}$, and $0 \leq \pi_{j} \leq 1, j = 1,\dots,n,$ are probabilities that sum one. 

The function $\mathcal{L}_{S_{n}}$ is a \textit{pseudo-likelihood} function of sample $S_{n}$ that approximates the probability density function of $X_{i}$ by a Gaussian mixture with means $X_{j}$ and covariance matrices $\Sigma_{j}, j = 1,\dots,n, j \neq i$. A maximum pseudo-likelihood estimator for $\Sigma_{1},\dots,\Sigma_{n}$ is obtained by maximizing \eqref{pseudo_likelihood}. 

The EM algorithm can be applied to obtain the kernels that maximize \eqref{pseudo_likelihood}. The algorithm is an adaptation of the classical EM algorithm for the mixture of Gaussian distributions \citep{dempster1977maximum}, in which the means are known. We present the steps of the algorithm in Algorithm \ref{A1}, where we denote by $\boldsymbol{\Sigma} = \{\Sigma_{1},\dots,\Sigma_{n}\} \coloneqq \{(n - 1)^{-1}\tilde{\Sigma}_{1},\dots,(n - 1)^{-1}\tilde{\Sigma}_{n}\}$ a collection of kernels.

We consider a hyperparameter $\lambda$ in the E-step by updating the weights to
\begin{linenomath}
	\begin{equation*}
		w_{i,j}^{(t)} = \frac{(\lambda + p_{i,\Sigma_{i}^{(t)}}(X_{j})) \ \mathds{1}_{i \neq j}}{\lambda (n - 1) + \sum_{i=1}^{n} p_{i,\Sigma_{i}^{(t)}}(X_{j}) \ \mathds{1}_{i \neq j}}\,
	\end{equation*}
\end{linenomath}
to obtain a more stable solution. On the one hand, if $\lambda = 0$, then the solution is unstable, since $p_{i,\Sigma_{i}^{(t)}}(X_{j})$ can be very small for all values of $j \neq i$ and the weights might not converge. On the other hand, when $\lambda \to \infty$ the matrices converge to
\begin{linenomath}
	\begin{equation*}
		\Sigma_{i} = \frac{1}{(n-1)^{2}} \sum_{j} (X_{j} - X_{i})(X_{j} - X_{i})^{T} \, \mathds{1}_{i \neq j}
	\end{equation*}
\end{linenomath}
that is $(n-1)^{-1}$ times the mean distance matrix from each point in the sample to $X_{i}$. In this paper we consider $\lambda = 1$, so the weights are a perturbation of $1/(n - 1)$ by the Gaussian densities. We call the estimator $\boldsymbol{\hat{\Sigma}}$ obtained via the EM algorithm as the \textit{maximum pseudo-likelihood estimator} (MPE).

\begin{algorithm}[ht]
	\centering
	\caption{EM algorithm for kernel estimation in Gaussian bolstering.}
	\label{A1}
	\begin{algorithmic}[1]
		\STATE \textbf{Initialize}: Kernels $\boldsymbol{\Sigma}^{(0)}, t = 0, \lambda > 0$ and $\epsilon_{c} > 0$.
		\WHILE{$\max\limits_{i} \ \lVert \Sigma_{i}^{(t-1)}  - \Sigma_{i}^{(t)} \rVert_{\infty} \geq \epsilon_{c}$}
		\STATE \textbf{E-step}: Compute the conditional probabilities
		\begin{linenomath}
			\begin{equation*}
				w_{i,j}^{(t)} = \frac{(\lambda + p_{i,\Sigma_{i}^{(t)}}(X_{j})) \ \mathds{1}_{i \neq j}}{\lambda (n - 1) + \sum_{i=1}^{n} p_{i,\Sigma_{i}^{(t)}}(X_{j}) \ \mathds{1}_{i \neq j}}
			\end{equation*}
		\end{linenomath}
		\STATE \textbf{M-step}: Update the kernels as
		\begin{linenomath}
			\begin{equation*}
				\Sigma_{i}^{(t+1)} = \frac{1}{n-1} \ \sum_{j} w_{i,j}^{(t)} (X_{j} - X_{i})(X_{j} - X_{i})^{T}
			\end{equation*}
		\end{linenomath}	
		\STATE t = t + 1
		\ENDWHILE	
	\end{algorithmic}
\end{algorithm}

Figure \ref{fig_ex_kernel_estimation} presents an example of the method of moments (MM) estimators, exact and approximated, with spherical matrix $\Sigma = I_{d}$, and the MPE for Gaussian and $XY$-Gaussian bolstering. We see that the exact chi approximation kernel is wider than the MM exact kernel, since the chi approximation is an upper bound of it. Furthermore, the MPE kernel is way smaller (in the matrix trace sense) than those obtained via the MM, and it \textit{points} to the direction where the other data points are in average \textit{farther} from the point. This means that the directions with more probability are those in which there are less sample points close to the $i$-th point. This makes sense, as these are the directions less represented in the sample, where the loss of $\psi_{n}$ is not evaluated through the standard resubstitution.

\begin{figure}[ht]
	\centering
	\includegraphics[width=\linewidth]{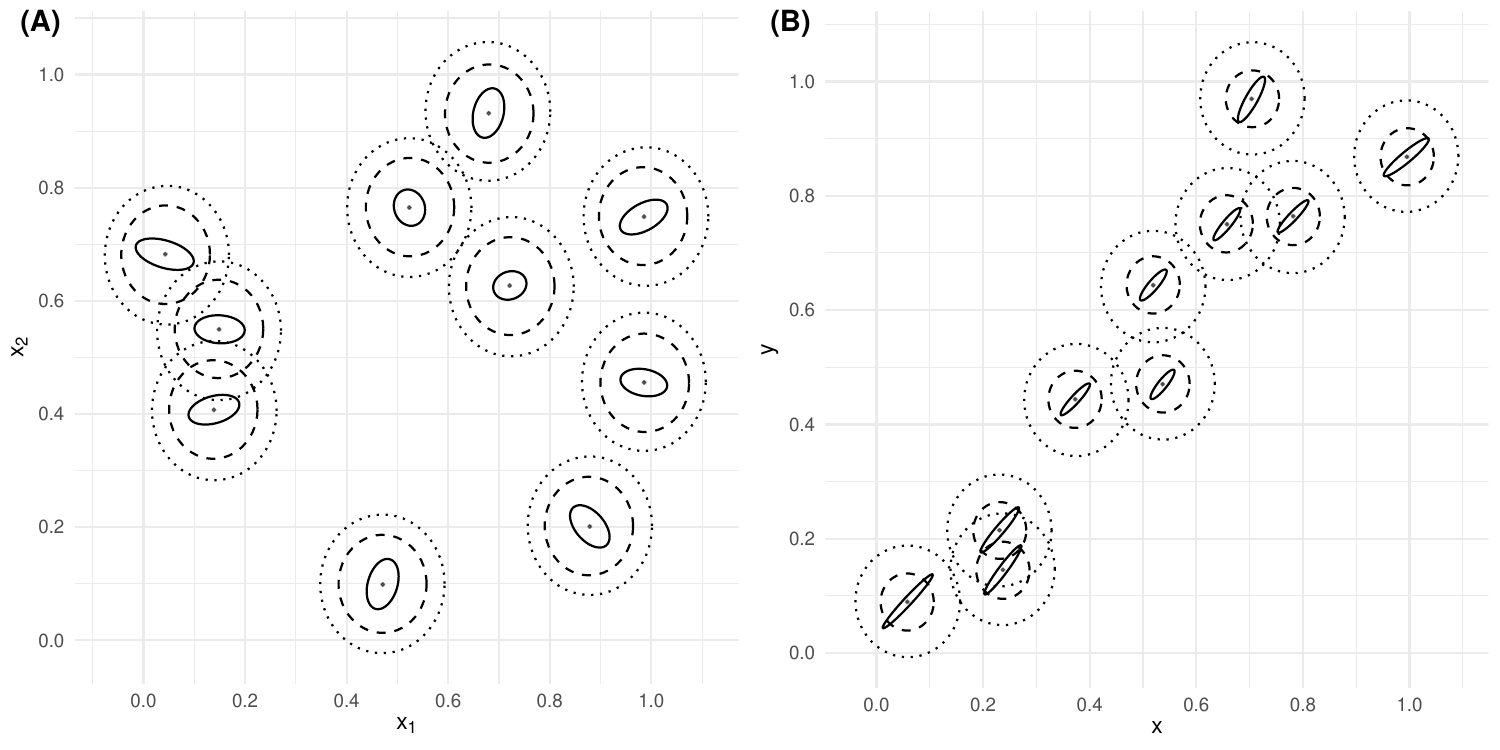}
	\caption{\footnotesize Example of method of moments and maximum pseudo-likelihood estimators for the kernel. The plots present the level curves of the Gaussian distributions with kernel estimated by the method of moments, exact (dashed) and approximated (dotted), and maximum pseudo-likelihood estimation (solid) in \textbf{(A)} Gaussian bolstering and \textbf{(B)} $XY$-Gaussian bolstering. All the level curves are such that the probability inside the sphere/ellipse equals $0.05$.}
	\label{fig_ex_kernel_estimation}
\end{figure}

\section{Experimental results}
\label{SecExp}

\begin{figure}[ht]
\centering
\includegraphics[width = \linewidth]{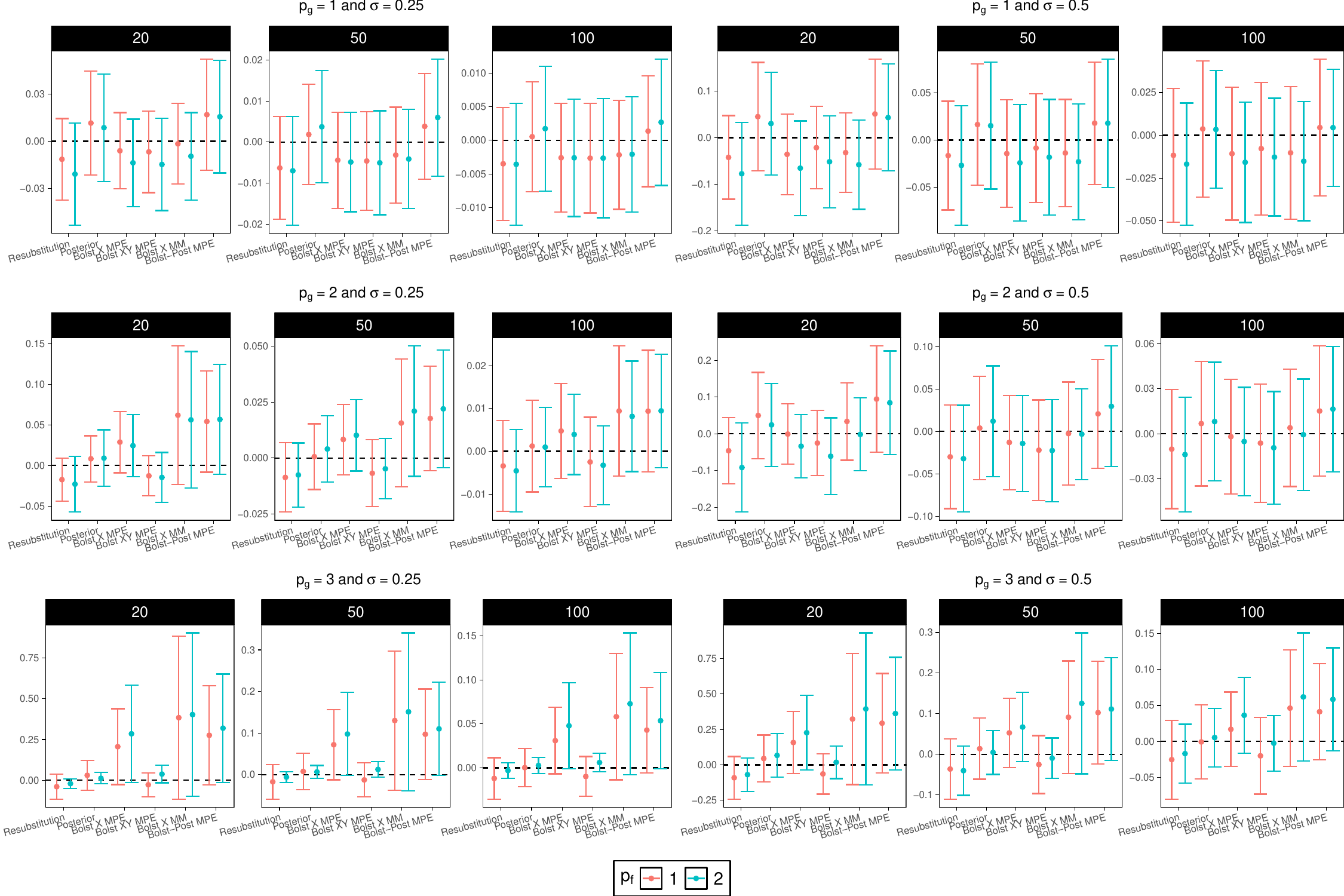}
\caption{Bias $\pm$ RMSE for each estimator and sample size over the 100 samples for $d = 1$. Each plot represents a value of $p_{g}$ and $\sigma$, and the colors refer to $p_{f}$.} \label{fig_res1}
\end{figure}

\begin{figure}[ht]
\centering
\includegraphics[width = \linewidth]{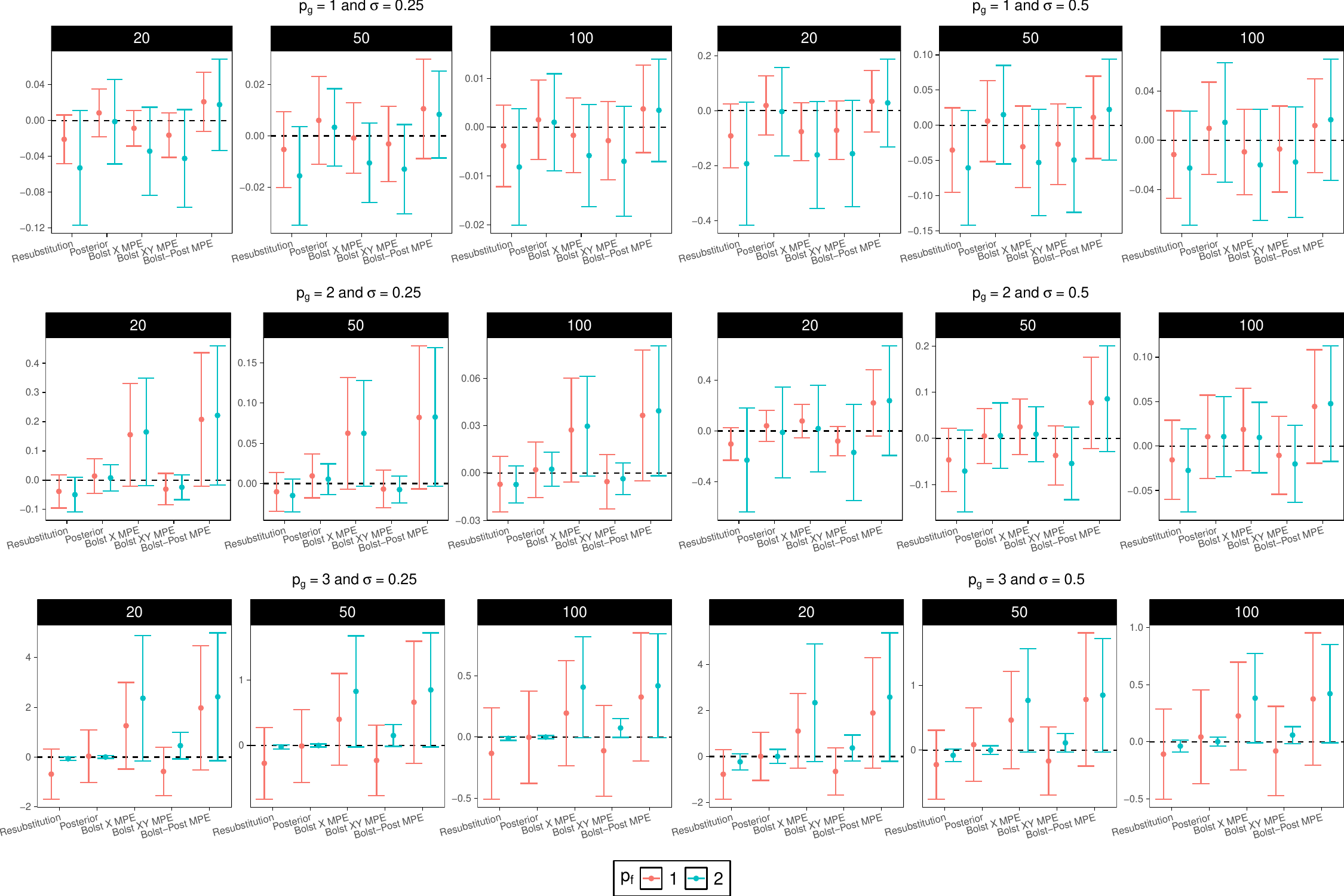}
\caption{Bias $\pm$ RMSE for each estimator and sample size over the 100 samples for $d = 2$. Each plot represents a value of $p_{g}$ and $\sigma$, and the colors refer to $p_{f}$.} \label{fig_res2}
\end{figure}

\begin{figure}[ht]
\centering
\includegraphics[width = \linewidth]{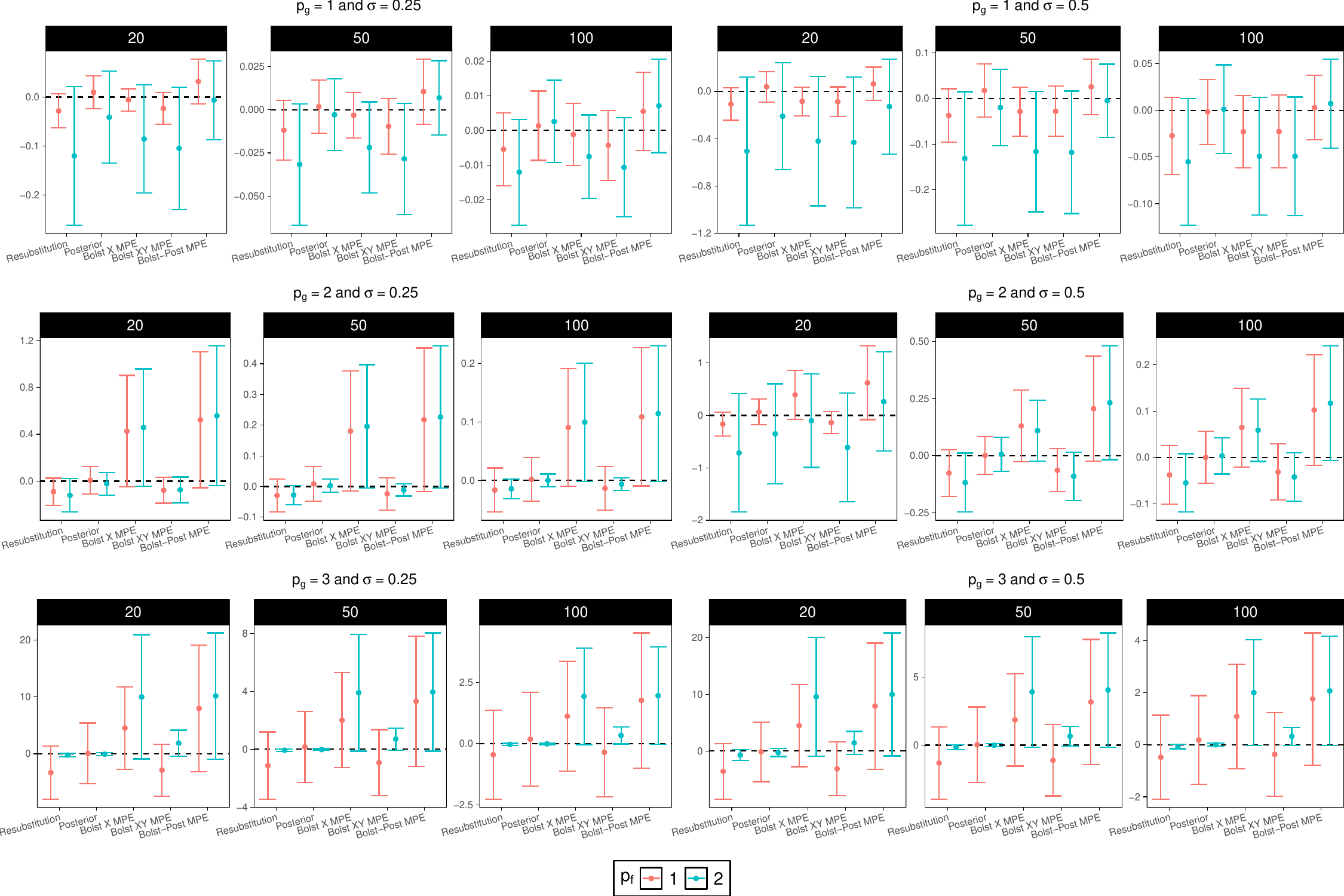}
\caption{Bias $\pm$ RMSE for each estimator and sample size over the 100 samples for $d = 3$. Each plot represents a value of $p_{g}$ and $\sigma$, and the colors refer to $p_{f}$.} \label{fig_res3}
\end{figure}

In this section, we present experimental results on error estimation for polynomial regression with synthetic data. The simulations were performed in \textbf{R} \cite{R} with the \textit{genree} package that we developed specially for it and that is available at \href{https://github.com/dmarcondes/genree}{\textbf{https://github.com/dmarcondes/genree}}. Bayesian polynomial regressions were fitted with the package \textit{rstanarm} \citep{rstanarm}. 

The synthetic data is generated from $X$ uniformly distributed in $[0,1]^{d}$ for $d = 1,2,3$ and $Y = \psi^{\star}(X) + \xi_{\sigma}$ in which 
\begin{linenomath}
	\begin{equation*}
		\psi^{\star}(x) = \left[1 + \sum_{i=1}^{d} x_{i}\right]^{p_{g}}
	\end{equation*}
\end{linenomath}
for $p_{g} = 1,2,3$, and $\xi_{\sigma}$ is a random variable, independent of $X$, with the Gaussian distribution of mean zero and variance $\sigma^{2}$ for $\sigma = 0.25,0.5$. We consider sample sizes of $n = 20,50,100$ and perform Monte Carlo integration with a sample size of $1,000$ to calculate $\varepsilon_{n}$ and the error estimators. 

The simulations are performed for prediction rules generated by least square estimation of a degree $p_{f}$ polynomial for $p_{f} = 1,2$.  We consider the following generalized error estimators: $X$-Gaussian bolstering with kernel estimated via maximum pseudo-likelihood (MPE) and exact method of moments (MM); $XY$-Gaussian bolstering with kernel estimated via maximum pseudo-likelihood (MPE); posterior-probability under degree $p_{f}$ Bayesian polynomial regression; and Gaussian bolstering posterior-probability under MPE and degree $p_{f}$ Bayesian polynomial regression. The degree $p_{f}$ of Bayesian regression is the same degree of the fitted polynomial, but may differ from the degree $p_{g}$ of the polynomial that generated the data. The bias and root-mean-square error (RMSE) of each scenario is estimated by their mean on $100$ independent samples.

In Figures \ref{fig_res1} to \ref{fig_res3} we present the bias $\pm$ RMSE of the estimators for each scenario for $d = 1, 2$ and $3$, respectively. The result for the $X$-Gaussian bolstering with kernel estimated via the method of moments is omitted from the figures for $d \geq 2$ since its results are, in general, significantly worse than the other estimators. The bias and RMSE of each scenario and estimator can be found in Appendix \ref{AppExpRes}. 

The posterior-probability estimator has the least absolute bias in the majority of the scenarios (83 out of 108), with the exception of the scenarios with $d = 1$ and $\sigma = 0.5$ in which bolstering estimators had the least absolute bias in many cases. The bolstering estimators in the $X$ direction, and the Gaussian bolstering posterior-probability estimator, perform poorly for $p_{g} = 2,3$ specially for $d \geq 2$ evidencing that bolstering only on the $X$ direction may not be suitable in higher dimensional scenarios and when the data is generated by more complex functions. 

Conversely, the performance of the $XY$-bolstering estimator is not sensible to the dimension $d$ and the data generating polynomial degree $p_{g}$. It is among the estimators with the least RMSE in almost all cases, and has the least one in 39 out of 108 scenarios. The posterior estimator is also not sensible to the dimension and the data generating polynomial degree, and has the least RMSE in 25 scenarios.

\section{Discussion}
\label{SecDiscuss}

In this paper, generalized resubstitution error estimators were extended to the general supervised learning framework, in special to regression problems, and studied from a statistical learning perspective. We formally defined generalized resubstitution error estimators as the expected loss under a generalized empirical distribution and, equivalently, as the resubstitution error of a generalized loss function. Sufficient conditions for the consistency of these estimators were established without making any assumptions about the prediction rule or distribution; only an assumption about the moments of the loss function was made (cf. \eqref{finite_moments_text}). Consistency of these estimators was studied in general and specific cases, and the results were particularized to smoothing on the $X$, $Y$ and both directions. Two methods for learning the kernel in Gaussian bolstering were proposed, formalizing the heuristic approach of \cite{braga2004bolstered}, by applying the EM algorithm for Gaussian mixtures. The estimators were applied to polynomial regression and their superiority with respect to plain resubstitution was empirically observed.

Two natural paths for future studies arise from this paper. From a theoretical perspective, we note that we have established the consistency of the estimators, but we have not provided their convergence rate. The rates we could deduce in a distribution-free and prediction rule-free framework would be very pessimistic, and therefore have no practical use. However, if assumptions about the data distribution and/or about the prediction rule were to be made, sharper bounds could be obtained. From an applied perspective, it would be interesting to apply these estimators to more complex regression models, e.g., neural networks. In such cases, overfitting is expected to be more prevalent, and a greater benefit from smoothing the empirical measure could be obtained.

\acks{D. Marcondes was funded by grants \#2022/06211-2 and \#2023/00256-7, S\~ao Paulo Research Foundation (FAPESP). U. Braga-Neto was supported by NSF Award CCF-2225507. Most of this work was developed while D. Marcondes was a visiting scholar at the Department of Electrical and Computer Engineering, Texas A\&M University.}

\bibliography{Ref}

\appendix
\section{VC dimension and uniform relative deviation convergence}
\label{AppVCdim}

The complexity of a hypothesis space may be measured by its VC dimension under loss function $\ell$, that is a special case of the pseudo-dimension of a family of real-valued functions \cite{pollard1989asymptotics,vapnik1998}. We start by recalling the definition of the shatter coefficient and VC dimension of a class $\mathcal{A} \subset \mathcal{B}(\mathbb{R}^{d+m})$ of Borel-measurable sets.

\begin{definition}
	\label{defVC_set}
	The $n$-th shatter coefficient of a class $\mathcal{A} \subset \mathcal{B}(\mathbb{R}^{d+m})$ is defined as
	\begin{linenomath}
		\begin{align*}
			\mathcal{S}(\mathcal{A},n) = \sup\limits_{z_{1},\dots,z_{n} \in \mathcal{Z}} \Big| \Big\{\left(\mathds{1}\{z_{1} \in A\},\dots,\mathds{1}\{z_{n} \in A\}\right): A \in \mathcal{A}\Big\}\Big|.
		\end{align*}
	\end{linenomath}
	The VC dimension of $\mathcal{A}$ is the greatest integer such that $\mathcal{S}(\mathcal{A},n) = 2^{n}$ and is denoted by $d_{VC}(\mathcal{A})$.
\end{definition}

For a fixed loss function $\ell$, $\psi \in \mathcal{H}$ and $b \in (0,C_{\ell})$ define
\begin{linenomath}
	\begin{equation*}
		A_{\ell,\psi,b} = \{z \in \mathcal{Z}: \ell(z,\psi) > b\} \in \mathcal{B}(\mathbb{R}^{d+m})
	\end{equation*}
\end{linenomath}
as the points in $\mathcal{Z}$ where $\ell(z,\psi) > b$ and let 
\begin{linenomath}
	\begin{equation}
		\label{def_Astar}
		\mathcal{A}^{\star}_{\mathcal{H},\ell} = \left\{A_{\ell,\psi,b}: \psi \in \mathcal{H}, b \in (0,C_{\ell})\right\} \subset \mathcal{B}(\mathbb{R}^{d+m})
	\end{equation}
\end{linenomath}
be the collection of such sets for $\psi \in \mathcal{H}$ and $b \in (0,C_{\ell})$. We have that
\begin{linenomath}
	\begin{align*}
		\mathcal{S}(\mathcal{A}^{\star}_{\mathcal{H},\ell},n) = \sup\limits_{z_{1},\dots,z_{n} \in \mathcal{Z}} \Big| \Big\{\left(\mathds{1}\{\ell(z_{1},\psi) > b\},\dots,\mathds{1}\{\ell(z_{n},\psi) > b\}\right): \psi \in \mathcal{H}, b \in (0,C_{\ell})\Big\}\Big|
	\end{align*}
\end{linenomath}
which is the usual shatter coefficient of the space of binary functions of form $I(z) = \mathds{1}\{\ell(z,\psi) > b\}$ for $z \in \mathcal{Z}$.

In classification problems with two classes, under the simple loss function, $\mathcal{S}(\mathcal{A}^{\star}_{\mathcal{H},\ell},n)$ reduces to the usual definition of the shatter coefficient of a hypothesis space of binary functions:
\begin{linenomath}
	\begin{align*}
		\mathcal{S}(\mathcal{H},n) \coloneqq \sup\limits_{x_{1},\dots,x_{n} \in \mathcal{X}} \Big| \Big\{\left(\psi(x_{1}),\dots,\psi(x_{n})\right): \psi \in \mathcal{H}\Big\}\Big|.
	\end{align*}
\end{linenomath}
This fact justifies the definition of the shatter coefficient and VC dimension of a hypothesis space $\mathcal{H}$ under a loss function $\ell$ as that of $\mathcal{A}^{\star}_{\mathcal{H},\ell}$. See Appendix \ref{SecSquareLoss} for more details about the VC dimension under a loss function.

\begin{definition}
	The $n$-th shatter coefficient of $\mathcal{H}$ under loss function $\ell$ is defined as $\mathcal{S}(\mathcal{H},\ell,n) \coloneqq \mathcal{S}(\mathcal{A}^{\star}_{\mathcal{H},\ell},n)$.  The VC dimension of a hypothesis space $\mathcal{H}$ under loss function $\ell$ is defined as $d_{VC}(\mathcal{H},\ell) \coloneqq d_{VC}(\mathcal{A}^{\star}_{\mathcal{H},\ell})$. 
\end{definition}

Bounds on the rate of uniform relative deviation convergence based on the VC dimension have been obtained in the literature. This paper will rely on a result in \cite{cortes2019relative}, that is an improvement of results in \cite{vapnik1998}, and we refer to the references in \cite{cortes2019relative} for a historical overview of these results. 

For $\epsilon, \tau > 0$ and $n \geq 1$, let $\mathbb{B}_{n,\epsilon,\tau}: \mathbb{Z}_{+} \to \mathbb{R}_{+}$ be a function such that
\begin{linenomath}
	\begin{equation}
		\label{condBC}
		\sum_{n=1}^{\infty} \mathbb{B}_{n,\epsilon,\tau}(k) < \infty
	\end{equation}
\end{linenomath}
for all $k \in \mathbb{Z}_{+}$ and all $\epsilon,\tau > 0$ fixed. The results in \cite{cortes2019relative} imply that, under the assumption \eqref{finite_moments_text}, there exists $\mathbb{B}_{n,\epsilon,\tau}(d_{VC}(\mathcal{H},\ell))$ satisfying \eqref{condBC} which is a bound for the rate of uniform relative deviation in $\mathcal{H}$ under loss function $\ell$. See \cite{cortes2019relative} for an explicit form for $\mathbb{B}_{n,\epsilon,\tau}$.

\begin{proposition}
	\label{prop_bound_convergence}
	Fix a loss function $\ell$ and let $\mathcal{H}$ be a hypothesis space with $d_{VC}(\mathcal{H},\ell) < \infty$. If \eqref{finite_moments_text} holds, then there exists a function $\mathbb{B}_{n,\epsilon,\tau}$ satisfying \eqref{condBC} such that, for all $\epsilon > 0$, and for a $\tau > 0$ small enough,
	\begin{linenomath}
		\begin{equation}
			\label{rel_conv_rate}
			\mathbb{P}\left(\sup\limits_{\psi \in \mathcal{H}} \left|\frac{L(\psi) - L_{n}(\psi)}{\sqrt[\alpha]{L^{\alpha}(\psi) + \tau}}\right| > \epsilon\right) < \mathbb{B}_{n,\epsilon,\tau}(d_{VC}(\mathcal{H},\ell)).
		\end{equation}
	\end{linenomath}
	It follows from \eqref{condBC} and Borel-Cantelli Lemma that
	\begin{linenomath}
		\begin{equation}
			\label{rel_conv}
			\sup\limits_{\psi \in \mathcal{H}} \left|\frac{L(\psi) - L_{n}(\psi)}{\sqrt[\alpha]{L^{\alpha}(\psi) + \tau}}\right| \to \ 0
		\end{equation}
	\end{linenomath}
	with probability one as $n \to \infty$.
\end{proposition}

\begin{remark}
	The constant $\tau$ avoids the discontinuity of the ratio when $\inf\limits_{\psi \in \mathcal{H}} L^{\alpha}(\psi) = 0$. See \cite{cortes2019relative} for more details.
\end{remark}

\begin{remark}
	The results in this paper rely only on \eqref{rel_conv} and not on the rate in \eqref{rel_conv_rate}. Therefore, they rely only on the finiteness of $d_{VC}(\mathcal{H},\ell)$ rather than the specific form of $\mathbb{B}_{n,\epsilon,\tau}$, so other complexity measures that guarantee a result analogous to \eqref{rel_conv} could be considered. 
\end{remark}

\begin{remark}
	The main goal of this paper is to establish the consistency of generalized resubstitution errors in a general setting, rather than obtain sharp bounds for the rate of such convergence. Therefore, we do not explicit the bound $\mathbb{B}_{n,\epsilon,\tau}$ or consider better bounds, based for example on the $n$-th shatter coefficient.
\end{remark}

The next lemma shows that, under the assumptions of this paper, the consistency of a generalized resubstitution error $\hat{\varepsilon}_{n}^{\mathscr{B}}$ is equivalent to the convergence to zero of its relative deviation from $\varepsilon_{n}$.

\begin{lemma}
	\label{corollary_equiv_cons}
	Fix a loss function $\ell$. The generalized resubstitution error $\hat{\varepsilon}_{n}^{\mathscr{B}}$ is consistent if, and only if, for any $\tau > 0$ fixed,
	\begin{linenomath}
		\begin{align}
			\label{cond_equi}
			\frac{\left| \hat{\varepsilon}_{n}^{\mathscr{B}} - \varepsilon_{n} \right|}{\sqrt[\alpha]{L^{\alpha}(\psi_{n}) + \tau}} \to 0
		\end{align}
	\end{linenomath}
	with probability one as $n \to \infty$. 
\end{lemma}
\begin{proof}
	The result follows from condition \eqref{finite_moments_text} and inequality
	\begin{linenomath}
		\begin{align*}
			\frac{\left| \hat{\varepsilon}_{n}^{\mathscr{B}} - \varepsilon_{n} \right|}{\sqrt[\alpha]{\sup\limits_{\psi \in \mathcal{H}} L^{\alpha}(\psi) + \tau}} \leq \frac{\left| \hat{\varepsilon}_{n}^{\mathscr{B}} - \varepsilon_{n} \right|}{\sqrt[\alpha]{L^{(\alpha)}(\psi_{n}) + \tau}} \leq \frac{\left| \hat{\varepsilon}_{n}^{\mathscr{B}} - \varepsilon_{n} \right|}{\sqrt[\alpha]{\tau}}.
		\end{align*}
	\end{linenomath}
\end{proof}

\section{Proof of results}
\label{AppProofs}

\begin{proof}[Proof of Proposition \ref{proposition_unbia}]
	The result follows from \eqref{finite_moments_text}, \eqref{finite_moments_Q} and the Dominated Convergence Theorem \cite[Theorem~A7]{braga2020fundamentals} since
	\begin{linenomath}
		\begin{align*}
			\left| \hat{\varepsilon}_{n}^{\mathscr{B}} - \varepsilon_{n} \right| \leq 2 \max\left\{\varepsilon_{n},\hat{\varepsilon}_{n}^{\mathscr{B}}\right\} \leq 2 \max\left\{\sup\limits_{\psi \in \mathcal{H}} L(\psi),L_{n}^{\mathscr{B}}(\psi_{n})\right\} < \infty
		\end{align*}
	\end{linenomath}
	with probability one for all $n \geq 1$.
\end{proof}

\begin{proof}[Proof of Theorem \ref{theoremER}]
	The result follows from Proposition \ref{prop_bound_convergence}, Lemma \ref{corollary_equiv_cons} and Borel-Cantelli Lemma by noting that, for any $\epsilon > 0$ and $\tau > 0$ fixed,
	\begin{linenomath}
		\begin{align*}
			\mathbb{P}\left(\frac{\left|\hat{\varepsilon}_{n}^{r} - \varepsilon_{n}\right|}{\sqrt[\alpha]{L^{\alpha}(\psi_{n}) + \tau}} > \epsilon\right) &= \mathbb{P}\left(\frac{|L_{n}(\psi_{n}) - L(\psi_{n})|}{\sqrt[\alpha]{L^{\alpha}(\psi_{n}) + \tau}} > \epsilon\right) \\
			&\leq \mathbb{P}\left(\left\{\exists \psi \in \mathcal{H}: \frac{|L_{n}(\psi) - L(\psi)|}{\sqrt[\alpha]{L^{\alpha}(\psi) + \tau}} > \epsilon\right\}\right)\\
			&= \mathbb{P}\left(\sup\limits_{\psi \in \mathcal{H}} \frac{\left|L(\psi) - L_{n}(\psi)\right|}{\sqrt[\alpha]{L^{\alpha}(\psi) + \tau}} > \epsilon\right).
		\end{align*}
	\end{linenomath}
\end{proof}

\begin{proof}[Proof of Proposition \ref{prop_suf1}]
	Observe that
	\begin{linenomath}
		\begin{align}
			\label{triangle_proof}
			\left| \hat{\varepsilon}_{n}^{\mathscr{B}} - \varepsilon_{n} \right| \leq \left| \hat{\varepsilon}_{n}^{\mathscr{B}} - \hat{\varepsilon}_{n}^{r} \right| + \left| \hat{\varepsilon}_{n}^{r} - \varepsilon_{n} \right|.
		\end{align}
	\end{linenomath}
	The result follows since the first term on the right-hand side of \eqref{triangle_proof} converges to zero with probability one by hypothesis and the second one converges to zero with probability one by Theorem \ref{theoremER}.
\end{proof}

\begin{proof}[Proof of Corollary \ref{cor_suf2}]
	We show that $|\hat{\varepsilon}_{n}^{\mathscr{B}} - \hat{\varepsilon}_{n}^{r}| \to 0$ with probability one, so the result follows from Proposition \ref{prop_suf1}. Denote by $\hat{\nu}_{n}$ and $\nu_{n}$ the empirical measures
	\begin{linenomath}
		\begin{align*}
			\hat{\nu}_{n} = \frac{1}{n} \sum_{i=1}^{n} \beta_{Z_{i},\psi_{n},S_{n}} & & \nu_{n} = \sum_{i=1}^{n} \delta_{Z_{i}}.
		\end{align*}
	\end{linenomath}
	Since $\ell$ is bounded by a constant $C_{\ell}$, we have that
	\begin{linenomath}
		\begin{equation*}
			\hat{\varepsilon}_{n}^{r} = \nu_{n}[\ell(Z,\psi_{n})] = \lim\limits_{m \to \infty} \sum_{k=1}^{m-1} \frac{C_{\ell}}{m} \nu_{n}\left(\ell(Z,\psi_{n}) > \frac{kC_{\ell}}{m}\right),
		\end{equation*}
	\end{linenomath} 
	and that
	\begin{linenomath}
		\begin{equation*}
			\hat{\varepsilon}_{n}^{\mathscr{B}} = \hat{\nu}_{n}[\ell(Z,\psi_{n})] = \lim\limits_{m \to \infty} \sum_{k=1}^{m-1} \frac{C_{\ell}}{m} \hat{\nu}_{n}\left(\ell(Z,\psi_{n}) > \frac{kC_{\ell}}{m}\right),
		\end{equation*}
	\end{linenomath}     
	in which $\nu_{n}[\cdot]$ and $\hat{\nu}_{n}[\cdot]$ mean expectation under $\nu_{n}$ and $\hat{\nu}_{n}$, respectively. From the representation of $\hat{\varepsilon}_{n}^{r}$ and $\hat{\varepsilon}_{n}^{\mathscr{B}}$ described above, we have that
	\begin{linenomath}
		\begin{align*}
			\left|\hat{\varepsilon}_{n}^{\mathscr{B}} - \hat{\varepsilon}_{n}^{r}\right| &= \left|\lim\limits_{m \to \infty} \sum_{k=1}^{m-1} \frac{C_{\ell}}{m} \left(\hat{\nu}_{n}\left(\ell(Z,\psi_{n}) > \frac{kC_{\ell}}{m}\right) - \nu_{n}\left(\ell(Z,\psi_{n}) > \frac{kC_{\ell}}{m}\right)\right) \right|\\
			&\leq \left|\lim\limits_{m \to \infty} \sum_{k=1}^{m-1} \frac{C_{\ell}}{m} \sup\limits_{\substack{\psi \in \mathcal{H}\\0 \leq b \leq C_{\ell}}} \left(\hat{\nu}_{n}\left(\ell(Z,\psi) > b\right) - \nu_{n}\left(\ell(Z,\psi) > b\right)\right) \right|\\
			&\leq C_{\ell} \sup\limits_{\substack{\psi \in \mathcal{H}\\0 \leq b \leq C_{\ell}}} \left|\hat{\nu}_{n}\left(\ell(Z,\psi) > b\right) - \nu_{n}\left(\ell(Z,\psi) > b\right)\right|\\
			&= C_{\ell} \sup\limits_{A \in \mathcal{A}_{\mathcal{H},\ell}^{\star}} \left|\hat{\nu}_{n}\left(A\right) - \nu_{n}\left(A\right)\right|
		\end{align*}
	\end{linenomath}
	for all samples $S_{n} \in \mathcal{Z}^{n}$ and the result follows. The second assertion is a consequence of the inequality 
	\begin{linenomath}
		\begin{equation*}
			\sup\limits_{A \in \mathcal{A}_{\mathcal{H},\ell}^{\star}} \left|\hat{\nu}_{n}\left(A\right) - \nu_{n}\left(A\right)\right| \leq \sup\limits_{\substack{A \in \mathcal{A}^{\star}_{\mathcal{H},\ell}\\ z \in \mathcal{Z}, \psi \in \mathcal{H}}} \left|\beta_{z,\psi,S_{n}}(A) - \delta_{z}(A)\right|
		\end{equation*}
	\end{linenomath}
	which holds for all samples $S_{n} \in \mathcal{Z}^{n}$.
\end{proof}

\begin{proof}[Proof of Corollary \ref{cor_suf3}]
	We show that $|\hat{\varepsilon}_{n}^{\mathscr{B}} - \hat{\varepsilon}_{n}^{r}| \to 0$ with probability one, so the result follows from Proposition \ref{prop_suf1}. The result follows since
	\begin{linenomath}
		\begin{align*}
			|\hat{\varepsilon}_{n}^{\mathscr{B}} - \hat{\varepsilon}_{n}^{r}| &\leq \frac{1}{n} \sum_{i=1}^{n} \left|\ell_{\mathscr{B}}(Z_{i},\psi_{n}) - \ell(Z_{i},\psi_{n})\right|\\
			&\leq \sup\limits_{z \in \mathcal{Z},\psi \in \mathcal{H}} \left|\ell_{\mathscr{B}}(z,\psi) - \ell(z,\psi)\right|
		\end{align*}
	\end{linenomath}
	for all sample $S_{n} \in \mathcal{Z}^{n}$.
\end{proof}

\begin{proof}[Proof of Proposition \ref{prop_condExchange}]
	Observe that
	\begin{linenomath}
		\begin{align*}
			\left|\hat{\varepsilon}_{n}^{\mathscr{B}_{n}} - \varepsilon_{n}\right| \leq \left|\hat{\varepsilon}_{n}^{\mathscr{B}_{n}} -  \varepsilon_{n}^{\mathscr{B}_{n}}\right| + \left|\varepsilon_{n}^{\mathscr{B}_{n}} - \varepsilon_{n}\right|.
		\end{align*}
	\end{linenomath}
	The result follows since the first term of the sum above converges to zero with probability one due to Theorem \ref{theoremER} and the second converges to zero with probability one by hypothesis.
\end{proof}

\begin{proof}[Proof of Proposition \ref{prop_symmetric}]
	Observe that, for all $\psi \in \mathcal{H}$,
	\begin{linenomath}
		\begin{align*}
			L^{\mathscr{B}_{n}}(\psi) &= \int_{\mathcal{Z}} \ell_{\mathscr{B}_{n}}(z,\psi) \ d\nu(z)\\
			&= \int_{\mathcal{Z}} \int_{\mathcal{Z}} \ell(z^{\prime},\psi) \ d\beta_{z,\psi,n}(z^{\prime}) \ d\nu(z)\\
			&= \int_{\mathcal{Z}} \int_{\mathcal{Z}} \ell(z^{\prime},\psi) \rho_{z,\psi,n}(z^{\prime})  \ d\nu(z^{\prime}) \ d\nu(z)\\
			&= \int_{\mathcal{Z}} \ell(z^{\prime},\psi) \int_{\mathcal{Z}}  \rho_{z,\psi,n}(z^{\prime})  \ d\nu(z)   \ d\nu(z^{\prime})\\
			&= \int_{\mathcal{Z}} \ell(z^{\prime},\psi) \ d\nu(z^{\prime}) = L(\psi)
		\end{align*}
	\end{linenomath}
	hence the condition of Proposition \ref{prop_condExchange} holds trivially.
\end{proof}

\begin{proof}[Proof of Corollary \ref{corr_symmetric}]
	The result follows since \eqref{cond-prop-symm} implies that 
	\begin{linenomath}
		\begin{align*}
			L^{\mathscr{B}_{n}}(\psi) &= \int_{\mathcal{Z}} \int_{\mathcal{Z}} \ell(z^{\prime},\psi) \rho_{z,\psi,n}(z^{\prime})  \ d\nu(z^{\prime}) \ d\nu(z)\\
			&= \int_{\mathcal{Z}} \int_{\mathcal{Z}} \ell(z^{\prime},\psi) \rho_{z^{\prime},\psi,n}(z)  \ d\nu(z^{\prime}) \ d\nu(z)\\
			&= \int_{\mathcal{Z}} \ell(z^{\prime},\psi) \int_{\mathcal{Z}}  \rho_{z^{\prime},\psi,n}(z)  \ d\nu(z)   \ d\nu(z^{\prime}) = L(\psi).
		\end{align*}
	\end{linenomath}
\end{proof}

\begin{proof}[Proof of Proposition \ref{prop_gaus_br}]
	The result follows from Corollary \ref{corr_symmetric} since (a) follows from the assumption that $X$ is absolutely continuous wrt Lebesgue measure and (b) follows since $p_{x,\psi,n}(x^{\prime}) = p_{x^{\prime},\psi,n}(x)$ for all $x,x^{\prime} \in \mathbb{R}^{d}$ and $\psi \in \mathcal{H}$.
\end{proof}

\begin{proof}[Proof of Proposition \ref{Prop_2der}]
	For $z \in \mathcal{Z}, \psi \in \mathcal{H}$ and $S_{n} \in \mathcal{Z}^{n}$ it holds
	\begin{linenomath}
		\begin{align}
			\label{taylor} \nonumber
			&\ell_{\mathscr{B}}(z,\psi) = \int_{\mathcal{Z}} \ell(z^{\prime},\psi) \ d\beta_{z,\psi,S_{n}}(z^{\prime})\\ \nonumber
			&= \ell(z,\psi) + \mathbb{E}\left[\left(Z_{z,\psi,S_{n}} - z\right)\right] \cdot \nabla \ell(z,\psi) \\
			&+ \frac{1}{2} \sum_{i,j = 1}^{d + m} \int_{\mathcal{Z}} (z^{\prime}_{i} - z_{i})(z^{\prime}_{j} - z_{j}) \frac{\partial^{2}\ell}{\partial z_{i} \partial z_{j}}(\bar{z},\psi) \ d\beta_{z,\psi,S_{n}}(z^{\prime}) \nonumber \\
			&= \ell(z,\psi) + \frac{1}{2} \sum_{i,j=1}^{d+m} \int_{\mathcal{Z}} (z^{\prime}_{i} - z_{i})(z^{\prime}_{j} - z_{j}) \frac{\partial^{2}\ell}{\partial z_{i} \partial z_{j}}(\bar{z},\psi) \ d\beta_{z,\psi,S_{n}}(z^{\prime}) 
		\end{align}
	\end{linenomath}
	in which $\bar{z}$ depends on $z^{\prime}$. Observe that the first equality follows from the Taylor expansion of $\ell(\cdot,\psi)$ around $z$ and the second follows since $\mathbb{E}(Z_{z,\psi,S_{n}}) = z$. Now,
	\begin{linenomath}
		\begin{align*}
			&|\ell_{\mathscr{B}}(z,\psi) - \ell(z,\psi)| \\
			&\leq \frac{1}{2} \sum_{i,j=1}^{d+m} \left[\sup\limits_{z \in \mathcal{Z}} \left|\frac{\partial^{2}\ell}{\partial z_{i} \partial z_{j}}(z,\psi)\right|\right] \int_{\mathcal{Z}} |(z^{\prime}_{i} - z_{i})||(z^{\prime}_{j} - z_{j})| \ d\beta_{z,\psi,S_{n}}(z^{\prime})  \\
			&\leq \frac{C_{2}}{2} \sum_{i,j=1}^{d+m} \sqrt{Var\left([Z_{z,\psi,S_{n}}]_{i}\right)Var\left([Z_{z,\psi,S_{n}}]_{j}\right)}
		\end{align*}
	\end{linenomath}
	in which the second inequality follows from \eqref{bounded_dev2} and Cauchy-Schwarz inequality. The result follows by Corollary \ref{cor_suf3} since the variances converge to zero with probability one as $n \to \infty$.
\end{proof}

\begin{proof}[Proof of Proposition \ref{prop_gaus_brR}]
	We will show that the conditions of Proposition \ref{Prop_2der} are in force. Let $\beta_{z,\psi,S_{n}}$ be a degenerate Gaussian law with mean $z$ and covariance matrix $\tilde{K}_{z,\psi,S_{n}}$ satisfying $[\tilde{K}_{z,\psi,S_{n}}]_{i,j} = [K_{z,\psi,S_{n}}]_{i,j}$ if $i,j < d+1$ and $[\tilde{K}_{z,\psi,S_{n}}]_{i,j} = 0$ if $i = d + 1$ or $j = d + 1$. Under these assumptions $\hat{\varepsilon}_{n}^{\mathscr{B}} = \hat{\varepsilon}_{n}^{gbr}$. 
	
	Recall that $Z_{z,\psi,S_{n}}$ is a random variable with distribution $\beta_{z,\psi,S_{n}}$ for $z \in \mathcal{Z}$ and $\psi \in \mathcal{H}$. Condition (b) of Proposition \ref{Prop_2der} is a direct consequence of \eqref{cond_var_zero}. It remains to show that condition (a) of Proposition \ref{Prop_2der} is in force.
	
	Since $\mathcal{Y}$ is compact, and the derivatives and second derivatives of $\psi$ are uniformly bounded in $\mathcal{X}$ and $\mathcal{H}$, it follows that
	\begin{linenomath}
		\begin{align*}
			\frac{\partial^{2}\ell}{\partial x_{i} \partial x_{j}}((x,y),\psi) &= 2(\psi(x) - y) \frac{\partial^{2}\psi}{\partial x_{i} \partial x_{j}}(x) + 2 \frac{\partial\psi}{\partial x_{i}}(x)\frac{\partial\psi}{\partial x_{j}}(x),
		\end{align*}
	\end{linenomath}
	is uniformly bounded. Moreover,
	\begin{linenomath}
		\begin{align*}
			\frac{\partial^{2}\ell}{\partial x_{i} \partial y}((x,y),\psi) &= - 2 \frac{\partial\psi}{\partial x_{i}}(x) & & \text{ and } & & \frac{\partial^{2}\ell}{\partial^{2} y}((x,y),\psi) = 2
		\end{align*}
	\end{linenomath}
	are also uniformly bounded. Therefore, condition (a) of Proposition \ref{Prop_2der} also holds and the result follows.
\end{proof} 

\begin{proof}[Proof of Lemma \ref{lemma_mean_distance}]
	We will compute the mean of $\delta(\hat{X})$ and an analogous deduction, but considering only the points in the sample with $Y_{i} = y$, holds to compute the mean of $\delta_{y}(\hat{X})$ conditioned on $\hat{Y} = y$. Let $I_{n}$ be a random variable and $N_{i}, i = 1,\dots,n$, be random vectors taking values in $\{1,\dots,n\}$ and $\mathcal{X} \times \mathcal{Y}$, respectively, all defined on the probability space $(\Omega,\mathcal{S},\mathbb{P})$ and independent, such that $\mathbb{P}(I_{n} = i) = 1/n$ for $i = 1,\dots,n$ and $N_{i}$ has the probability law $\mu_{X_{i},S_{n}} \delta_{Y_{i}}$. Define on $(\Omega,\mathcal{S},\mathbb{P})$ the random vector $(\hat{X},\hat{Y})$ by
	\begin{linenomath}
		\begin{equation*}
			(\hat{X},\hat{Y})(\omega) = N_{I_{n}(\omega)}(\omega)
		\end{equation*}
	\end{linenomath}
	for $\omega \in \Omega$. It follows that $(\hat{X},\hat{Y})$ has probability law $\nu_{n}$ and we will compute the mean $\mathbb{E}[\delta(\hat{X})]$. 
	
	Denote by $p_{i,S_{n}}$ the density of a Gaussian distribution with mean $X_{i}$ and covariance matrix $\sigma_{S_{n}}^{2}\Sigma$. It follows that
	\begin{linenomath}
		\begin{align*}
			\mathbb{E}\left[\delta(\hat{X})\right] &= \frac{1}{n} \sum_{i=1}^{n} \mathbb{E}[\delta(\hat{X}) | I_{n} = i]\\
			&= \frac{1}{n} \sum_{i=1}^{n}\int_{\sX} \delta(x) p_{i,S_{n}}(x) \ dx\\
			&= \frac{1}{n} \sum_{i,j=1}^{n} \int_{\sX} \delta(x,X_{j}) p_{i,S_{n}}(x) \mathds{1}_{\delta(x) = \delta(x,X_{j})} \ dx.
		\end{align*}
	\end{linenomath}
\end{proof}

\begin{proof}[Proof of Proposition \ref{lemma_chi_approx}]
	Define
	\begin{linenomath}
		\begin{equation*}
			M = \frac{1}{2} \min_{j \neq i} \delta(X_{i},X_{j})
		\end{equation*}
	\end{linenomath}
	and observe that
	\begin{linenomath}
		\begin{align*}
			\sigma_{S_{n}} \int_{\mathbb{R}^{d}} \frac{\delta(x,X_{i})}{\sigma_{S_{n}}} \ p_{i,S_{n}}(x) \ \mathds{1}_{\delta(x,X_{i}) < M} \ dx &\leq \mathbb{E}[\delta(\hat{X}) | I_{n} = i] \\
			&\leq \sigma_{S_{n}} \int_{\mathbb{R}^{d}} \frac{\delta(x,X_{i})}{\sigma_{S_{n}}} \ p_{i,S_{n}}(x) \ dx
		\end{align*}
	\end{linenomath}
	If $\hat{X}$ has a Gaussian distribution with mean $X_{i}$ and kernel $\sigma_{S_{n}} \Sigma$, then $\sigma_{S_{n}}^{-1} \delta(\hat{X},X_{i})$ has a chi distribution with $d$ degrees of freedom. Denoting by $\chi$ a random variable with this distribution, it follows that
	\begin{linenomath}
		\begin{equation*}
			\sigma_{S_{n}} \mathbb{E}[\chi\mathds{1}_{\chi < M/\sigma_{S_{n}}}] \leq \mathbb{E}[\delta(\hat{X}) | I_{n} = i] \leq \sigma_{S_{n}} \mathbb{E}[\chi]
		\end{equation*}
	\end{linenomath}
	The result follows by noting that
	\begin{linenomath}
		\begin{align*}
			\lim\limits_{\sigma_{S_{n}} \to 0} \frac{\mathbb{E}[\delta(\hat{X})]}{\sigma_{S_{n}}} = \frac{1}{n} \sum_{i=1}^{n} \lim\limits\limits_{\sigma_{S_{n}} \to 0} \frac{\mathbb{E}[\delta(\hat{X}) | I_{n} = i]}{\sigma_{S_{n}}}.
		\end{align*}
	\end{linenomath}
\end{proof}

\section{$L^{\infty}$-convergence and VC dimension}
\label{AppLinf}

In this section, we estimate the VC dimension of a class from the VC dimension of a class that is \textit{close} to it, in a sense we will make clear below. To easy notation, let $\Theta$ be a general set and consider the following sets of non-negative functions with domain $\mathcal{Z}$ indexed by the elements in $\Theta$:
\begin{linenomath}
	\begin{align*}
		\mathcal{F} = \{f_{\theta}: \theta \in \Theta\} & & \mathcal{G} = \{g_{\theta}: \theta \in \Theta\} & & \mathcal{G}_{n} = \{g_{\theta}^{(n)}: \theta \in \Theta\}
	\end{align*}
\end{linenomath}
for $n \geq 1$. The functions in these sets should be understood as loss functions of a hypothesis space indexed by $\Theta$, for example, $f_{\theta}(z) = \ell(z,\psi_{\theta})$. We define the VC dimension of these sets as, for example, $$d_{VC}(\mathcal{F}) = d_{VC}(\{\{z \in \mathcal{Z}: f_{\theta}(z) > \beta\}: \theta \in \Theta,\beta \in (0,C_{\mathcal{F}})\}).$$

We define the distance between two sets $\mathcal{F}$ and $\mathcal{G}$ as
\begin{linenomath}
	\begin{equation}
		\label{dist_sets}
		d(\mathcal{F},\mathcal{G}) = \sup\limits_{\theta \in \Theta} \lVert f_{\theta} - g_{\theta}\rVert_{\infty}.
	\end{equation}    
\end{linenomath}
For each integer $1 \leq d \leq d_{VC}(\mathcal{F})$ let
\begin{linenomath}
	\begin{equation*}
		S(\mathcal{F},d) = \left\{\boldsymbol{z} = \{z_{1},\dots,z_{d}\} \subset \mathcal{Z}: \mathcal{F} \text{ shatters } \{z_{1},\dots,z_{d}\}\right\}
	\end{equation*}
\end{linenomath}
be the collection of $d$ points that $\mathcal{F}$ can shatter. We denote simply $S(\mathcal{F}) \coloneqq S(\mathcal{F},d_{VC}(\mathcal{F}))$. For each set $\boldsymbol{z}$ in $S(\mathcal{F},d)$ and $b \in \{-1,+1\}^{s}, s \geq d,$ define
\begin{linenomath}
	\begin{equation*}
		\mathcal{F}(\boldsymbol{z},b) = \left\{\theta \in \Theta: \sup_{\beta} \min_{i = 1,\dots,d} b_{i}(f_{\theta}(z_{i}) - \beta) > 0\right\}
	\end{equation*}
\end{linenomath}
as the functions in $\mathcal{F}$ that realize dichotomy $b$ with the points $\boldsymbol{z}$.

For $d \geq 1$, we define the $d$-margin of a set $\mathcal{F}$ as
\begin{linenomath}
	\begin{align*}
		\delta_{\mathcal{F}}(d) = \begin{cases}
			\sup\limits_{\boldsymbol{z} \in S(\mathcal{F},d)} \min\limits_{b \in \{-1,1\}^{d}} \sup\limits_{\theta \in \mathcal{F}(\boldsymbol{z},b)} \sup\limits_{\beta} \min\limits_{i} b_{i}\left(f_{\theta}(z_{i}) - \beta\right) & \text{ if } d \leq d_{VC}(\mathcal{F})\\
			\sup\limits_{\boldsymbol{z} \in S(\mathcal{F})} \sup\limits_{z_{d_{VC}(\mathcal{F}) + 1},\dots,z_{d}} \min\limits_{b \in \{-1,1\}^{d}} \sup\limits_{\theta \in \mathcal{F}(\boldsymbol{z},b)} \sup\limits_{\beta} \min\limits_{i} b_{i}\left(f_{\theta}(z_{i}) - \beta\right), & \text{ if } d > d_{VC}(\mathcal{F}).
		\end{cases}
	\end{align*}
\end{linenomath}
Observe that $\delta_{\mathcal{F}}(d)$ is a non-increasing function of $d$. We denote $\delta_{\mathcal{F}} \coloneqq \delta_{\mathcal{F}}(d_{VC}(\mathcal{F}))$. The VC dimension of $\mathcal{F}$ can be defined based on $\delta_{\mathcal{F}}(d)$.

\begin{lemma}
	\label{lemma_margin}
	\begin{align*}
		d_{VC}(\mathcal{F}) = d \iff \delta_{\mathcal{F}}(d) > 0 \text{ and } \delta_{\mathcal{F}}(d+1) \leq 0
	\end{align*}
\end{lemma}
\begin{proof}
	We show that $\delta_{\mathcal{F}}(d) > 0$ if, and only if, $\mathcal{F}$ shatters some $d$ points in $\mathcal{Z}$ so the result follows. On the one hand, if $\mathcal{F}$ shatters $\{z_{1},\dots,z_{d}\}$, then for all $b \in \{-1,1\}^{d}$ there exists $\theta_{b} \in \Theta$ and $\beta_{b} \in (0,C_{\mathcal{F}})$ such that
	\begin{linenomath}
		\begin{align*}
			b_{i}\left(f_{\theta_{b}}(z_{i}) - \beta_{b}\right) > 0 \text{ for all } i = 1,\dots,d,
		\end{align*}
	\end{linenomath}
	and hence $\delta_{\mathcal{F}}(d) > 0$. On the other hand, if $\delta_{\mathcal{F}}(d) =\delta > 0$ then there exists a sequence $\{z_{1},\dots,z_{d}\}$ such that, for all $b \in \{-1,1\}^{d}$ there exists $\theta_{b} \in \Theta$ and $\beta_{b} \in (0,C_{\mathcal{F}})$ such that
	\begin{linenomath}
		\begin{align*}
			b_{i}\left(f_{\theta_{b}}(z_{i}) - \beta_{b}\right) > \delta/2 > 0 \text{ for all } i = 1,\dots,d,
		\end{align*}
	\end{linenomath}
	and hence $\mathcal{F}$ shatters some $d$ points in $\mathcal{Z}$.
\end{proof}

The $d$-margins are also associated with the fat-shattering dimension \citep{bartlett1994fat} of a set of real-valued functions which can be defined based on them.

\begin{definition}
	\label{def_fat}
	The $\gamma$ fat-shattering dimension of a set $\mathcal{F}$ is the greatest integer $d$ such that $\delta_{\mathcal{F}}(d) \geq \gamma$.
\end{definition}

It follows from Lemma \ref{lemma_margin} that if $\mathcal{G}$ is close enough to $\mathcal{F}$, then its VC dimension cannot differ too much from that of $\mathcal{F}$.

\begin{lemma}
	\label{lemma_closeVC}
	For $d \geq 1$,
	\begin{linenomath}
		\begin{align*}
			d(\mathcal{F},\mathcal{G}) < |\delta_{\mathcal{F}}(d_{VC}(\mathcal{F}) + d)| \implies d_{VC}(\mathcal{G}) < d_{VC}(\mathcal{F}) + d.
		\end{align*}
	\end{linenomath}
	and
	\begin{linenomath}
		\begin{align*}
			d(\mathcal{F},\mathcal{G}) < \delta_{\mathcal{F}} \implies d_{VC}(\mathcal{F}) \leq d_{VC}(\mathcal{G}). 
		\end{align*}
	\end{linenomath}
	In particular,
	\begin{linenomath}
		\begin{equation*}
			d(\mathcal{F},\mathcal{G}) < \min\{\delta_{\mathcal{F}},\delta_{\mathcal{G}}\} \implies d_{VC}(\mathcal{F}) = d_{VC}(\mathcal{G}). 
		\end{equation*}
	\end{linenomath}
\end{lemma}
\begin{proof}
	We assume that $\delta_{\mathcal{F}}(d_{VC}(\mathcal{F}) + d) < 0$ otherwise the first assertion holds trivially since then $\mathcal{G} = \mathcal{F}$. Fix a $\delta > 0$. If $d(\mathcal{F},\mathcal{G}) < \delta$ then, for any sequence $\{z_{1},\dots,z_{d}\}$ and $b \in \{-1,1\}^{d}$, it holds
	\begin{linenomath}
		\begin{align*}
			\left|\sup\limits_{\theta,\beta} \min_{i} b_{i}\left(f_{\theta}(z_{i}) - \beta\right) - \sup\limits_{\theta,\beta} \min_{i} b_{i}\left(g_{\theta}(z_{i}) - \beta\right) \right| < \delta
		\end{align*}
	\end{linenomath}
	and therefore
	\begin{linenomath}
		\begin{align*}
			\delta_{\mathcal{F}}(d) - \delta < \delta_{\mathcal{G}}(d) < \delta_{\mathcal{F}}(d) + \delta
		\end{align*}
	\end{linenomath}
	for all $d \geq 1$. The result follows from the inequality above and Lemma \ref{lemma_margin} since, if $d(\mathcal{F},\mathcal{G}) < |\delta_{\mathcal{F}}(d_{VC}(\mathcal{F}) + d)|$, then 
	\begin{linenomath}
		\begin{equation}
			\delta_{\mathcal{G}}(d_{VC}(\mathcal{F}) + d) < \delta_{\mathcal{F}}(d_{VC}(\mathcal{F}) + d) + |\delta_{\mathcal{F}}(d_{VC}(\mathcal{F}) + d)| = 0
		\end{equation}
	\end{linenomath}
	and if $d(\mathcal{F},\mathcal{G}) < \delta_{\mathcal{F}}$, then $0 < \delta_{\mathcal{G}}(d_{VC}(\mathcal{F}))$.
\end{proof}

It follows from Lemma \ref{lemma_closeVC} that the VC dimension is a continuous function with respect to the distance \eqref{dist_sets} when $\delta_{\mathcal{F}}(d_{VC}(\mathcal{F}) + 1) < 0$.

\begin{lemma}
	\label{lemma_ControlVC}
	If $\delta_{\mathcal{F}}(d_{VC}(\mathcal{F}) + d) < 0$ for a $d \geq 1$ then
	\begin{linenomath}
		\begin{equation*}
			\lim\limits_{n \to \infty} d(\mathcal{G}_{n},\mathcal{F}) = 0 \implies d_{VC}(\mathcal{F}) \leq \lim\limits_{n \to \infty} d_{VC}(\mathcal{G}_{n}) < d_{VC}(\mathcal{F}) + d.
		\end{equation*}
	\end{linenomath}
	In particular, if $\delta_{\mathcal{F}}(d_{VC}(\mathcal{F}) + 1) < 0$ then
	\begin{linenomath}
		\begin{equation*}
			\lim\limits_{n \to \infty} d(\mathcal{G}_{n},\mathcal{F}) = 0 \implies \lim\limits_{n \to \infty} d_{VC}(\mathcal{G}_{n}) = d_{VC}(\mathcal{F}).
		\end{equation*}
	\end{linenomath}
\end{lemma}
\begin{proof}
	That $\lim d_{VC}(\mathcal{G}_{n}) \geq d_{VC}(\mathcal{F})$ follows from Lemma \ref{lemma_closeVC}, and the fact that there exists a $n_{0}$ such that $d(\mathcal{G}_{n},\mathcal{F}) < \delta_{\mathcal{F}}$ for all $n > n_{0}$ when this distance converges to zero. That $\lim d_{VC}(\mathcal{G}_{n}) < d_{VC}(\mathcal{F}) + d$ for $d \geq 1$ whenever $\delta_{\mathcal{F}}(d_{VC}(\mathcal{F}) + d) < 0$ also follows from Lemma \ref{lemma_closeVC} and the fact that there exists a $n_{0}$ such that $d(\mathcal{G}_{n},\mathcal{F}) < |\delta_{\mathcal{F}}(d_{VC}(\mathcal{F}) + d)|$ for all $n > n_{0}$ when this distance converges to zero. 
\end{proof}

Denote by $\delta_{\mathcal{H},\ell}(d)$ the $d$-margin of hypothesis space $\mathcal{H}$ under loss function $\ell$. We define its marginal VC dimension.

\begin{definition}
	\label{def_marginalVC}
	The marginal VC dimension of a hypothesis space $\mathcal{H}$ under loss function $\ell$ is the least integer $d$ such that $\delta_{\mathcal{H},\ell}(d) < 0$ and is denoted by $m_{VC}(\mathcal{H},\ell)$. If $\delta_{\mathcal{H},\ell}(d) \geq 0$ for all $d \geq 1$ then $m_{VC}(\mathcal{H},\ell) = \infty$.
\end{definition}

It follows from Lemma \ref{lemma_margin} that $d_{VC}(\mathcal{H},\ell) < m_{VC}(\mathcal{H},\ell)$. Let $\{\ell_{n}\}$ be a sequence of loss functions that converges in $L^{\infty}$ to $\ell$. The main result of this section states that if $m_{VC}(\mathcal{H},\ell) < \infty$ then there exists a $n_{0}$ such that $d_{VC}(\mathcal{H},\ell_{n}) < \infty$ for all $n > n_{0}$.

\begin{proposition}
	\label{prop_ConvLF}
	Fix a hypothesis space $\mathcal{H}$ and let $\{\ell_{n}\}$ and $\ell$ be loss functions such that
	\begin{linenomath}
		\begin{equation*}
			\lim\limits_{n \to \infty} \sup_{z \in \mathcal{Z},\psi \in \mathcal{H}} |\ell_{n}(z,\psi) - \ell(z,\psi)| = 0.
		\end{equation*}
	\end{linenomath}
	If $m_{VC}(\mathcal{H},\ell) < \infty$, then 
	\begin{linenomath}
		\begin{equation*}
			\limsup\limits_{n_{0} \to \infty} d_{VC}(\mathcal{H},\ell_{n}) < \infty.
		\end{equation*}
	\end{linenomath}
\end{proposition}
\begin{proof}
	The result follows from inequality
	\begin{linenomath}
		\begin{equation*}
			\limsup\limits_{n \to \infty} d_{VC}(\mathcal{H},\ell_{n}) < m_{VC}(\mathcal{H},\ell)
		\end{equation*}
	\end{linenomath}
	which is a direct consequence of Lemma \ref{lemma_ControlVC}.
\end{proof}

The next lemma is a non-asymptotic version of Proposition \ref{prop_ConvLF}.

\begin{lemma}
	\label{lemma_nonAss}
	Fix a hypothesis space $\mathcal{H}$ and let $\ell_{1}$ and $\ell_{2}$ be loss functions such that
	\begin{linenomath}
		\begin{equation*}
			\sup_{z \in \mathcal{Z},\psi \in \mathcal{H}} |\ell_{1}(z,\psi) - \ell_{2}(z,\psi)| = \gamma
		\end{equation*}
	\end{linenomath}
	for a $\gamma > 0$. If there exists a $d \geq 1$ such that $\gamma < |\delta_{\mathcal{H},\ell_{1}}(d_{VC}(\mathcal{H},\ell_{1}) + d)|$, then 
	\begin{linenomath}
		\begin{equation*}
			d_{VC}(\mathcal{H},\ell_{2}) < d_{VC}(\mathcal{H},\ell_{1}) + \min \{d \geq 1: \gamma < |\delta_{\mathcal{H},\ell_{1}}(d_{VC}(\mathcal{H},\ell_{1}) + d)|\}.
		\end{equation*}
	\end{linenomath}    
	In particular, if $d_{VC}(\mathcal{H},\ell_{1}) < \infty$ then $d_{VC}(\mathcal{H},\ell_{2}) < \infty$.
\end{lemma}
\begin{proof}
	The result is a direct consequence of Lemma \ref{lemma_closeVC}.
\end{proof}

If $\ell$ is a binary loss function, then $m_{VC}(\mathcal{H},\ell) = d_{VC}(\mathcal{H},\ell) + 2$.

\begin{proposition}
	\label{prop_marginClass}
	If $\ell$ is a binary loss function and $\mathcal{H}$ is a hypothesis space with $d_{VC}(\mathcal{H},\ell) < \infty$, then $\delta_{\mathcal{H},\ell}(d_{VC}(\mathcal{H},\ell) + 2) = -1/2$. In particular, $m_{VC}(\mathcal{H},\ell) = d_{VC}(\mathcal{H},\ell) + 2$. 
\end{proposition}
\begin{proof}
	To easy notation, we will show this result for $\mathcal{F}$ assuming that it is a set of binary functions. Denote $d = d_{VC}(\mathcal{F})$, and fix $\boldsymbol{z} \in S(\mathcal{F})$ and $\{z_{d+1},z_{d+2}\} \subset \mathcal{Z}$. Observe that
	\begin{linenomath}
		\begin{equation*}
			\min\limits_{b \in \{-1,1\}^{d+2}} \sup\limits_{\theta \in \mathcal{F}(\boldsymbol{z},b)} \sup\limits_{\beta} \min\limits_{i = 1,\dots,d + 2} b_{i}\left(f_{\theta}(z_{i}) - \beta\right) = \sup\limits_{\beta} \min\{-\beta,-(1-\beta)\} = -1/2
		\end{equation*}
	\end{linenomath}
	as long as there exists a dichotomy $b \in \{-1,1\}^{d+2}$ with $b_{d+1} = 1$ and $b_{d+2} = -1$ such that for all $\theta \in \mathcal{F}(\boldsymbol{z},b)$, $f_{\theta}(z_{d+1}) = 0$ and $f_{\theta}(z_{d+2}) = 1$, or vice versa. Hence, it suffices to show this condition for all $\theta \in \mathcal{F}(\boldsymbol{z},b)$.
	
	Define the following subsets of $\{-1,1\}^{d+2}$ for a $\boldsymbol{z} \in S(\mathcal{F})$ and $e \in \{-1,1\}$:
	\begin{linenomath}
		\begin{align*}
			B(d+1,e) = \left\{b \in \{-1,1\}^{d+2}: b_{d+1} = e \text{ and } \mathcal{F}(\boldsymbol{z},b) \text{ cannot shatter } \boldsymbol{z} \cup \{z_{d+1}\} \text{ as } b\right\}\\
			B(d+2,e) = \left\{b \in \{-1,1\}^{d+2}: b_{d+2} = e \text{ and } \mathcal{F}(\boldsymbol{z},b) \text{ cannot shatter } \boldsymbol{z} \cup \{z_{d+2}\} \text{ as } b\right\}
		\end{align*}
	\end{linenomath}
	as the dichotomies $b$ with $b_{d+1} = e$ and $b_{d+2} = e$ that $\mathcal{F}(\boldsymbol{z},b)$ cannot realize with $\boldsymbol{z} \cup \{z_{d+1}\}$ and $\boldsymbol{z} \cup \{z_{d+2}\}$, respectively. Since $\mathcal{F}$ cannot shatter $d+2$ points, it follows that $B(d+1,1) \cup B(d+1,-1) \neq \emptyset$ and the same is true for $d+2$.
	
	We claim that if $B(d+1,1) \cap B(d+2,-1) \neq \emptyset$, then by taking $b \in B(d+1,1) \cap B(d+2,-1)$ it follows that for all $\theta \in \mathcal{F}(\boldsymbol{z},b)$, $f_{\theta}(z_{d+1}) = 0$ and $f_{\theta}(z_{d+2}) = 1$. Indeed, if $b \in B(d+1,1) \cap B(d+2,-1)$ then 
	\begin{linenomath}
		\begin{align*}
			\sup\limits_{\theta \in \mathcal{F}(\boldsymbol{z},b)} \sup\limits_{\beta} \left(f_{\theta}(z_{d+1}) - \beta\right) \leq 0 & & \text{ and } & & \sup\limits_{\theta \in \mathcal{F}(\boldsymbol{z},b)} \sup\limits_{\beta} -\left(f_{\theta}(z_{d+2}) - \beta\right) \leq 0
		\end{align*}
	\end{linenomath}
	from which follows that for all $\theta \in \mathcal{F}(\boldsymbol{z},b)$, $f_{\theta}(z_{d+1}) = 0$ and $f_{\theta}(z_{d+2}) = 1$. The condition of interest also follows if $B(d+1,-1) \cap B(d+2,1) \neq \emptyset$.
	
	We proceed by contradiction. Assume that
	\begin{linenomath}
		\begin{equation}
			\label{cap_empty}
			(B(d+1,1) \cap B(d+2,-1)) \cup (B(d+1,-1) \cap B(d+2,1)) = \emptyset.
		\end{equation}
	\end{linenomath}
	This implies that $\mathcal{F}$ can realize all dichotomies in $\{-1,1\}^{d+2}$ such that $b_{d+1} \neq d_{d+2}$. In special, it can realize all the dichotomies of the form $(b^\prime,b_{d+2})$ with $b^\prime \in \{-1,1\}^{d+1}$ and $b_{d+2} \neq b_{d+1}$. But this implies that $\mathcal{F}$ shatters $\{z_{1},\dots,z_{d+1}\}$ what is absurd since $d_{VC}(\mathcal{F}) = d$. Hence, \eqref{cap_empty} does not hold.
	
	Fix $\boldsymbol{z} \in S(\mathcal{F})$ and $z_{d+1} \in \mathcal{Z}$ and observe that
	\begin{linenomath}
		\begin{equation*}
			\min\limits_{b \in \{-1,1\}^{d+1}} \sup\limits_{\theta \in \mathcal{F}(\boldsymbol{z},b)} \sup\limits_{\beta} \min\limits_{i = 1,\dots,d+1} b_{i}\left(f_{\theta}(z_{i}) - \beta\right) = \sup\limits_{\beta} -\beta = 0
		\end{equation*}
	\end{linenomath}
	since $\mathcal{F}$ cannot shatter $d + 1$ points. This result implies that $\delta_{\mathcal{F}}(d_{VC}(\mathcal{F}) + 1) = 0$. Therefore, $m_{VC}(\mathcal{H},\ell) = d_{VC}(\mathcal{H},\ell) + 2$ when $\ell$ is a binary loss function.
\end{proof}

We combine Propositions \ref{prop_marginClass} and Lemmas \ref{lemma_ControlVC} and \ref{lemma_nonAss} to establish that if a sequence of loss functions $\{\ell_{n}\}$ converges to a binary loss $\ell$, then the limiting VC dimension does not differ in more than one unit from $d_{VC}(\mathcal{H},\ell)$.

\begin{corollary}
	Fix a hypothesis space $\mathcal{H}$ and a binary loss function with $d_{VC}(\mathcal{H},\ell) < \infty$. Let $\{\ell_{n}\}$ be a sequence of loss functions such that
	\begin{linenomath}
		\begin{equation*}
			\lim\limits_{n \to \infty} \sup_{z \in \mathcal{Z},\psi \in \mathcal{H}} |\ell_{n}(z,\psi) - \ell(z,\psi)| = 0.
		\end{equation*}
	\end{linenomath}
	Then 
	\begin{linenomath}
		\begin{equation*}
			\limsup\limits_{n \to \infty} d_{VC}(\mathcal{H},\ell_{n}) \leq d_{VC}(\mathcal{H},\ell) + 1.
		\end{equation*}
	\end{linenomath}
	In particular, if $\sup_{z \in \mathcal{Z},\psi \in \mathcal{H}} |\ell_{n}(z,\psi) - \ell(z,\psi)| < 1/2$ for some $n$, then 
	\begin{linenomath}
		\begin{equation*}
			d_{VC}(\mathcal{H},\ell_{n}) \leq d_{VC}(\mathcal{H},\ell) + 1.
		\end{equation*}
	\end{linenomath}
\end{corollary}

\section{VC dimension under real-valued loss functions}
\label{SecSquareLoss}

\subsection{VC dimension under the quadratic loss function}

If $\ell$ is the quadratic loss function, denoting, for $x \in \mathcal{X}, y \in \mathcal{Y}, 0 < \beta < C_{\ell}$ and $\psi \in \mathcal{H}$,
\begin{linenomath}
	\begin{align*}
		I_{\psi,\beta}(x,y) &= \mathds{1}\left\{(\psi(x) - y)^{2} > \beta\right\} \\
		&= \mathds{1}\left\{\psi(x) - y - \sqrt{\beta} > 0 \text{ or } \psi(x) - y + \sqrt{\beta} < 0\right\},
	\end{align*}
\end{linenomath}
we can associate $\mathcal{H}$ with set
\begin{linenomath}
	\begin{equation*}
		\mathcal{F}_{\mathcal{H},\ell} = \{I_{\psi,\beta}: \psi \in \mathcal{H}, \beta \in (0,C_{\ell})\}
	\end{equation*}
\end{linenomath}
of classifiers in $\mathcal{Z}$ satisfying $d_{VC}(\mathcal{H},\ell) = d_{VC}(\mathcal{F}_{\mathcal{H},\ell})$. Furthermore, we can bound $d_{VC}(\mathcal{H},\ell)$ by a function of the VC dimension of the sets
\begin{linenomath}
	\begin{align*}
		\mathcal{F}^{+}_{\mathcal{H},\ell} &= \left\{I_{\psi,\beta}^{+}(x,y) = \mathds{1}\{\psi(x) - y - \sqrt{\beta} > 0\}: \psi \in \mathcal{H}, \beta \in (0,C_{\ell})\right\}\\
		\mathcal{F}^{-}_{\mathcal{H},\ell} &= \left\{I_{\psi,\beta}^{(-)}(x,y) = \mathds{1}\{\psi(x) - y + \sqrt{\beta} < 0\}: \psi \in \mathcal{H}, \beta \in (0,C_{\ell})\right\}
	\end{align*}
\end{linenomath}
of binary functions from $\mathcal{X} \times \mathcal{Y}$ to $\{0,1\}$. 

\begin{lemma}
	\label{lemma_VCsqloss}
	Fix a hypothesis space $\mathcal{H}$ and let $\ell$ be the quadratic loss function. Then,
	\begin{linenomath}
		\begin{equation*}
			d_{VC}(\mathcal{H},\ell) \leq 4(1 + \log 2) \max\{d_{VC}(\mathcal{F}^{+}_{\mathcal{H},\ell}),d_{VC}(\mathcal{F}^{-}_{\mathcal{H},\ell})\}.
		\end{equation*}
	\end{linenomath}
	In special, if $d_{VC}(\mathcal{F}^{+}_{\mathcal{H},\ell}) < \infty$ and $d_{VC}(\mathcal{F}^{-}_{\mathcal{H},\ell}) < \infty$, then $d_{VC}(\mathcal{H},\ell) < \infty$.
\end{lemma}
\begin{proof}[Proof of Lemma \ref{lemma_VCsqloss}]
	Let $g: \{0,1\}^{2} \to \{0,1\}$ be the Boolean function given by $g(b_{1},b_{2}) = \max\{b_{1},b_{2}\}$ and observe that
	\begin{linenomath}
		\begin{equation*}
			I_{\psi,\beta}(x,y) = g(I_{\psi,\beta}^{+}(x,y),I_{\psi,\beta}^{-}(x,y))
		\end{equation*}
	\end{linenomath}
	for $x \in \mathcal{X}, y \in \mathcal{Y}, 0 < \beta < C$ and $\psi \in \mathcal{H}$. Lemma 2 in \cite{sontag1998vc} states that, given a Boolean function $g$, and two sets $\mathcal{H}_{1}$ and $\mathcal{H}_{2}$ of binary functions, if $\mathcal{F} = \{g(h_{1},h_{2}): h_{1} \in \mathcal{H}_{1}, h_{2} \in \mathcal{H}_{2}\}$ then
	\begin{linenomath}
		\begin{equation}
			\label{ine_sontag}
			d_{VC}(\mathcal{F}) \leq 4(1 + \log 2)\max\{d_{VC}(\mathcal{H}_{1}),d_{VC}(\mathcal{H}_{2})\}.
		\end{equation}
	\end{linenomath}
	Since
	\begin{linenomath}
		\begin{equation*}
			\mathcal{F}_{\mathcal{H},\ell} \subseteq \{g(h_{1},h_{2}): h_{1} \in \mathcal{F}^{+}_{\mathcal{H},\ell}, h_{2} \in \mathcal{F}^{-}_{\mathcal{H},\ell}\},
		\end{equation*}
	\end{linenomath}
	it follows from \eqref{ine_sontag} that
	\begin{linenomath}
		\begin{align*}
			d_{VC}(\mathcal{H},\ell) &= d_{VC}(\mathcal{F}_{\mathcal{H},\ell})\\
			&\leq d_{VC}(\{g(h_{1},h_{2}): h_{1} \in \mathcal{F}^{+}_{\mathcal{H},\ell}, h_{2} \in \mathcal{F}^{-}_{\mathcal{H},\ell}\})\\
			&\leq 4(1 + \log 2) \max\{d_{VC}(\mathcal{F}^{+}_{\mathcal{H},\ell}),d_{VC}(\mathcal{F}^{-}_{\mathcal{H},\ell})\}
		\end{align*}
	\end{linenomath}
	in which the first inequality follows from the fact that the VC dimension is non-decreasing on inclusion.
\end{proof}

The bound of Lemma \ref{lemma_VCsqloss} is clearly not tight, but can be a tool to show that $d_{VC}(\mathcal{H},\ell)$ is finite. For example, in the case of linear regression, in which
\begin{linenomath}
	\begin{equation*}
		\mathcal{H} = \left\{\psi(x) = a_{0} + \sum_{i=1}^{d} a_{i}x_{i}: a_{i} \in \mathbb{R}, i = 0,\dots,d\right\},
	\end{equation*}
\end{linenomath}
we have that, after some simple algebraic computations,
\begin{linenomath}
	\begin{align*}
		\mathcal{F}^{+}_{\mathcal{H},\ell} = \mathcal{F}^{-}_{\mathcal{H},\ell} = \left\{\mathds{1}\left\{a_{0} + \sum_{i=1}^{d} a_{i}x_{i} + a_{d+1} - y > 0\right\}: a_{i} \in \mathbb{R}, i = 0,\dots,d+1\right\},
	\end{align*}
\end{linenomath}
so $\mathcal{F}^{+}_{\mathcal{H},\ell}$ and $\mathcal{F}^{-}_{\mathcal{H},\ell}$ is the set of linear classifiers on $d+1$ variables. It is well-known that $d_{VC}(\mathcal{F}^{+}_{\mathcal{H},\ell}) = d_{VC}(\mathcal{F}^{+}_{\mathcal{H},\ell}) = d+2$ and hence
\begin{linenomath}
	\begin{align*}
		d_{VC}(\mathcal{H},\ell) \leq 4(1 + \log 2)(d+2),
	\end{align*}
\end{linenomath}
which is linear on the number of variables $d$.

\subsection{VC dimension under $\ell_{\mathscr{B}_{n}}$}
\label{AppVClQ}

There is no general relation between $d_{VC}(\mathcal{H},\ell)$ and $d_{VC}(\mathcal{H},\ell_{\mathscr{B}})$ which holds for any loss function $\ell$ and smoothing probabilities $\mathcal{B}$. Actually, it is possible to construct pathological cases in which $d_{VC}(\mathcal{H},\ell) < \infty$, but $d_{VC}(\mathcal{H},\ell_{\mathscr{B}}) = \infty$. Nevertheless, it is possible to associate these VC dimensions in some specific cases.

We start by considering linear regression under the quadratic loss function.

\begin{lemma}
	\label{lemma_VClinear}
	Fix a class $\mathcal{B}_{n}$ of probability distributions for $n \geq 1$ with mean $z$ and kernel $K_{n,\psi}$ that does not depend on $z$, and let $\mathcal{H}$ be a hypothesis space of linear functions in $d \geq 1$ variables and $\ell$ be the quadratic loss function. Then,
	\begin{linenomath}
		\begin{equation*}
			d_{VC}(\mathcal{H},\ell_{\mathscr{B}_{n}}) = d_{VC}(\mathcal{H},\ell).
		\end{equation*}    
	\end{linenomath}
\end{lemma}
\begin{proof}
	For each $\psi \in \mathcal{H}$, denote by
	\begin{linenomath}
		\begin{equation*}
			R(\psi) = \frac{1}{2} \sum\limits_{i,j=1}^{d+1} \frac{\partial^{2}\ell}{\partial z_{i} \partial z_{j}}(\bar{z},\psi) K_{n,\psi}
		\end{equation*}
	\end{linenomath}
	the remainder of the degree two Taylor expansion of $\ell_{\mathscr{B}_{n}}(z,\psi)$ on $z$ so it holds
	\begin{linenomath}
		\begin{equation*}
			\ell_{\mathscr{B}_{n}}(z,\psi) = \ell(z,\psi) + R(\psi).
		\end{equation*}
	\end{linenomath}
	Observe that $R(\psi)$ does not depend on $z$ since $\ell(z,\psi)$ is a degree two polynomial on $z$. Since
	\begin{linenomath}
		\begin{align*}
			\mathcal{A}^{\star}_{\mathcal{H},\ell_{\mathscr{B}_{n}}} &= \left\{\left\{z \in \mathcal{Z}: \ell(z,\psi) > \beta - R(\psi)\right\}: \psi \in \mathcal{H},0 < \beta < C_{\ell}\right\}\\
			&= \left\{\{\mathcal{Z}\}\right\} \bigcup \left\{\left\{z \in \mathcal{Z}: \ell(z,\psi) > \beta\right\}: \psi \in \mathcal{H},0 < \beta < C_{\ell}\right\}\\
			&= \left\{\{\mathcal{Z}\}\right\} \bigcup \mathcal{A}^{\star}_{\mathcal{H},\ell}.
		\end{align*}
	\end{linenomath}
	and $\left\{\{\mathcal{Z}\}\right\}$ generates the dichotomy of all ones, which can already be realized by sets in $\mathcal{A}^{\star}_{\mathcal{H},\ell}$ for all $n$, the VC dimension remains the same.
\end{proof}

\begin{remark}
	From the proof of Lemma \ref{lemma_VClinear} follows that $d_{VC}(\mathcal{H},\ell_{\mathscr{B}_{n}}) \leq d_{VC}(\mathcal{H},\ell) + 1$ whenever $\ell_{\mathscr{B}_{n}}(z,\psi) = \ell(z,\psi) + f(\psi)$ in which $f$ is a function that depends solely on $\psi$ and not on $z$.
\end{remark}

From Lemma \ref{lemma_nonAss} follows a general relation between $d_{VC}(\mathcal{H},\ell_{\mathscr{B}_{n}})$ and $d_{VC}(\mathcal{H},\ell)$ based on the $d$-margins of $\mathcal{H}$ under $\ell$.

\begin{corollary}
	\label{cor_dmargilQ}
	Fix a loss function $\ell$ and a collection $\mathscr{B}_{n}$ of probability measures for each $n \geq 1$, and let $\mathcal{H}$ be a hypothesis space with $d_{VC}(\mathcal{H},\ell) < \infty$. If there exists a $d \geq 1$ such that
	\begin{linenomath}
		\begin{equation*}
			\limsup\limits_{n \to \infty} \sup_{z \in \mathcal{Z},\psi \in \mathcal{H}} |\ell_{\mathscr{B}_{n}}(z,\psi) - \ell(z,\psi)| < |\delta_{\mathcal{H},\ell}(d_{VC}(\mathcal{H},\ell) + d)|
		\end{equation*}
	\end{linenomath}
	then 
	\begin{linenomath}
		\begin{equation*}
			\limsup\limits_{n \to \infty} d_{VC}(\mathcal{H},\ell_{\mathscr{B}_{n}}) < d_{VC}(\mathcal{H},\ell) + d < \infty.
		\end{equation*}
	\end{linenomath}    
\end{corollary}

Corollary \ref{cor_dmargilQ} combined with Proposition \ref{prop_marginClass} yields a sufficient condition for the finiteness of $\limsup\limits_{n \to \infty} d_{VC}(\mathcal{H},\ell_{\mathscr{B}_{n}})$ in classification problems.

\begin{corollary}
	Fix a binary loss function $\ell$ and a collection $\mathscr{B}_{n}$ of probability measures for each $n \geq 1$, and let $\mathcal{H}$ be a hypothesis space with $d_{VC}(\mathcal{H},\ell) < \infty$. If 
	\begin{linenomath}
		\begin{equation*}
			\limsup\limits_{n \to \infty} \sup_{z \in \mathcal{Z},\psi \in \mathcal{H}} |\ell_{\mathscr{B}_{n}}(z,\psi) - \ell(z,\psi)| < 1/2
		\end{equation*}
	\end{linenomath}
	then 
	\begin{linenomath}
		\begin{equation*}
			\limsup\limits_{n \to \infty} d_{VC}(\mathcal{H},\ell_{\mathscr{B}_{n}}) \leq d_{VC}(\mathcal{H},\ell) + 1 < \infty.
		\end{equation*}
	\end{linenomath}    
\end{corollary}

\FloatBarrier
\section{Detailed results of the experiments}
\label{AppExpRes}

\scriptsize

\begin{longtable}{rrrrrllllll}
		\caption{Bias of the estimators in each scenario over the 100 samples for $d = 1$. The least absolute bias in each scenario is in boldface.} \label{tab_res1} \\
		\hline
		d & $\sigma$ & n & $p_{g}$ & $p_{f}$ & Resub & Post & X - MPE & XY - MPE & X - MM & MPE - Post \\ 
		\hline
		1 & 0.25 & 20 & 1 & 1 & -0.011 & 0.012 & -0.006 & -0.0067 & \textbf{-0.0016} & 0.017 \\ 
		1 & 0.25 & 20 & 1 & 2 & -0.021 & \textbf{0.0087} & -0.014 & -0.015 & -0.0095 & 0.016 \\ 
		1 & 0.25 & 20 & 2 & 1 & -0.017 & \textbf{0.0083} & 0.029 & -0.013 & 0.062 & 0.054 \\ 
		1 & 0.25 & 20 & 2 & 2 & -0.023 & \textbf{0.0091} & 0.024 & -0.015 & 0.056 & 0.057 \\ 
		1 & 0.25 & 20 & 3 & 1 & -0.039 & 0.031 & 0.21 & \textbf{-0.027} & 0.38 & 0.28 \\ 
		1 & 0.25 & 20 & 3 & 2 & -0.02 & \textbf{0.013} & 0.29 & 0.038 & 0.4 & 0.32 \\ 
		\hline 1 & 0.25 & 50 & 1 & 1 & -0.0063 & \textbf{0.0019} & -0.0044 & -0.0046 & -0.0032 & 0.0038 \\ 
		1 & 0.25 & 50 & 1 & 2 & -0.007 & \textbf{0.0037} & -0.0049 & -0.0051 & -0.0041 & 0.006 \\ 
		1 & 0.25 & 50 & 2 & 1 & -0.0086 & \textbf{0.00064} & 0.0083 & -0.0068 & 0.016 & 0.018 \\ 
		1 & 0.25 & 50 & 2 & 2 & -0.0076 & \textbf{0.0041} & 0.01 & -0.0048 & 0.021 & 0.022 \\ 
		1 & 0.25 & 50 & 3 & 1 & -0.017 & \textbf{0.0077} & 0.072 & -0.013 & 0.13 & 0.097 \\ 
		1 & 0.25 & 50 & 3 & 2 & \textbf{-0.0059} & 0.0065 & 0.098 & 0.013 & 0.15 & 0.11 \\ 
		\hline 1 & 0.25 & 100 & 1 & 1 & -0.0035 & \textbf{0.00054} & -0.0026 & -0.0027 & -0.0022 & 0.0014 \\ 
		1 & 0.25 & 100 & 1 & 2 & -0.0036 & \textbf{0.0017} & -0.0026 & -0.0027 & -0.0021 & 0.0027 \\ 
		1 & 0.25 & 100 & 2 & 1 & -0.0034 & \textbf{0.0012} & 0.0047 & -0.0025 & 0.0094 & 0.0094 \\ 
		1 & 0.25 & 100 & 2 & 2 & -0.0046 & \textbf{0.00097} & 0.0039 & -0.0033 & 0.0082 & 0.0094 \\ 
		1 & 0.25 & 100 & 3 & 1 & -0.012 & \textbf{0.00026} & 0.031 & -0.0098 & 0.058 & 0.043 \\ 
		1 & 0.25 & 100 & 3 & 2 & -0.003 & \textbf{0.0028} & 0.048 & 0.006 & 0.073 & 0.054 \\ 
		\hline 1 & 0.50 & 20 & 1 & 1 & -0.042 & 0.045 & -0.036 & \textbf{-0.021} & -0.032 & 0.051 \\ 
		1 & 0.50 & 20 & 1 & 2 & -0.077 & \textbf{0.03} & -0.065 & -0.052 & -0.058 & 0.043 \\ 
		1 & 0.50 & 20 & 2 & 1 & -0.046 & 0.05 & \textbf{-0.00053} & -0.025 & 0.034 & 0.095 \\ 
		1 & 0.50 & 20 & 2 & 2 & -0.092 & 0.024 & -0.033 & -0.061 & \textbf{-0.0015} & 0.085 \\ 
		1 & 0.50 & 20 & 3 & 1 & -0.091 & \textbf{0.045} & 0.16 & -0.064 & 0.32 & 0.29 \\ 
		1 & 0.50 & 20 & 3 & 2 & -0.069 & 0.067 & 0.23 & \textbf{0.018} & 0.39 & 0.36 \\ 
		\hline 1 & 0.50 & 50 & 1 & 1 & -0.017 & 0.016 & -0.014 & \textbf{-0.0085} & -0.014 & 0.018 \\ 
		1 & 0.50 & 50 & 1 & 2 & -0.027 & \textbf{0.015} & -0.024 & -0.018 & -0.023 & 0.018 \\ 
		1 & 0.50 & 50 & 2 & 1 & -0.03 & 0.0041 & -0.013 & -0.022 & \textbf{-0.0025} & 0.021 \\ 
		1 & 0.50 & 50 & 2 & 2 & -0.032 & 0.012 & -0.014 & -0.022 & \textbf{-0.0031} & 0.03 \\ 
		1 & 0.50 & 50 & 3 & 1 & -0.036 & \textbf{0.014} & 0.053 & -0.025 & 0.091 & 0.1 \\ 
		1 & 0.50 & 50 & 3 & 2 & -0.04 & \textbf{0.0047} & 0.067 & -0.0097 & 0.13 & 0.11 \\ 
		\hline 1 & 0.50 & 100 & 1 & 1 & -0.012 & \textbf{0.0038} & -0.011 & -0.0078 & -0.01 & 0.0046 \\ 
		1 & 0.50 & 100 & 1 & 2 & -0.017 & \textbf{0.0034} & -0.016 & -0.013 & -0.015 & 0.0045 \\ 
		1 & 0.50 & 100 & 2 & 1 & -0.01 & 0.0068 & \textbf{-0.002} & -0.0063 & 0.004 & 0.015 \\ 
		1 & 0.50 & 100 & 2 & 2 & -0.014 & 0.0081 & -0.0052 & -0.0094 & \textbf{-0.00062} & 0.016 \\ 
		1 & 0.50 & 100 & 3 & 1 & -0.025 & \textbf{-0.00083} & 0.017 & -0.02 & 0.046 & 0.041 \\ 
		1 & 0.50 & 100 & 3 & 2 & -0.017 & 0.0053 & 0.036 & \textbf{-0.0026} & 0.062 & 0.058 \\ 
		\hline
\end{longtable}

\begin{longtable}{rrrrrllllll}
		\caption{Bias of the estimators in each scenario over the 100 samples for $d = 2$. The least absolute bias in each scenario is in boldface.}  \\
		\hline
		d & $\sigma$ & n & $p_{g}$ & $p_{f}$ & Resub & Post & X - MPE & XY - MPE & X - MM & MPE - Post \\ 
		\hline
		2 & 0.25 & 20 & 1 & 1 & -0.021 & \textbf{0.0085} & -0.0087 & -0.016 & 0.022 & 0.021 \\ 
		2 & 0.25 & 20 & 1 & 2 & -0.053 & \textbf{-0.0013} & -0.034 & -0.042 & 0.012 & 0.018 \\ 
		2 & 0.25 & 20 & 2 & 1 & -0.038 & \textbf{0.014} & 0.16 & -0.03 & 0.66 & 0.21 \\ 
		2 & 0.25 & 20 & 2 & 2 & -0.049 & \textbf{0.008} & 0.16 & -0.024 & 0.76 & 0.22 \\ 
		2 & 0.25 & 20 & 3 & 1 & -0.68 & \textbf{0.038} & 1.26 & -0.58 & 6.09 & 1.97 \\ 
		2 & 0.25 & 20 & 3 & 2 & -0.064 & \textbf{0.0027} & 2.36 & 0.46 & 8.79 & 2.42 \\ 
		\hline 2 & 0.25 & 50 & 1 & 1 & -0.0053 & 0.006 & \textbf{-0.00082} & -0.0031 & 0.018 & 0.011 \\ 
		2 & 0.25 & 50 & 1 & 2 & -0.016 & \textbf{0.0033} & -0.01 & -0.013 & 0.0091 & 0.0083 \\ 
		2 & 0.25 & 50 & 2 & 1 & -0.01 & 0.0094 & 0.062 & \textbf{-0.0069} & 0.34 & 0.082 \\ 
		2 & 0.25 & 50 & 2 & 2 & -0.015 & \textbf{0.0054} & 0.062 & -0.0076 & 0.35 & 0.083 \\ 
		2 & 0.25 & 50 & 3 & 1 & -0.27 & \textbf{-0.0092} & 0.4 & -0.23 & 3.11 & 0.66 \\ 
		2 & 0.25 & 50 & 3 & 2 & -0.023 & \textbf{-0.00033} & 0.83 & 0.15 & 4.02 & 0.85 \\ 
		\hline 2 & 0.25 & 100 & 1 & 1 & -0.0038 & \textbf{0.0015} & -0.0017 & -0.0028 & 0.0092 & 0.0037 \\ 
		2 & 0.25 & 100 & 1 & 2 & -0.0082 & \textbf{0.001} & -0.0058 & -0.007 & 0.0053 & 0.0035 \\ 
		2 & 0.25 & 100 & 2 & 1 & -0.0071 & \textbf{0.0021} & 0.027 & -0.0054 & 0.2 & 0.037 \\ 
		2 & 0.25 & 100 & 2 & 2 & -0.0073 & \textbf{0.0024} & 0.03 & -0.0036 & 0.2 & 0.039 \\ 
		2 & 0.25 & 100 & 3 & 1 & -0.13 & \textbf{-0.0027} & 0.2 & -0.11 & 1.81 & 0.33 \\ 
		2 & 0.25 & 100 & 3 & 2 & -0.012 & \textbf{-0.00081} & 0.41 & 0.074 & 2.15 & 0.42 \\ 
		\hline 2 & 0.50 & 20 & 1 & 1 & -0.091 & \textbf{0.019} & -0.076 & -0.071 & -0.033 & 0.034 \\ 
		2 & 0.50 & 20 & 1 & 2 & -0.19 & \textbf{-0.0031} & -0.16 & -0.16 & -0.071 & 0.028 \\ 
		2 & 0.50 & 20 & 2 & 1 & -0.1 & \textbf{0.04} & 0.078 & -0.08 & 0.53 & 0.22 \\ 
		2 & 0.50 & 20 & 2 & 2 & -0.23 & \textbf{-0.011} & 0.019 & -0.17 & 0.68 & 0.24 \\ 
		2 & 0.50 & 20 & 3 & 1 & -0.77 & \textbf{0.0099} & 1.11 & -0.64 & 6.34 & 1.89 \\ 
		2 & 0.50 & 20 & 3 & 2 & -0.23 & \textbf{0.0062} & 2.34 & 0.37 & 8.3 & 2.58 \\ 
		\hline 2 & 0.50 & 50 & 1 & 1 & -0.035 & \textbf{0.006} & -0.03 & -0.027 & -0.013 & 0.011 \\ 
		2 & 0.50 & 50 & 1 & 2 & -0.061 & \textbf{0.015} & -0.053 & -0.049 & -0.027 & 0.022 \\ 
		2 & 0.50 & 50 & 2 & 1 & -0.047 & \textbf{0.0052} & 0.025 & -0.037 & 0.3 & 0.077 \\ 
		2 & 0.50 & 50 & 2 & 2 & -0.071 & \textbf{0.0059} & 0.009 & -0.054 & 0.29 & 0.086 \\ 
		2 & 0.50 & 50 & 3 & 1 & -0.22 & \textbf{0.085} & 0.46 & -0.17 & 2.98 & 0.78 \\ 
		2 & 0.50 & 50 & 3 & 2 & -0.08 & \textbf{-0.0013} & 0.77 & 0.11 & 3.71 & 0.85 \\ 
		\hline 2 & 0.50 & 100 & 1 & 1 & -0.012 & 0.0098 & -0.0094 & -0.0071 & \textbf{0.0014} & 0.012 \\ 
		2 & 0.50 & 100 & 1 & 2 & -0.023 & 0.015 & -0.02 & -0.018 & \textbf{-0.0082} & 0.017 \\ 
		2 & 0.50 & 100 & 2 & 1 & -0.016 & \textbf{0.011} & 0.019 & -0.011 & 0.18 & 0.045 \\ 
		2 & 0.50 & 100 & 2 & 2 & -0.027 & 0.011 & \textbf{0.0096} & -0.02 & 0.18 & 0.048 \\ 
		2 & 0.50 & 100 & 3 & 1 & -0.11 & \textbf{0.042} & 0.23 & -0.08 & 1.8 & 0.37 \\ 
		2 & 0.50 & 100 & 3 & 2 & -0.037 & \textbf{0.0016} & 0.38 & 0.059 & 2.08 & 0.42 \\ 
		\hline
\end{longtable}

\begin{longtable}{rrrrrllllll}
		\caption{Bias of the estimators in each scenario over the 100 samples for $d = 3$. The least absolute bias in each scenario is in boldface.} \\
		\hline
		d & $\sigma$ & n & $p_{g}$ & $p_{f}$ & Resub & Post & X - MPE & XY - MPE & X - MM & MPE - Post \\ 
		\hline
		3 & 0.25 & 20 & 1 & 1 & -0.028 & 0.0096 & \textbf{-0.0057} & -0.023 & 0.077 & 0.032 \\ 
		3 & 0.25 & 20 & 1 & 2 & -0.12 & -0.041 & -0.086 & -0.1 & 0.065 & \textbf{-0.0062} \\ 
		3 & 0.25 & 20 & 2 & 1 & -0.09 & \textbf{0.0069} & 0.43 & -0.078 & 2.57 & 0.52 \\ 
		3 & 0.25 & 20 & 2 & 2 & -0.12 & \textbf{-0.022} & 0.46 & -0.074 & 2.53 & 0.56 \\ 
		3 & 0.25 & 20 & 3 & 1 & -3.31 & \textbf{0.081} & 4.54 & -2.88 & 35.06 & 8 \\ 
		3 & 0.25 & 20 & 3 & 2 & -0.23 & \textbf{-0.065} & 10.04 & 1.87 & 45.51 & 10.19 \\ 
		\hline 3 & 0.25 & 50 & 1 & 1 & -0.012 & \textbf{0.0018} & -0.0032 & -0.0097 & 0.049 & 0.011 \\ 
		3 & 0.25 & 50 & 1 & 2 & -0.032 & \textbf{-0.0029} & -0.022 & -0.028 & 0.04 & 0.0069 \\ 
		3 & 0.25 & 50 & 2 & 1 & -0.029 & \textbf{0.0086} & 0.18 & -0.024 & 1.42 & 0.22 \\ 
		3 & 0.25 & 50 & 2 & 2 & -0.028 & \textbf{0.0029} & 0.2 & -0.011 & 1.47 & 0.23 \\ 
		3 & 0.25 & 50 & 3 & 1 & -1.13 & \textbf{0.16} & 2.01 & -0.93 & 20.49 & 3.32 \\ 
		3 & 0.25 & 50 & 3 & 2 & -0.073 & \textbf{-0.022} & 3.91 & 0.7 & 24.96 & 3.96 \\ 
		\hline 3 & 0.25 & 100 & 1 & 1 & -0.0054 & 0.0014 & \textbf{-0.0012} & -0.0043 & 0.034 & 0.0055 \\ 
		3 & 0.25 & 100 & 1 & 2 & -0.012 & \textbf{0.0026} & -0.0076 & -0.011 & 0.03 & 0.0071 \\ 
		3 & 0.25 & 100 & 2 & 1 & -0.016 & \textbf{0.0016} & 0.09 & -0.013 & 0.96 & 0.11 \\ 
		3 & 0.25 & 100 & 2 & 2 & -0.014 & \textbf{0.00025} & 0.099 & -0.0064 & 0.99 & 0.11 \\ 
		3 & 0.25 & 100 & 3 & 1 & -0.45 & \textbf{0.18} & 1.12 & -0.35 & 13.7 & 1.77 \\ 
		3 & 0.25 & 100 & 3 & 2 & -0.032 & \textbf{-0.007} & 1.94 & 0.33 & 17.13 & 1.96 \\ 
		\hline 3 & 0.50 & 20 & 1 & 1 & -0.11 & 0.039 & -0.085 & -0.087 & \textbf{0.0086} & 0.064 \\ 
		3 & 0.50 & 20 & 1 & 2 & -0.51 & -0.21 & -0.42 & -0.43 & \textbf{-0.059} & -0.13 \\ 
		3 & 0.50 & 20 & 2 & 1 & -0.16 & \textbf{0.067} & 0.39 & -0.14 & 2.43 & 0.62 \\ 
		3 & 0.50 & 20 & 2 & 2 & -0.72 & -0.35 & \textbf{-0.099} & -0.61 & 2.2 & 0.27 \\ 
		3 & 0.50 & 20 & 3 & 1 & -3.62 & \textbf{-0.16} & 4.45 & -3.18 & 37.28 & 7.91 \\ 
		3 & 0.50 & 20 & 3 & 2 & -0.72 & \textbf{-0.3} & 9.57 & 1.41 & 45.95 & 9.97 \\ 
		\hline 3 & 0.50 & 50 & 1 & 1 & -0.037 & \textbf{0.017} & -0.029 & -0.028 & 0.027 & 0.026 \\ 
		3 & 0.50 & 50 & 1 & 2 & -0.13 & -0.02 & -0.12 & -0.12 & -0.023 & \textbf{-0.005} \\ 
		3 & 0.50 & 50 & 2 & 1 & -0.076 & \textbf{0.00054} & 0.13 & -0.064 & 1.33 & 0.21 \\ 
		3 & 0.50 & 50 & 2 & 2 & -0.12 & \textbf{0.0051} & 0.11 & -0.09 & 1.42 & 0.23 \\ 
		3 & 0.50 & 50 & 3 & 1 & -1.31 & \textbf{0.033} & 1.86 & -1.11 & 21.75 & 3.18 \\ 
		3 & 0.50 & 50 & 3 & 2 & -0.15 & \textbf{-0.002} & 3.93 & 0.66 & 25.38 & 4.06 \\ 
		\hline 3 & 0.50 & 100 & 1 & 1 & -0.027 & \textbf{-0.0016} & -0.023 & -0.023 & 0.014 & 0.0029 \\ 
		3 & 0.50 & 100 & 1 & 2 & -0.055 & \textbf{0.0012} & -0.049 & -0.049 & -0.0039 & 0.0073 \\ 
		3 & 0.50 & 100 & 2 & 1 & -0.038 & \textbf{0.00019} & 0.065 & -0.031 & 0.91 & 0.1 \\ 
		3 & 0.50 & 100 & 2 & 2 & -0.055 & \textbf{0.0034} & 0.059 & -0.042 & 0.97 & 0.12 \\ 
		3 & 0.50 & 100 & 3 & 1 & -0.48 & \textbf{0.18} & 1.09 & -0.37 & 13.87 & 1.75 \\ 
		3 & 0.50 & 100 & 3 & 2 & -0.07 & \textbf{0.00087} & 2 & 0.32 & 16.8 & 2.07 \\ 
		\hline
\end{longtable}

\begin{longtable}{rrrrrllllll}
		\caption{RMSE of the estimators in each scenario over the 100 samples for $d = 1$. The least RMSE in each scenario is in boldface.} \\
		\hline
		d & $\sigma$ & n & $p_{g}$ & $p_{f}$ & Resub & Post & X - MPE & XY - MPE & X - MM & MPE - Post \\ 
		\hline
		1 & 0.25 & 20 & 1 & 1 & 0.026 & 0.033 & \textbf{0.024} & 0.026 & 0.026 & 0.035 \\ 
		1 & 0.25 & 20 & 1 & 2 & 0.032 & 0.034 & 0.028 & 0.029 & \textbf{0.028} & 0.036 \\ 
		1 & 0.25 & 20 & 2 & 1 & 0.026 & 0.028 & 0.038 & \textbf{0.025} & 0.085 & 0.062 \\ 
		1 & 0.25 & 20 & 2 & 2 & 0.034 & 0.035 & 0.038 & \textbf{0.031} & 0.084 & 0.068 \\ 
		1 & 0.25 & 20 & 3 & 1 & 0.076 & 0.091 & 0.23 & \textbf{0.073} & 0.5 & 0.3 \\ 
		1 & 0.25 & 20 & 3 & 2 & \textbf{0.029} & 0.035 & 0.3 & 0.053 & 0.5 & 0.33 \\ 
		\hline 1 & 0.25 & 50 & 1 & 1 & 0.013 & 0.012 & 0.012 & 0.012 & \textbf{0.012} & 0.013 \\ 
		1 & 0.25 & 50 & 1 & 2 & 0.013 & 0.014 & 0.012 & 0.013 & \textbf{0.012} & 0.014 \\ 
		1 & 0.25 & 50 & 2 & 1 & 0.016 & \textbf{0.015} & 0.016 & 0.015 & 0.029 & 0.023 \\ 
		1 & 0.25 & 50 & 2 & 2 & 0.014 & 0.015 & 0.016 & \textbf{0.013} & 0.029 & 0.026 \\ 
		1 & 0.25 & 50 & 3 & 1 & 0.042 & 0.044 & 0.084 & \textbf{0.041} & 0.17 & 0.11 \\ 
		1 & 0.25 & 50 & 3 & 2 & \textbf{0.013} & 0.015 & 0.1 & 0.018 & 0.19 & 0.11 \\ 
		\hline 1 & 0.25 & 100 & 1 & 1 & 0.0084 & 0.0082 & \textbf{0.0081} & 0.0082 & 0.0081 & 0.0083 \\ 
		1 & 0.25 & 100 & 1 & 2 & 0.0091 & 0.0093 & 0.0087 & 0.0089 & \textbf{0.0086} & 0.0094 \\ 
		1 & 0.25 & 100 & 2 & 1 & 0.011 & 0.011 & 0.011 & \textbf{0.01} & 0.015 & 0.014 \\ 
		1 & 0.25 & 100 & 2 & 2 & 0.0096 & 0.0093 & 0.0094 & \textbf{0.0092} & 0.013 & 0.013 \\ 
		1 & 0.25 & 100 & 3 & 1 & 0.024 & \textbf{0.022} & 0.038 & 0.023 & 0.072 & 0.049 \\ 
		1 & 0.25 & 100 & 3 & 2 & \textbf{0.0087} & 0.0093 & 0.049 & 0.01 & 0.081 & 0.055 \\ 
		\hline 1 & 0.50 & 20 & 1 & 1 & 0.09 & 0.12 & 0.087 & 0.089 & \textbf{0.085} & 0.12 \\ 
		1 & 0.50 & 20 & 1 & 2 & 0.11 & 0.11 & 0.1 & 0.099 & \textbf{0.096} & 0.11 \\ 
		1 & 0.50 & 20 & 2 & 1 & 0.091 & 0.12 & \textbf{0.082} & 0.089 & 0.11 & 0.14 \\ 
		1 & 0.50 & 20 & 2 & 2 & 0.12 & 0.11 & \textbf{0.087} & 0.1 & 0.099 & 0.14 \\ 
		1 & 0.50 & 20 & 3 & 1 & 0.15 & 0.17 & 0.22 & \textbf{0.14} & 0.46 & 0.35 \\ 
		1 & 0.50 & 20 & 3 & 2 & 0.12 & 0.15 & 0.26 & \textbf{0.12} & 0.54 & 0.4 \\ 
		\hline 1 & 0.50 & 50 & 1 & 1 & 0.058 & 0.064 & 0.057 & 0.058 & \textbf{0.057} & 0.065 \\ 
		1 & 0.50 & 50 & 1 & 2 & 0.064 & 0.067 & 0.062 & \textbf{0.061} & 0.062 & 0.068 \\ 
		1 & 0.50 & 50 & 2 & 1 & 0.061 & 0.061 & \textbf{0.056} & 0.059 & 0.061 & 0.064 \\ 
		1 & 0.50 & 50 & 2 & 2 & 0.063 & 0.065 & 0.057 & 0.06 & \textbf{0.054} & 0.071 \\ 
		1 & 0.50 & 50 & 3 & 1 & 0.074 & 0.075 & 0.085 & \textbf{0.071} & 0.14 & 0.13 \\ 
		1 & 0.50 & 50 & 3 & 2 & 0.061 & 0.054 & 0.086 & \textbf{0.05} & 0.17 & 0.13 \\ 
		\hline 1 & 0.50 & 100 & 1 & 1 & 0.039 & 0.04 & 0.039 & \textbf{0.039} & 0.039 & 0.04 \\ 
		1 & 0.50 & 100 & 1 & 2 & 0.036 & 0.034 & 0.035 & 0.034 & 0.035 & \textbf{0.034} \\ 
		1 & 0.50 & 100 & 2 & 1 & 0.04 & 0.042 & \textbf{0.038} & 0.04 & 0.039 & 0.043 \\ 
		1 & 0.50 & 100 & 2 & 2 & 0.038 & 0.04 & \textbf{0.036} & 0.038 & 0.037 & 0.042 \\ 
		1 & 0.50 & 100 & 3 & 1 & 0.055 & \textbf{0.052} & 0.052 & 0.053 & 0.081 & 0.067 \\ 
		1 & 0.50 & 100 & 3 & 2 & 0.041 & 0.041 & 0.053 & \textbf{0.039} & 0.089 & 0.072 \\ 
		\hline
\end{longtable}

\begin{longtable}{rrrrrllllll}
		\caption{RMSE of the estimators in each scenario over the 100 samples for $d = 2$. The least RMSE in each scenario is in boldface.} \\
		\hline
		d & $\sigma$ & n & $p_{g}$ & $p_{f}$ & Resub & Post & X - MPE & XY - MPE & X - MM & MPE - Post \\ 
		\hline
		2 & 0.25 & 20 & 1 & 1 & 0.027 & 0.027 & \textbf{0.02} & 0.025 & 0.033 & 0.033 \\ 
		2 & 0.25 & 20 & 1 & 2 & 0.064 & 0.047 & 0.049 & 0.055 & \textbf{0.047} & 0.051 \\ 
		2 & 0.25 & 20 & 2 & 1 & 0.057 & 0.06 & 0.18 & \textbf{0.054} & 0.73 & 0.23 \\ 
		2 & 0.25 & 20 & 2 & 2 & 0.06 & 0.045 & 0.18 & \textbf{0.042} & 0.84 & 0.24 \\ 
		2 & 0.25 & 20 & 3 & 1 & 1 & 1.06 & 1.74 & \textbf{0.97} & 6.88 & 2.49 \\ 
		2 & 0.25 & 20 & 3 & 2 & 0.074 & \textbf{0.053} & 2.51 & 0.53 & 9.76 & 2.57 \\ 
		\hline 2 & 0.25 & 50 & 1 & 1 & 0.015 & 0.017 & \textbf{0.014} & 0.015 & 0.024 & 0.019 \\ 
		2 & 0.25 & 50 & 1 & 2 & 0.019 & \textbf{0.015} & 0.015 & 0.017 & 0.018 & 0.017 \\ 
		2 & 0.25 & 50 & 2 & 1 & 0.024 & 0.027 & 0.069 & \textbf{0.023} & 0.36 & 0.089 \\ 
		2 & 0.25 & 50 & 2 & 2 & 0.021 & 0.019 & 0.066 & \textbf{0.017} & 0.38 & 0.086 \\ 
		2 & 0.25 & 50 & 3 & 1 & 0.55 & 0.56 & 0.7 & \textbf{0.54} & 3.37 & 0.94 \\ 
		2 & 0.25 & 50 & 3 & 2 & 0.029 & \textbf{0.022} & 0.85 & 0.17 & 4.27 & 0.87 \\ 
		\hline 2 & 0.25 & 100 & 1 & 1 & 0.0084 & 0.0081 & \textbf{0.0077} & 0.0081 & 0.013 & 0.009 \\ 
		2 & 0.25 & 100 & 1 & 2 & 0.012 & \textbf{0.0099} & 0.01 & 0.011 & 0.011 & 0.011 \\ 
		2 & 0.25 & 100 & 2 & 1 & 0.018 & 0.018 & 0.033 & \textbf{0.017} & 0.2 & 0.042 \\ 
		2 & 0.25 & 100 & 2 & 2 & 0.012 & 0.011 & 0.032 & \textbf{0.01} & 0.21 & 0.041 \\ 
		2 & 0.25 & 100 & 3 & 1 & 0.37 & 0.38 & 0.43 & \textbf{0.37} & 1.95 & 0.53 \\ 
		2 & 0.25 & 100 & 3 & 2 & 0.016 & \textbf{0.012} & 0.41 & 0.077 & 2.25 & 0.42 \\ 
		\hline 2 & 0.50 & 20 & 1 & 1 & 0.12 & 0.11 & 0.11 & 0.11 & \textbf{0.089} & 0.11 \\ 
		2 & 0.50 & 20 & 1 & 2 & 0.22 & 0.16 & 0.19 & 0.19 & \textbf{0.14} & 0.16 \\ 
		2 & 0.50 & 20 & 2 & 1 & 0.13 & 0.12 & 0.13 & \textbf{0.12} & 0.59 & 0.26 \\ 
		2 & 0.50 & 20 & 2 & 2 & 0.41 & 0.36 & \textbf{0.34} & 0.38 & 0.85 & 0.43 \\ 
		2 & 0.50 & 20 & 3 & 1 & 1.08 & 1.04 & 1.62 & \textbf{1.03} & 7.18 & 2.4 \\ 
		2 & 0.50 & 20 & 3 & 2 & 0.35 & \textbf{0.31} & 2.55 & 0.56 & 9.05 & 2.79 \\ 
		\hline 2 & 0.50 & 50 & 1 & 1 & 0.06 & 0.058 & 0.058 & 0.057 & \textbf{0.054} & 0.059 \\ 
		2 & 0.50 & 50 & 1 & 2 & 0.081 & 0.07 & 0.076 & 0.074 & \textbf{0.059} & 0.072 \\ 
		2 & 0.50 & 50 & 2 & 1 & 0.069 & \textbf{0.059} & 0.06 & 0.064 & 0.33 & 0.099 \\ 
		2 & 0.50 & 50 & 2 & 2 & 0.089 & 0.071 & \textbf{0.06} & 0.078 & 0.32 & 0.11 \\ 
		2 & 0.50 & 50 & 3 & 1 & 0.53 & 0.57 & 0.75 & \textbf{0.52} & 3.24 & 1.03 \\ 
		2 & 0.50 & 50 & 3 & 2 & 0.096 & \textbf{0.068} & 0.79 & 0.14 & 3.9 & 0.87 \\ 
		\hline 2 & 0.50 & 100 & 1 & 1 & 0.036 & 0.038 & 0.035 & 0.035 & \textbf{0.033} & 0.038 \\ 
		2 & 0.50 & 100 & 1 & 2 & 0.046 & 0.049 & 0.045 & 0.045 & \textbf{0.042} & 0.049 \\ 
		2 & 0.50 & 100 & 2 & 1 & 0.045 & 0.047 & 0.046 & \textbf{0.044} & 0.2 & 0.064 \\ 
		2 & 0.50 & 100 & 2 & 2 & 0.047 & 0.045 & \textbf{0.04} & 0.043 & 0.19 & 0.065 \\ 
		2 & 0.50 & 100 & 3 & 1 & 0.39 & 0.41 & 0.47 & \textbf{0.39} & 1.93 & 0.58 \\ 
		2 & 0.50 & 100 & 3 & 2 & 0.05 & \textbf{0.039} & 0.39 & 0.073 & 2.2 & 0.43 \\ 
		\hline
\end{longtable}

\begin{longtable}{rrrrrllllll}
		\caption{RMSE of the estimators in each scenario over the 100 samples for $d = 3$. The least RMSE in each scenario is in boldface.} \label{tab_res2} \\
		\hline
		d & $\sigma$ & n & $p_{g}$ & $p_{f}$ & Resub & Post & X - MPE & XY - MPE & X - MM & MPE - Post \\ 
		\hline
		3 & 0.25 & 20 & 1 & 1 & 0.035 & 0.034 & \textbf{0.023} & 0.032 & 0.088 & 0.046 \\ 
		3 & 0.25 & 20 & 1 & 2 & 0.14 & 0.094 & 0.11 & 0.13 & 0.11 & \textbf{0.081} \\ 
		3 & 0.25 & 20 & 2 & 1 & 0.12 & 0.12 & 0.48 & \textbf{0.11} & 2.7 & 0.58 \\ 
		3 & 0.25 & 20 & 2 & 2 & 0.14 & \textbf{0.097} & 0.5 & 0.11 & 2.65 & 0.6 \\ 
		3 & 0.25 & 20 & 3 & 1 & 4.66 & 5.34 & 7.24 & \textbf{4.57} & 37.52 & 11.15 \\ 
		3 & 0.25 & 20 & 3 & 2 & 0.29 & \textbf{0.24} & 10.94 & 2.29 & 48.25 & 11.11 \\ 
		\hline 3 & 0.25 & 50 & 1 & 1 & 0.017 & 0.015 & \textbf{0.013} & 0.016 & 0.053 & 0.019 \\ 
		3 & 0.25 & 50 & 1 & 2 & 0.035 & \textbf{0.021} & 0.026 & 0.032 & 0.047 & 0.022 \\ 
		3 & 0.25 & 50 & 2 & 1 & 0.054 & 0.056 & 0.2 & \textbf{0.053} & 1.46 & 0.23 \\ 
		3 & 0.25 & 50 & 2 & 2 & 0.031 & 0.021 & 0.2 & \textbf{0.019} & 1.51 & 0.23 \\ 
		3 & 0.25 & 50 & 3 & 1 & 2.32 & 2.47 & 3.28 & \textbf{2.29} & 21.29 & 4.52 \\ 
		3 & 0.25 & 50 & 3 & 2 & 0.085 & \textbf{0.058} & 4.04 & 0.76 & 25.71 & 4.09 \\ 
		\hline 3 & 0.25 & 100 & 1 & 1 & 0.011 & 0.01 & \textbf{0.009} & 0.01 & 0.036 & 0.011 \\ 
		3 & 0.25 & 100 & 1 & 2 & 0.015 & \textbf{0.012} & 0.012 & 0.014 & 0.033 & 0.014 \\ 
		3 & 0.25 & 100 & 2 & 1 & 0.038 & 0.037 & 0.1 & \textbf{0.037} & 0.98 & 0.12 \\ 
		3 & 0.25 & 100 & 2 & 2 & 0.017 & 0.011 & 0.1 & \textbf{0.011} & 1.01 & 0.12 \\ 
		3 & 0.25 & 100 & 3 & 1 & \textbf{1.82} & 1.93 & 2.26 & 1.82 & 14.15 & 2.77 \\ 
		3 & 0.25 & 100 & 3 & 2 & 0.041 & \textbf{0.032} & 1.97 & 0.35 & 17.55 & 2 \\ 
		\hline 3 & 0.50 & 20 & 1 & 1 & 0.14 & 0.13 & 0.12 & 0.13 & \textbf{0.11} & 0.14 \\ 
		3 & 0.50 & 20 & 1 & 2 & 0.63 & 0.45 & 0.55 & 0.55 & \textbf{0.3} & 0.4 \\ 
		3 & 0.50 & 20 & 2 & 1 & 0.22 & 0.25 & 0.47 & \textbf{0.21} & 2.55 & 0.7 \\ 
		3 & 0.50 & 20 & 2 & 2 & 1.13 & 0.95 & \textbf{0.89} & 1.04 & 2.59 & 0.95 \\ 
		3 & 0.50 & 20 & 3 & 1 & 4.89 & 5.27 & 7.24 & \textbf{4.75} & 40.81 & 11.14 \\ 
		3 & 0.50 & 20 & 3 & 2 & 0.93 & \textbf{0.71} & 10.52 & 2.01 & 48.53 & 10.87 \\ 
		\hline 3 & 0.50 & 50 & 1 & 1 & 0.059 & 0.058 & \textbf{0.054} & 0.055 & 0.058 & 0.061 \\ 
		3 & 0.50 & 50 & 1 & 2 & 0.15 & 0.084 & 0.13 & 0.13 & \textbf{0.066} & 0.08 \\ 
		3 & 0.50 & 50 & 2 & 1 & 0.1 & \textbf{0.082} & 0.16 & 0.094 & 1.37 & 0.23 \\ 
		3 & 0.50 & 50 & 2 & 2 & 0.13 & \textbf{0.074} & 0.13 & 0.11 & 1.46 & 0.25 \\ 
		3 & 0.50 & 50 & 3 & 1 & 2.65 & 2.79 & 3.4 & \textbf{2.62} & 22.7 & 4.6 \\ 
		3 & 0.50 & 50 & 3 & 2 & 0.17 & \textbf{0.11} & 4.07 & 0.74 & 26.17 & 4.2 \\ 
		\hline 3 & 0.50 & 100 & 1 & 1 & 0.041 & 0.035 & 0.039 & 0.039 & 0.036 & \textbf{0.035} \\ 
		3 & 0.50 & 100 & 1 & 2 & 0.068 & 0.048 & 0.063 & 0.063 & \textbf{0.039} & 0.048 \\ 
		3 & 0.50 & 100 & 2 & 1 & 0.063 & \textbf{0.056} & 0.085 & 0.06 & 0.92 & 0.12 \\ 
		3 & 0.50 & 100 & 2 & 2 & 0.063 & \textbf{0.039} & 0.067 & 0.053 & 0.98 & 0.12 \\ 
		3 & 0.50 & 100 & 3 & 1 & 1.61 & 1.7 & 2.01 & \textbf{1.6} & 14.36 & 2.54 \\ 
		3 & 0.50 & 100 & 3 & 2 & 0.084 & \textbf{0.056} & 2.04 & 0.34 & 17.11 & 2.1 \\ 
		\hline
\end{longtable}

\normalsize

\end{document}